\documentclass[twoside,11pt]{article}
\usepackage{jmlr2e}

\usepackage{hyperref}       
\usepackage{url}            
\usepackage{booktabs}       
\usepackage{microtype}      

\usepackage{tikz}
\usetikzlibrary{arrows.meta} 
\usepackage{tikz-3dplot} 

\usepackage{nicefrac}       
\usepackage{amsmath}
\usepackage{amssymb}
\usepackage{algpseudocode,algorithm,algorithmicx}
\usepackage{graphicx}
\usepackage{caption}
\usepackage{enumitem} 

\usepackage{subfig}

\usepackage{fixltx2e}

\graphicspath{ {img/} }
\usepackage{footnote}
\makesavenoteenv{table}
\makesavenoteenv{tabular}
\usepackage{multirow}

\usepackage{xstring} 
\usepackage{color}

\usepackage{verbatim} 

\usepackage{ulem}

\usepackage{array} 
\usepackage{longtable} 
\usepackage{colortab} 
\usepackage{colortbl}
\usepackage{arydshln}

\usepackage{color, xcolor}

\newcommand{\syncrecord}[2]{#2} 

\def\oconv{\circledast}
\def\rank{\text{rank}}
\def\dsmark{{\scriptstyle \#}}

\def\diag{\text{diag}}
\def\tvar#1{\mathbf{#1}} 
\def\tvarhat#1{\widehat{\mathbf{#1}}} 
\def\vsymb#1{\vec{\mathbf{#1}}}

\def\isintinf{\int_{-\infty}^{+\infty}\!\!\!\cdots\!\int_{-\infty}^{+\infty}\!}
\def\intinf{\int_{-\infty}^{+\infty}\!}
\def\iintinf{\int_{-\infty}^{+\infty}\!\!\!\int_{-\infty}^{+\infty}\!}
\def\lcerfl#1{\left\lceil{#1}\right\rfloor}

\def\argmax{\mathop{\text{argmax}}}

\def\drawUnitBox(#1,#2,#3){
  \draw[fill, color=black!70!white] 
  (#1,#2,#3) -- +(1,0,0) -- +(1,0,1) -- +(0,0,1);
  \draw[fill, color=black!50!white] 
  (#1,#2,#3) -- +(0,1,0) -- +(1,1,0) -- +(1,0,0);
  \draw[fill, color=black!30!white] 
  (#1,#2,#3) ++(1,0,0) -- +(0,1,0) -- +(0,1,1) -- +(0,0,1);
}

\firstpageno{1}

\title{
  Toward Understanding Convolutional Neural Networks from Volterra Convolution Perspective%
}

\author{%
  \name Tenghui Li  \email tenghui.lee@foxmail.com \\
  \addr School of Automation at Guangdong University of Technology, Guangzhou \\
  \AND
  \name Guoxu Zhou \email gx.zhou@gdut.edu.cn\\
  \addr School of Automation at Guangdong University of Technology, Guangzhou \\
  \AND
  \name Yuning Qiu \email yuning.qiu.gd@gmail.com\\
  \addr School of Automation at Guangdong University of Technology, Guangzhou \\
  \AND
  \name Qibin Zhao \email qibin.zhao@riken.jp \\
  \addr RIKEN Center for Advanced Intelligence Project and \\
  School of Automation at Guangdong University of Technology, Guangzhou \\
}%

\editor{***}

\begin{document}

\maketitle

\begin{abstract}%
  We make an attempt to understanding
  convolutional neural network by exploring the relationship between (deep) convolutional neural networks and Volterra convolutions.
  We propose a novel approach to explain and  study the overall characteristics of neural networks without being disturbed by the horribly complex architectures.
  Specifically, we  attempt to convert the basic structures of a convolutional neural network (CNN) and their combinations to the form of Volterra convolutions.
  The results show that most of convolutional neural networks can be approximated in the form of Volterra convolution, where the approximated proxy kernels preserve the characteristics of the original network. Analyzing these proxy kernels may give valuable insight about the original network.
  Base on this setup, we presented methods to approximating the order-zero and order-one proxy kernels, and verified the correctness and effectiveness of our results.
\end{abstract}


\section{Introduction}
\label{sec:introduction}

Deep neural networks (DNNs) can effectively characterize most complex data as long as training data is large enough and the capable models are well-trained.
Nevertheless, a deep network often has a horribly complex structure, the results are hard to interpret in some sense, and the network is likely to be deceptive by adversarial examples.
We are eager to search for methods that allow us to analyze the network.

There is a vast number of excellent researches focusing on theoretically understanding neural networks.
Some take the statistical perspective.
Deep neural networks can be thought of as being discrete dynamical systems \citep{Weinan2017}.
Instead of thinking about features and neurons, one focus on representation of functions, calculus of variation problems, and continuous gradient flow \citep{E2019}.
Besides, others take a geometric perspective (grids, groups, graphs, geodesics, and gauges), which shows that deep neural networks can be understood in a unified manner as methods that respect the structure and symmetries (invariants and equivalents) of the geometric domains \citep{Bronstein2021}.
Moreover, we can study how modern deep neural networks transform topologies of data sets \citep{Naitzat2020}, or draw the phase diagram for the two-layer ReLU neural network at the infinite-width limit \citep{Luo2021}.

Most of these works focus on certain classes of structures, such as two-layer neural network, multilayer fully connected network, ResNet, pure abstract network, and fully connected network with specific activation functions, i.e., ReLU, sigmoid, tanh, and so on.
It seems that it is unlikely to represent and analyze an arbitrarily complex network from a theoretical point of view.

In this paper, we attempt to build a generic and unified model for analyzing most deep convolutional neural networks rather than thinking about features and layers.
The neural network is a universal approximator that is able to approximate any Borel measurable function \citep{Cybenko1989,Hornik1989,Barron1993,Jinshan2021}.
The proposed model is expected to be a universal approximator.
Additionally, it is supposed to ensure that most networks can be represented by this kind of model.
Furthermore, it should be expressed as superposition of submodules, which makes it convenient to analyze.

The Volterra convolution or Volterra series operator \citep{volterra1932theory} owns exactly such features.
Briefly, Volterra convolution has the form of
\begin{equation}
  \tvar{y}
  = \tvar{H}_0 + \tvar{H}_1 * \tvar{x} + \tvar{H}_2 * \tvar{x}^2 + \cdots
  = \sum_{n=0}^{+\infty} \tvar{H}_n * \tvar{x}^n,
  \label{equ:original-definition-infty-Volterra-convolution}
\end{equation}
where \(\tvar{x}\) is the input signal, \(\tvar{H}_0, \tvar{H}_1, \tvar{H}_2, \cdots\) are kernels, \(\tvar{y}\) is the output signal, and \(\tvar{H}_n * \tvar{x}^n\) is the order-\(n\) convolution (All of these will be precisely defined in Section \ref{sec:extension-of-convolutions}).

Firstly, it has been proved that any time-invariant continuous nonlinear operator can be approximated by a Volterra series operator and any time invariant operator with fading memory can be approximated (in a strong sense) by a nonlinear moving-average operator, the finite term Volterra series operator \citep{Boyd1985}.
Secondly, a certain class of artificial neural networks (feed-forward network or multilayer perceptron) are equivalent to a finite memory Volterra series \citep{Wray1994, Fung1996}.
In addition to this certain class of artificial neural networks, we show that neural networks, including convolutional neural networks and their numerous variants, can be approximated in the form of Volterra convolution.
Thirdly, this is an accumulation of multiple submodules \(\tvar{H}_1 * \tvar{x}, \tvar{H}_2 * \tvar{x}^2, \cdots\), which can be analyzed independently and without being disturbed by the complex network architecture.

Suppose the well-trained network is \(f(\tvar{x})\) or \(g(f(\tvar{x}))\), we are looking for a Volterra convolution to approximate the network \(f(\tvar{x})\).
The functions learned by practical convolutional neural networks (i.e., ReLU-based, sigmoid-based) often can only be represented by Volterra convolution with \textit{infinite} series. 
Nevertheless, if small truncation errors are allowed in practice, we can use \textit{finite} term Volterra convolution via truncating the infinite counterpart to approximate the functions,
which is mainly considered in this paper. Formally, for a function $f(\tvar{x})$, we are looking for $N+1$ proxy kernels \(\tvar{H}_0, \tvar{H}_1, \cdots, \tvar{H}_N\) such that
\begin{equation}
  f(\tvar{x}) \approx \sum_{n=0}^{N} \tvar{H}_n * \tvar{x}^n.
\end{equation}

If a network can be approximated in this form, its kernels \(\tvar{H}_0, \tvar{H}_1, \cdots, \tvar{H}_N\) shall preserve characteristics of the original network, and they will probably help us to analyze the stability or robustness or other useful properties of a well-trained network.

The Volterra convolution looks like a polynomial network \citep{Giles1987,Shin1995,Shin2003, Fallahnezhad2011}. They are similar in the sense of ``polynomial''. Nevertheless, they are quite different. A polynomial network learns a function \(f: \mathbb{R}^n \rightarrow \mathbb{R}\), but Volterra convolution is a map \(f : \mathbb{R}^{n} \rightarrow \mathbb{R}^{m}\). Besides, we are not interested in the training this Volterra convolution by raw data. Instead, we will just approximate a well-trained network in the formulation of Volterra convolution.

In summary, the main contribution of this paper include:
\begin{itemize}
  \item We showed that most convolutional neural networks can be represented in the form of Volterra convolutions. This formulation provides a novel perspective on understanding neural networks, and will probably help us to analyze a neural network without being disturbed by its complex architecture.
  \item We studied some important properties of this representation, including the effects of perturbations and the rank of combining two order-\(n\) convolutions.
  \item We showed that a convolutional neural can be well approximated by  finite term Volterra convolutions. All we need in this approximation are the proxy kernels. We proposed two methods to infer the proxy kernels. The first is by direct calculation provided that the original network is a white-box to users, whereas the second one is for the case where the original network is a black-box.
\end{itemize}

The main structure of this article is listed as follows.
In Section \ref{sec:extension-of-convolutions}, we introduce and review the definition and properties of order-\(n\) convolution, outer convolution and Volterra convolution. 
In Section \ref{sec:connect-to-neural-network}, we show that most convolutional neural networks can be approximated in the form of Volterra convolutions, and validate this approximation on simple structures. 
In Section \ref{sec:inferring-the-proxy-kernels}, we proposed a hacking network to approximate the order-zero and order-one proxy kernels.

\textbf{Notations: }
In most situations, vectors are notated in lower case bold face letters \(\tvar{x}, \tvar{y}, \cdots\), while matrices and tensors are notated in upper case bold face letters \(\tvar{H}, \tvar{G}, \cdots\).
We will not deliberately distinguish between them, as this does not affect the generality of our discussion.
Elements are notated by regular letters with brackets \(x(i), X(i,j,k), \cdots\). The \(i\)th-slice of \(\tvar{X}\) is noted by  \(X(:,i,:)\) using the colon notation.

The vector like notation is \(\vsymb{x} = \lcerfl{\tvar{x}_1, \tvar{x}_2, \cdots}\), which is nothing but a list of objects and the addition and subtraction are defined as \(\vsymb{x} \pm l = \lcerfl{x_1 \pm l, x_2 \pm l, \cdots}\) and \(\vsymb{x} \pm \vsymb{y} = \lcerfl{x_1 \pm y_1, x_2 \pm y_2, \cdots}\).

\section{Extension of Convolutions}
\label{sec:extension-of-convolutions}

In this section, four kinds of convolutions will be introduced in both continuous and discrete time, including the well-known convolution (Equation \ref{equ:def-convolution-1d}), order-\(n\) convolution (Definition \ref{def:order-n-convolution}), Volterra convolution (Definition \ref{def:volterra-convolution-1d}), and outer convolution (Definition \ref{def:outer-convolution}).

Without loss of generality, all kernels and signals are bounded by a constant \(M_1 < \infty\) and have Lipschitz constant \(M_2 < \infty\),
\begin{equation*}
  \left\{
  \tvar{x} \in C(\mathbb{R}) : | x(t) | \le M_1, | x(s) - x(t) | \le M_2(s - t), \text{ for } t \le s
  \right\},
\end{equation*}
where \(C(\mathbb{R}): \mathbb{R} \rightarrow \mathbb{R}\) is the space of bounded continuous functions,
and all kernels are absolute integrable \(\intinf | h(t) | dt < \infty\) or \(\sum_t | h(t) | < \infty\).

The well known one-dimensional convolution \citep{Gonzalez2017-la} of kernel \(\tvar{h}\) and signal \(\tvar{x}\) is
\begin{equation}
  \begin{aligned}
    (\tvar{h} * \tvar{x})(t) & = \intinf h(\tau) x(t - \tau) d \tau & \text{(continuous)}, \\
    (\tvar{h} * \tvar{x})(t) & = \sum_{\tau} h(\tau) x(t - \tau)    & \text{(discrete)}.
  \end{aligned}
  \label{equ:def-convolution-1d}
\end{equation}

\begin{remark}
  If kernel size equals to signal size and padding is zero, the discrete one-dimensional order-one convolution is equivalents to vector inner product with flipped kernel at \(t = 0\),
  \begin{equation*}
    \sum_{\tau} h(\tau) x(t - \tau) \Rightarrow \sum_{\tau} h(-\tau) x(\tau).
  \end{equation*}
\end{remark}

In signal processing, discrete convolution is notated in minus type, such as \(\sum_{\tau} h(\tau) x(t - \tau)\), and correlation is defined in plus type, like \(\sum_{\tau} h(\tau) x(t+\tau)\).
While in convolutional networks, we prefer to notate convolution in plus type.
These two operations can be converted from one to the other by flipping kernels and shifting time (see Figure \ref{fig:difference-between-addition-and-subtraction-conv-type}).
Discussing only in minus type does not affect generality of the results.

\begin{figure}[htb]
  \centering
  \begin{tikzpicture}[thick, ih/.style={draw,rectangle,minimum height=0.64cm, minimum width=0.64cm, fill=black!15, rounded corners=0.2cm}]
    \draw (-4.0, 0.4) node[ih] {$x(t)$}
    ++(0.0,0.8) node[ih] {$h(0)$}
    ++(1.1,0) node[ih] {$h(1)$}
    +(0,-1.6) node {$\sum_{\tau} h(\tau) x(t + \tau)$}
    ++(1.1,0) node[ih] {$h(2)$}
    ++(1.1,0) node {$\cdots$};
    \draw[dashed, gray] (-0.4, 0) -- ++(-4.4,0) (-0.4,0.8) -- ++(-4.4,0);

    \draw (0.8, 1.2) node {$\cdots$}
    ++(1.1,0) node[ih] {$h(2)$}
    ++(1.1,0) node[ih] {$h(1)$}
    +(0,-1.6) node {$\sum_{\tau} h(\tau) x(t - \tau)$}
    ++(1.1,0) node[ih] {$h(0)$}
    ++(0, -0.8) node[ih] {$x(t)$};
    \draw[dashed, gray] (0.4, 0) -- ++(4.4,0) (0.4,0.8) -- ++(4.4,0);
  \end{tikzpicture}
  \caption{Differences between addition type and subtraction type.}
  \label{fig:difference-between-addition-and-subtraction-conv-type}
\end{figure}
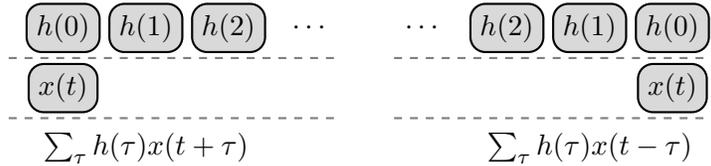

\subsection{Order-n Convolution}
\label{subsec:order-n-convolution}

Nonlinearities of convolutional neural networks come from their nonlinear activation functions. Theoretically, it is possible to embed nonlinearity in convolutional operation by taking order-\(n\) convolution.

For a simple example: a parabola can be described by a two-layer network with activation \(\sigma(\cdot)\), i.e., \(w_2 \sigma(w_1 x + b_1) + b_2\), or it can be expressed in polynomial \(\alpha_0 + \alpha_1 x + \alpha_2 x^2 + \cdots\).
This type of ``polynomial'' can be also applied to convolutional operation, i.e.,
\begin{equation*}
  \begin{aligned}
    \alpha x^2 & \leftrightarrows \intinf H(\tau_1, \tau_2) x(t - \tau_1) x(t-\tau_2) d \tau_1 d \tau_2,                                \\
    \alpha xy  & \leftrightarrows \intinf H(\tau_1, \tau_2) x(t - \tau_1) y(t-\tau_2) d \tau_1 d \tau_2,                                \\
    \alpha x^3 & \leftrightarrows \intinf H(\tau_1, \tau_2, \tau_3) x(t - \tau_1) x(t-\tau_2) x(t - \tau_3) d \tau_1 d \tau_2 d \tau_3. \\
  \end{aligned}
\end{equation*}

With these basic concepts in mind, formal definition of order-\(n\) convolution is presented in Definition \ref{def:order-n-convolution}.

\begin{definition}
  Order-\(n\) convolution \citep{volterra1932theory, rugh1981nonlinear} of kernel \(\tvar{H}\) and \(n\) signals \(\lcerfl{\tvar{x}_1, \tvar{x}_2, \cdots, \tvar{x}_n} \equiv \vsymb{x}\) is
  \begin{equation}
    \tvar{H} * \lcerfl{\tvar{x}_1, \tvar{x}_2, \cdots, \tvar{x}_n} \equiv \tvar{H} * \vsymb{x},
    \label{equ:def-order-n-convolution-1d}
  \end{equation}
  where \(\vsymb{x}\) is vector like notation and all \(\tvar{x}_i\) have the same dimension.
  \label{def:order-n-convolution}
\end{definition}

If \(\tvar{x}_1, \tvar{x}_2, \cdots, \tvar{x}_n\) are all one-dimensional signals and \(\tvar{H}\) is an \(n\)-dimensional signal, the continuous order-\(n\) convolution for one-dimensional signal is
\begin{equation}
  \begin{aligned}
    \left(\tvar{H} * \vsymb{x}\right)(t)
     & = \left(\tvar{H} * \lcerfl{\tvar{x}_1, \tvar{x}_2, \cdots, \tvar{x}_n} \right)(t)                     \\
     & = \isintinf H(\tau_1, \tau_2, \cdots, \tau_n) \prod_{i=1}^{n} \left( x_i(t - \tau_i) d \tau_i \right) \\
     & \equiv \isintinf H(\vsymb{\tau}) \prod_{i=1}^{n}\left( x_i(t - \tau_i) d \tau_i \right).
  \end{aligned}
  \label{equ:def-continuous-order-n-convolution-1d}
\end{equation}

If \(\tvar{x}_1, \tvar{x}_2, \cdots, \tvar{x}_n\) are all \(m\)-dimensional signals and \(\tvar{H}\) is an \(nm\)-dimensional signal, and vector like notations have the form \(\vsymb{\tau}_1 = \lcerfl{\tau_{1,1}, \tau_{1,2}, \cdots, \tau_{1,m}}\); \(\cdots\); \(\vsymb{\tau}_n = \lcerfl{\tau_{n,1}, \tau_{n,2}, \cdots, \tau_{n,m}}\) and \(\vsymb{t} = \lcerfl{t_1, t_2, \cdots, t_m}\), order-\(n\) convolution for \(m\)-dimensional signal is
\begin{equation}
  \begin{aligned}
    \left(\tvar{H} * \vsymb{x} \right)(t_1, \cdots, t_m)
     & = \left(\tvar{H} * \lcerfl{\tvar{x}_1, \tvar{x}_2, \cdots, \tvar{x}_n} \right)(t_1, \cdots, t_m) \\
     & = \isintinf H(\tau_{1,1}, \cdots, \tau_{1,m}; \cdots ; \tau_{n,1}, \cdots, \tau_{n,m})           \\
     & \qquad\quad
    \prod_{i=1}^{n} \left(
    x_i(t_1 - \tau_{i,1}, \cdots, t_m - \tau_{i,m}) d \tau_{i,1} \cdots d \tau_{i,m}
    \right)                                                                                             \\
     & \equiv \isintinf H(\vsymb{\tau}_1, \vsymb{\tau}_2, \cdots, \vsymb{\tau}_n)
    \prod_{i=1}^{n} \left(
    x(\vsymb{t} - \vsymb{\tau}_i) d \vsymb{\tau}_i
    \right).
  \end{aligned}
  \label{equ:def-continuous-order-n-convolution-md}
\end{equation}

With the vector like notation, discrete order-\(n\) convolution for \(m\)-dimensional signal is simplified as
\begin{equation}
  \left(\tvar{H} * \vsymb{x}\right)(\vsymb{t})
  \equiv \sum_{\vsymb{\tau}_1, \cdots, \vsymb{\tau}_n} H(\vsymb{\tau}_1, \vsymb{\tau}_2, \cdots, \vsymb{\tau}_n) \prod_{i=1}^{n} x_i(\vsymb{t} - \vsymb{\tau}_i).
  \label{equ:def-discrete-order-n-convolution-md}
\end{equation}

If \(n = 0\), order-zero convolution \(\tvar{H} * \tvar{x}^0 = \tvar{H} * \delta = \tvar{H}\), where \(\delta\) is the Dirac delta,
\begin{equation*}
  \delta(t) = \left\{\begin{array}{cc}
    \infty, & t = 0   \\
    0,      & t \ne 0
  \end{array}\right.  \text{(continuous),}
  ~~~~
  \delta(t) = \left\{\begin{array}{cc}
    1, & t = 0   \\
    0, & t \ne 0
  \end{array}\right.  \text{(discrete).}
\end{equation*}

If \(n = 1\), this is order-one (first-order) convolution (Equation \ref{equ:def-convolution-1d}), and \(n = 2\), this is order-two (second-order) convolution \(\tvar{H} * \lcerfl{\tvar{x}, \tvar{y}}\).
The order-two convolution does not come from void, as it is an extension of order-one convolution. (Please read Appendix \ref{appendix:convolution-from-order-1-to-2} for more details.)

For better understanding this notation, few examples are expressed as follows.

\begin{figure}[H]
  \centering
  \begin{tikzpicture}
    \node[inner sep=1pt] (A) at (0,0) {$\tvar{H} * \lcerfl{\tvar{x}, \tvar{y}, \tvar{z}, \cdots}$};
    \draw[->] (A.south east) ++(0.6,-0.2) node[below] {The number of signal equals the order.} -- (A.south east);
  \end{tikzpicture}
\end{figure}

Besides, if all signals are equal \(\tvar{x}_1 = \tvar{x}_2 = \cdots = \tvar{x}_n = \tvar{x}\), it can be written as \(\tvar{H} * \tvar{x}^n\) for short.
\begin{figure}[H]
  \centering
  \begin{tikzpicture}
    \node[inner sep=1pt] (A) at (0,0) {$\tvar{H} * \lcerfl{\tvar{x}, \tvar{x}, \cdots, \tvar{x}} = \tvar{H} * \tvar{x}^{n}$};
    \draw[->] (A.north east) ++(0.6,0.2) node[right] {$n$ equals the order.} -- (A.north east);
  \end{tikzpicture}
\end{figure}

Dimension of this convolution is determined by its signals.
If \(\tvar{x}_j, j = 1, 2, \cdots, n\) are all \(m\)-dimensional signals, \(\tvar{H} * \vsymb{x}\) is called order-\(n\) convolution for \(m\)-dimensional signal.
\begin{figure}[H]
  \centering
  \begin{tikzpicture}
    \node[inner sep=1pt] (A) at (0,0) {$\tvar{H} * \lcerfl{\cdots, \left[x_j(i_1, i_2, \cdots, i_m)\right], \cdots}$};
    \draw[->] (A.south) ++(0.5,-0.3) node[left] {$m$ equals the dimension of signal.} -- ++(0.6, 0.2);
  \end{tikzpicture}
\end{figure}

If signals are grouped and \(n_1 + n_2 + \cdots + n_m\) equals the order, it can simplify as
\begin{equation*}
  \tvar{H} * \lcerfl{
    \underbrace{\tvar{x}_1, \cdots, \tvar{x}_1}_{n_1 \text{ terms}},
    \underbrace{\tvar{x}_2, \cdots, \tvar{x}_2}_{n_2 \text{ terms}},
    \cdots,
    \underbrace{\tvar{x}_m, \cdots, \tvar{x}_m}_{n_m \text{ terms}}
  } = \tvar{H} * \lcerfl{
  \tvar{x}_1^{n_1},
  \tvar{x}_2^{n_2},
  \cdots,
  \tvar{x}_m^{n_m}
  }.
\end{equation*}

In the following, we will illustrate two observations of \(\tvar{H} * \tvar{x}^n\), where \(\tvar{x}\) is \textbf{discrete one-dimensional} signal.

The first observation is that each dimension of \(\tvar{H}\) is equal, i.e., \(\tvar{H} \in \mathbb{R}^{m \times m \times \cdots \times m}\).
The second observation is that there exists symmetric \(\tvarhat{H}\) such that \(\tvarhat{H} * \tvar{x}^n = \tvar{H} * \tvar{x}^n\), where symmetry means \(\hat{H}(\cdots, \tau_i, \cdots, \tau_j, \cdots) = \hat{H}(\cdots, \tau_j, \cdots, \tau_i, \cdots)\) for any \(\tau_i, \tau_j\) at any dimension.
The first is obvious, and the second is demonstrated as below.

Expanding \(\tvar{H} * \tvar{x}^n\) as
\begin{equation*}
  \begin{aligned}
    \cdots
     & + H(\cdots, \tau_i, \cdots, \tau_j, \cdots) \cdots x(t - \tau_i) \cdots x(t - \tau_j) \cdots \\
     & + H(\cdots, \tau_j, \cdots, \tau_i, \cdots) \cdots x(t - \tau_j) \cdots x(t - \tau_i) \cdots
    + \cdots,                                                                                       \\
  \end{aligned}
\end{equation*}
and take \(\hat{H}(\cdots, \tau_i, \cdots, \tau_j, \cdots) = \dfrac{1}{2}\left(
H(\cdots, \tau_i, \cdots, \tau_j, \cdots) +
H(\cdots, \tau_i, \cdots, \tau_j, \cdots)
\right)\), we have \(\tvarhat{H} * \tvar{x}^n = \tvar{H} * \tvar{x}^n\). Therefore, without further notice, this kind of kernels are always symmetric.

\subsection{Volterra Convolution}
\label{subsec:voltarra-convolution}

In this subsection, we sum these convolutions from order-zero to order-\(n\) or order-\(\infty\).
If the order is finite, it is called the finite term Volterra convolution or order-\(n\) Volterra convolution, otherwise it is called the infinity term Volterra convolution or Volterra convolution.
For instance, the order-two Volterra convolution is sum of order-zero, order-one and order-two convolutions,
\(\tvar{H}_0 * \tvar{x}^0 + \tvar{H}_1 * \tvar{x}^1 + \tvar{H}_2 * \tvar{x}^2\).

For simplicity and the fact that a neural network takes only one input (multiple inputs are packed to one tensor), input signals of each order are set to be the same \(\tvar{x}_1 = \tvar{x}_2 = \cdots = \tvar{x}_n\). If the input signals are one-dimensional, all kernels are symmetric.

\begin{definition}
  Let \(\tvar{x}\) be signal and \(\tvar{H}_n\) as kernels, Volterra convolution \citep{volterra1932theory, rugh1981nonlinear} is defined as
  \begin{equation}
    \sum_{n=0}^{+\infty} \tvar{H}_n * \tvar{x}^n
    = \sum_{n=0}^{+\infty} \tvar{H}_n * \underbrace{\lcerfl{\tvar{x}, \tvar{x}, \cdots}}_{n \text{ terms}}.
    \label{equ:def-volterra-convolution-1d}
  \end{equation}
  If \(n=0\), \(\tvar{x}^0 = \delta\), i.e., the Dirac delta.
  \label{def:volterra-convolution-1d}
\end{definition}

If \(\tvar{x}\) is a one-dimensional signal and each \(\tvar{H}_n\) is an \(n\)-dimensional signal, continuous Volterra convolution for one-dimensional signal is
\begin{equation}
  \left( \sum_{n=0}^{+\infty} \tvar{H}_n * \tvar{x}^n  \right)(t)
  = \sum_{n=0}^{+\infty} \isintinf H_n(\tau_1, \cdots, \tau_n)
  \prod_{i=1}^{n} \left( x(t - \tau_i) d \tau_i \right).
  \label{equ:def-continuous-volterra-convolution-1d}
\end{equation}
According to previous discussion, here the kernels \(\tvar{H}_n\), $n=1,2,3,\ldots$, are symmetric.

If \(\tvar{x}\) is an \(m\)-dimensional signal and \(\tvar{H}_n\) is an \(nm\)-dimensional signal for each \(n=0, 1, \cdots\), and \(\vsymb{\tau}_1 = \lcerfl{\tau_{1,1}, \tau_{1,2}, \cdots, \tau_{1,m}}\); \(\cdots\); \(\vsymb{\tau}_n = \lcerfl{\tau_{n,1}, \tau_{n,2}, \cdots, \tau_{n,m}}\) and \(\vsymb{t} = \lcerfl{t_1, t_2, \cdots, t_m}\), continuous Volterra convolution for \(m\)-dimensional signal is
\begin{equation}
  \begin{aligned}
    \left(
    \sum_{n=0}^{+\infty} \tvar{H}_n * \tvar{x}^n
    \right)(t_1, \cdots, t_m)
     & = \sum_{n=0}^{+\infty} \isintinf H_n(\tau_{1,1}, \cdots, \tau_{1,m}; \cdots ; \tau_{n,1}, \cdots, \tau_{n,m}) \\
     & \qquad\quad
    \prod_{i=1}^{n} \left(
    x(t_1 - \tau_{i,1}, \cdots, t_m - \tau_{i,m}) d \tau_{i,1} \cdots d \tau_{i,m}
    \right)                                                                                                          \\
     & \equiv \sum_{n=0}^{+\infty} \isintinf
    H_n(\vsymb{\tau}_1, \vsymb{\tau}_2, \cdots, \vsymb{\tau}_n)
    \prod_{i=1}^{n} \left( x(\vsymb{t} - \vsymb{\tau}_i) d \vsymb{\tau}_i \right).
  \end{aligned}
  \label{equ:def-continuous-volterra-convolution-md}
\end{equation}

With the vector like notation, discrete Volterra convolution for \(m\)-dimensional signal is simplified as
\begin{equation}
  \left( \sum_{n=0}^{+\infty} \tvar{H}_n * \tvar{x}^n \right)(\vsymb{t})
  = \sum_{n=0}^{+\infty} \sum_{\vsymb{\tau}_1, \cdots, \vsymb{\tau}_n}
  H_n(\vsymb{\tau}_1, \vsymb{\tau}_2, \cdots, \vsymb{\tau}_n) \prod_{i=1}^{n} x(\vsymb{t} - \vsymb{\tau}_i).
  \label{equ:def-discrete-volterra-convolution-md}
\end{equation}

\subsection{Outer Convolution}
\label{subsec:outer-convolution}

Stacking two order-one one-dimensional convolutions will produce a one-dimensional convolution with a longer kernel. How do we stack order-\(n\) convolutions?
In this subsection, an operation, the outer convolution, is introduced to combine these convolutions. Moreover, the rank properties for outer convolutions are described in Appendix \ref{appendix:convolution-rank}.

Let \(\tvar{G}\) and \(\lcerfl{\tvar{H}_1, \tvar{H}_2, \cdots, \tvar{H}_n} \equiv \vsymb{H}\) are the kernels for convolutions of one-dimensional signal. The outer convolution of $\tvar{G}$ and $\vsymb{H}$, denoted by
\begin{equation}
  \tvar{G} \oconv \lcerfl{\tvar{H}_1, \tvar{H}_2, \cdots, \tvar{H}_n} \equiv \tvar{G} \oconv \vsymb{H},
  \label{equ:def-outer-convolution}
\end{equation}
is defined as follows.

\begin{definition}[Continuous outer convolution of kernels]
  Let \(\tvar{G}\) be an \(n\)-dimensional kernel and each dimension of \(\tvar{H}_i\) no less than one, then the continuous outer convolution yields an $L$-dimensional kernel satisfying
  \begin{equation}
    \begin{aligned}
       & \left(\tvar{G} \oconv \vsymb{H}\right)(t_{1,1}, t_{1,2}, \cdots; t_{2,1}, t_{2,2}, \cdots; \cdots; t_{n,1}, t_{n,2}, \cdots)           \\
       & = \isintinf G(\tau_1, \tau_2, \cdots, \tau_n) \prod_{i=1}^{n} \left( H_i (t_{i,1} - \tau_i, t_{i,2} - \tau_i, \cdots) d \tau_i \right) \\
       & \equiv \isintinf G(\vsymb{\tau}) \prod_{i=1}^{n} \left(
      H_i (\vsymb{t}_i - \tau_i) d \tau_i
      \right),
    \end{aligned}
    \label{equ:def-continuous-outer-convolution-1d}
  \end{equation}
  where $L$ is the sum of the dimensions of $\tvar{H}_i$, \(\vsymb{\tau} = \lcerfl{\tau_1, \tau_2, \cdots, \tau_n}\),  \(\vsymb{t}_1 = \lcerfl{t_{1,1}, t_{1,2}, \cdots}\), \(\vsymb{t}_2 = \lcerfl{t_{2,1}, t_{2,2}, \cdots}\), \(\cdots\), \(\vsymb{t}_n = \lcerfl{t_{n,1}, t_{n,2}, \cdots}\), using the vector like notations.
  \label{def:outer-convolution}
\end{definition}

With Equation \ref{equ:def-continuous-outer-convolution-1d} we can compute the convolution between $\tvar{G} \oconv \vsymb{H}$ and one-dimensional signals $\tvar{x}$. Since most signals in this article are one-dimensional, we prefer to express the outer convolution in the form of Equation \ref{equ:def-continuous-outer-convolution-1d}.

Note that Equation \ref{equ:def-continuous-outer-convolution-1d} allows us to combine multiple convolution layers, and detailed rules will be described in Subsection \ref{subsec:convolution-properties}.
Two of these rules are quick previewed as follows:
\begin{itemize}
  \item \(\tvar{G} * (\tvar{H} * \tvar{x}) = (\tvar{G} \oconv \tvar{H}) * \tvar{x}\) (Property \ref{prop:conv-g-conv-h-x}),
  \item \(\tvar{G} * \lcerfl{\tvar{H}_1 * \tvar{x}_{1}, \tvar{H}_2 * \tvar{x}_2} = (\tvar{G} \oconv \lcerfl{\tvar{H}_1, \tvar{H}_2}) * \lcerfl{\tvar{x}_1, \tvar{x}_2}\) (Property \ref{prop:conv-g-mul-conv-h1-x-conv-h2-y}).
\end{itemize}

More generally, we consider the layers involving the convolution of $m$-dimensional signals $\tvar{x}$. To this end, suppose that \(\tvar{G}\) is an \(nm\)-dimensional kernel, and each dimension of \(\tvar{H}_i\) is no less than $m$, continuous outer convolution of $\tvar{G}$ and $\tvar{H}_i$ can be expressed as
\begin{equation}
  \begin{aligned}
     & \left(\tvar{G} \oconv \vsymb{H}\right)\left(\begin{array}{c}
        t_{1,1,1}, t_{1,1,2}, \cdots t_{1,1,m};
        t_{1,2,1}, \cdots t_{1,2,m}; \cdots; \\
        t_{2,1,1}, t_{2,1,2}, \cdots t_{2,1,m};
        t_{2,2,1}, \cdots t_{2,2,m}; \cdots; \\
        \cdots; \cdots;                      \\
        t_{n,1,1}, t_{n,1,2}, \cdots t_{n,1,m};
        t_{n,2,1}, \cdots t_{n,2,m}; \cdots; \\
      \end{array}\right) \\
     & = \isintinf \!\!G \left(
    \tau_{1,1}, \tau_{1,2}, \cdots, \tau_{1,m};
    \tau_{2,1}, \tau_{2,2}, \cdots, \tau_{2,m};
    \cdots; \cdots;
    \tau_{n,1}, \tau_{n,2}, \cdots, \tau_{n,m};
    \right)                                                                          \\
     & ~~~
    \prod_{i=1}^{n} \left(
    H_i \left(\begin{array}{c}
          t_{i,1,1} - \tau_{i,1}, t_{i,1,2} - \tau_{i,2}, \cdots, t_{i,1,m} - \tau_{i,m}; \\
          t_{i,2,1} - \tau_{i,1}, t_{i,2,2} - \tau_{i,2}, \cdots, t_{i,2,m} - \tau_{i,m}; \\
          \cdots                                                                          \\
        \end{array} \right)
    d \tau_{i,1} d \tau_{i,2} \cdots d \tau_{i,m}
    \right).                                                                         \\
     & \equiv \isintinf G(\vsymb{\tau}_1, \vsymb{\tau}_2, \cdots, \vsymb{\tau}_n)
    \prod_{i=1}^{n} \left(
    H_i (\vsymb{t}_{i,1} - \vsymb{\tau}_i, \vsymb{t}_{i,2} - \vsymb{\tau}_i, \cdots ) d \vsymb{\tau}_i
    \right).
  \end{aligned}
  \label{equ:def-continuous-outer-convolution-md}
\end{equation}
The outer convolution (Equation \ref{equ:def-continuous-outer-convolution-md}) will yield a new kernel whose convolution with $m$-dimentional signals is given by Equation \ref{equ:def-continuous-order-n-convolution-md}.

Using the vector like notation, discrete outer convolution for one-dimensional signal is simplified as
\begin{equation}
  \left( \tvar{G} \oconv \vsymb{H} \right)(\vsymb{t}_1, \vsymb{t}_2, \cdots, \vsymb{t}_n)
  = \sum_{\vsymb{\tau}} G(\vsymb{\tau}) \prod_{i=1}^{n} H_i(\vsymb{t}_i - \tau_i),
  \label{equ:def-descrete-outer-convolution-1d}
\end{equation}
and discrete outer convolution for \(m\)-dimensional signal is simplified as
\begin{equation}
  \begin{aligned}
     & \left( \tvar{G} \oconv \vsymb{H} \right)(\vsymb{t}_{1,1}, \vsymb{t}_{1,2}, \cdots; \vsymb{t}_{2,1}, \cdots; \cdots; \vsymb{t}_{n,1}, \cdots)                                               \\
     & = \sum_{\vsymb{\tau}_1, \cdots} G(\vsymb{\tau}_1, \vsymb{\tau}_2, \cdots, \vsymb{\tau}_n) \prod_{i=1}^{n} H_i(\vsymb{t}_{i,1} - \vsymb{\tau}_i, \vsymb{t}_{i,2} - \vsymb{\tau}_i, \cdots).
  \end{aligned}
  \label{equ:def-descrete-outer-convolution-md}
\end{equation}

The shorthand notations of outer convolution are similar to that of order-\(n\) convolution.
If all signals are equal \(\tvar{H}_1 = \tvar{H}_2 = \cdots = \tvar{H}_n\), we have \(\tvar{G} \oconv \lcerfl{\tvar{H}_1, \tvar{H}_2, \cdots, \tvar{H}_n} = \tvar{G} \oconv \tvar{H}^n\).

If signals are grouped and \(n_1 + n_2 + \cdots + n_m\) equals the order, the notation can be simplified as
\begin{equation*}
  \tvar{G} \oconv \lcerfl{
    \underbrace{\tvar{H}_1, \cdots, \tvar{H}_1}_{n_1 \text{ terms}},
    \underbrace{\tvar{H}_2, \cdots, \tvar{H}_2}_{n_2 \text{ terms}},
    \cdots,
    \underbrace{\tvar{H}_m, \cdots, \tvar{H}_m}_{n_m \text{ terms}}
  } = \tvar{G} * \lcerfl{
  \tvar{H}_1^{n_1},
  \tvar{H}_2^{n_2},
  \cdots,
  \tvar{H}_m^{n_m}
  }.
\end{equation*}

\begin{remark}
  In discrete outer convolution \(\tvar{G} \oconv \vsymb{H}\), kernels \(\tvar{H}_1, \cdots, \tvar{H}_n\) are all zero padded on both heads and tails. The padding size of \(\tvar{H}_i\) is \(s_i - 1\), where \(s_i\) is shape of \(\tvar{G}\) at dimension \(i\). For instance, one-dimensional \(\tvar{H}_i\) is zero padded as
  \begin{equation*}
    \left[ ~ \underbrace{\cdots ~ 0 ~ \cdots}_{s_i - 1} ~~ H_i(0) \cdots H_i(n) ~~ \underbrace{\cdots ~ 0 ~ \cdots}_{s_i-1} ~ \right].
  \end{equation*}
\end{remark}

For example, if \(\tvar{G}\) has shape \((s_1, s_2)\), \(\tvar{H}_1\) has shape \((z_1)\) and \(\tvar{H}_2\) has shape \((z_2, z_3, z_4)\), outer convolution (with zero padded) \(\tvar{G} \oconv \lcerfl{\tvar{H}_1, \tvar{H}_2}\) will have shape
\(( s_1 + z_1 - 1, s_2 + z_2 - 1, s_2 + z_3 - 1, s_2 + z_right)\).

\begin{remark}
  \(\tvar{G} \oconv \tvar{H}\) has the same operation as the ``ConvTranspose'', which is a deep learning operator \footnote{\url{https://pytorch.org/docs/stable/generated/torch.nn.functional.conv_transpose1d.html}}.
\end{remark}

\begin{remark}
  If \(\tvar{h}_1, \tvar{h}_2, \cdots, \tvar{h}_n\) are all one-dimensional vectors, and tensor \(H(t_1, t_2, \cdots, t_n) = h_1(t_1) h_2(t_2) \cdots h_n(t_n)\), outer convolution can be transformed to multidimensional convolution \(\tvar{G} \oconv \lcerfl{\tvar{h}_1, \tvar{h}_2, \cdots, \tvar{h}_n} = \tvar{G} * \tvar{H}\).
\end{remark}

The super diagonal kernel \(\tvar{G}\) has only non-zero elements on the diagonal, \(G(0, 0, \cdots, 0)\), \(G(1, 1, \cdots, 1)\), \(\cdots\), and all other elements are set to zero.
More specifically, we define the \(\tvar{G} = \diag(n, \tvar{g})\), where \(\tvar{g}\) is a vector, and
\begin{equation}
  \label{equ:briefly-diag-operator}
  G(\tau_1, \tau_2, \cdots, \tau_n)
  = \diag(n, \tvar{g})(\tau_1, \tau_2, \cdots, \tau_n)
  = \sum_{k} g(k) \prod_{i=1}^{n} \delta(\tau_i - k).
\end{equation}
For better understanding, this \(\diag(\cdot)\) operator is visualized in Figure \ref{fig:preview-diag-operator}.

\begin{figure}[htb]
  \centering
  \subfloat[one-dimensional \(\tvar{g} \in \mathbb{R}^{5}\)]{
    \begin{tikzpicture}[yscale=-1, scale=0.6]
      \draw[white] (0,-2.5) rectangle (4,3.3);
      \foreach \i in {0,...,4}{
          \draw[fill, color=black!40!white] (\i,0) rectangle (\i+0.8, 0.8);
        }
      \draw[thick, ->] node[left] {$0$} (0,0) -- (5.5,0);
    \end{tikzpicture}
  }
  \subfloat[\(\diag(2, \tvar{g}) \in \mathbb{R}^{5 \times 5}\)]{
    \begin{tikzpicture}[yscale=-1,scale=0.5] 
      \draw[black!20] (0,0) grid (5,5);
      \foreach \i in {0,...,4}{
          \draw[fill, color=black!40!white] (\i,\i) rectangle (\i+1,\i+1);
        }
      \draw[thick,->] node[left] {$0$} (0,0) -- (5.5, 0);
      \draw[thick,->] (0,0) -- (0, 5.5);
    \end{tikzpicture}
  }
  \subfloat[\(\diag(3, \tvar{g}) \in \mathbb{R}^{5 \times 5 \times 5}\)]{
    \tdplotsetmaincoords{-75}{45} 
    \begin{tikzpicture}[scale=0.4,tdplot_main_coords]
      \foreach \i in {1,...,5}{
          \draw[thin, black!20, -] (\i, 0, 5) -- (\i, 5, 5);
          \draw[thin, black!20, -] (\i, 5, 0) -- (\i, 5, 5);
          \draw[thin, black!20, -] (0, \i, 5) -- (5, \i, 5);
          \draw[thin, black!20, -] (0, \i, 0) -- (0, \i, 5);
          \draw[thin, black!20, -] (0, 5, \i) -- (5, 5, \i);
          \draw[thin, black!20, -] (0, 0, \i) -- (0, 5, \i);
        }
      \draw[thin, black!20, -] (0, 5, 0) -- (5, 5, 0);
      \draw[thin, black!20, -] (0, 0, 5) -- (5, 0, 5);

      \drawUnitBox(0,0,0);
      \drawUnitBox(1,1,1);
      \drawUnitBox(2,2,2);
      \drawUnitBox(3,3,3);
      \drawUnitBox(4,4,4);

      \draw[thick,->] node[left] {$0$} (0,0,0) -- (5.5,0,0);
      \draw[thick,->] (0,0,0) -- (0,5.5,0);
      \draw[thick,->] (0,0,0) -- (0,0,5.5);
    \end{tikzpicture}
  }
  \caption{A brief preview of the \(\diag(\cdot)\) operator.}
  \label{fig:preview-diag-operator}
\end{figure}
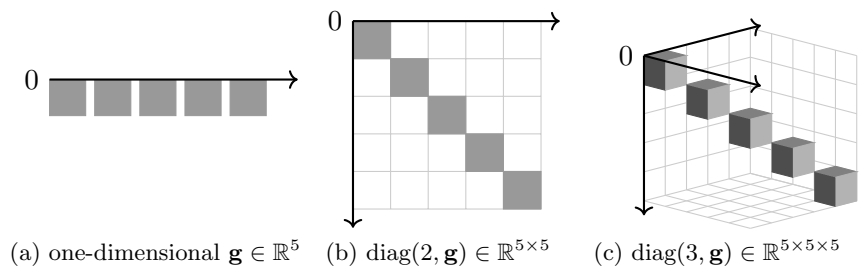

Specially, if \(\tvar{G}\) is a super diagonal tensor kernel, its outer convolution has the form of
\begin{equation}
  \label{equ:oconv-diag-g-h1-h2-hn}
  \begin{aligned}
     & \left(\tvar{G} \oconv \lcerfl{\tvar{H}_1, \tvar{H}_2, \cdots, \tvar{H}_n}\right)(
    \vsymb{t}_1, \vsymb{t}_2, \cdots, \vsymb{t}_n
    )                                                                                    \\
     & = \sum_{\vsymb{\tau}} G(\vsymb{\tau}) \prod_{i=1}^{n} H_i(\vsymb{t}_i - \tau_i)
    = \sum_{k} g(k) \prod_{i=1}^{n} \delta(\tau_i - k) H_i(\vsymb{t}_i - \tau_i)
    = \sum_{k} g(k) \prod_{i=1}^{n} H_i(\vsymb{t}_i - k).                                \\
  \end{aligned}
\end{equation}

\subsection{Convolution With Stride Grater Than One}
\label{subsec:convolution-with-stride-grater-than-one}

Convolution with stride grater than one is commonly used to replace convolution-pooling structure.
If the stride equals one, filters move one point at a time.
If the stride equals two, filters jump two points at a time.
In addition, convolution with stride \(s\) is
\begin{equation*}
  \left(\tvar{h} *_s \tvar{x}\right)(t) = \sum_{\tau} h(\tau) x(s t - \tau),
\end{equation*}
where subscript of asterisk \(*_s\) indicates stride.

This operation can be applied to order-\(n\) convolution and outer convolution.
With vector like notation, let \(s \vsymb{t} = \lcerfl{s t_1, s t_2, \cdots, s t_n}\). Discrete order-\(n\) convolution for \(m\)-dimensional signal with stride \(s\) is
\begin{equation}
  \left(\tvar{H} *_s \vsymb{x}\right)(\vsymb{t})
  = \sum_{\vsymb{\tau}_1, \cdots, \vsymb{\tau}_n} H(\vsymb{\tau}_1, \vsymb{\tau}_2, \cdots, \vsymb{\tau}_n) \prod_{i=1}^{n} x_i(s \vsymb{t} - \vsymb{\tau}_i).
\end{equation}
Similarly, the discrete outer convolution for \(m\)-dimensional signal with stride \(s\) is
\begin{equation}
  \begin{aligned}
     & \left( \tvar{G} \oconv_s \vsymb{H} \right)(\vsymb{t}_{1,1}, \vsymb{t}_{1,2}, \cdots; \vsymb{t}_{2,1}, \cdots; \cdots; \vsymb{t}_{n,1}, \cdots) \\
     & = \sum_{\vsymb{\tau}_1, \cdots}
    G(\vsymb{\tau}_1, \vsymb{\tau}_2, \cdots, \vsymb{\tau}_n)
    \prod_{i=1}^{n} H_i(s \vsymb{t}_{i,1} - \vsymb{\tau}_i, s \vsymb{t}_{i,2} - \vsymb{\tau}_i, \cdots).
  \end{aligned}
\end{equation}

The combination of two convolutions with strides is also equivalent to the outer convolution with strides, \(\tvar{G} *_s (\tvar{H} *_z \tvar{x}) = (\tvar{G} \oconv_z \tvar{H}) *_{sz} \tvar{x}\) (Property \ref{prop:conv-stride-s-g-conv-stride-z-h-x} in Subsection \ref{subsec:convolution-properties}).

\begin{remark}
  If \(\tvar{G}\) has shape \((z_1, z_2, \cdots, z_m)\) and \(\tvar{H}\) has shape \(c_1, c_2, \cdots, c_m\), the shape of \(\tvar{G} \oconv_s \tvar{H}\) is
  \begin{equation*}
    \left(c_1 + (z_1 - 1) s, c_2 + (z_2 - 1) s, \cdots,  c_m + (z_m - 1) s\right).
  \end{equation*}
\end{remark}

\subsection{Visualization of Outer Convolution}
\label{subsec:Visualization-of-outer-convolution}

For the following cases, three examples are provided to help understand the outer convolution.

Let \(\tvar{G} \in \{1\}^{8 \times 8}\) be a \(8 \times 8\) matrix with all elements are one, and \(\tvar{h} \in \{1\}^{8}\) be a vector of length eight with all elements are one. Figure \ref{fig:oconv-2d-1d-1d} demonstrates
\begin{equation*}
  \left(\tvar{G} \oconv \lcerfl{\tvar{h}, \tvar{h}}\right)(t_1, t_2)
  = \sum_{\tau_1 = 0, \tau_2 = 0}^{7,7} G(\tau_1, \tau_2) h(t_1 - \tau_1) h(t_2 - \tau_2).
\end{equation*}

Let \(\tvar{g} \in \{1\}^{8}\) and \(\tvar{H} \in \{1\}^{8 \times 8}\). Figure \ref{fig:oconv-1d-2d} demonstrates
\begin{equation*}
  \left( \tvar{g} \oconv \tvar{H} \right)(t_1, t_2)
  = \sum_{t_1=0, t_2=0}^{7,7} g(\tau) H(t_1 - \tau, t_2 - \tau).
\end{equation*}

Assume \(\tvar{G} \in \{1\}^{8 \times 8}\) and \(\tvar{H} \in \{1\}^{8 \times 8}\), \figurename{} \ref{fig:oconv-2d-2d-2d} demonstrates
\begin{equation*}
  \left( \tvar{G} \oconv \lcerfl{\tvar{H}, \tvar{H}} \right)(t_1, t_2, t_3, t_4)
  = \sum_{\tau_1=0,\tau_2=0}^{7,7}
  G(\tau_1, \tau_2)
  H(t_1 - \tau_1, t_2 - \tau_1)
  H(t_3 - \tau_2, t_4 - \tau_2).
\end{equation*}
It is clear that \(\tvar{G} \oconv \lcerfl{\tvar{H}, \tvar{H}}\) is a four-dimensional tensor with shape \((15, 15, 15, 15)\), and it is flattened to shape \((225, 225)\) and drawn in Figure \ref{fig:oconv-2d-2d-2d}.

\begin{figure}[htb]
  \centering
  \subfloat[\(\tvar{G} \oconv \lcerfl{\tvar{h}, \tvar{h}}\)]{
    \includegraphics[width=0.22\linewidth]{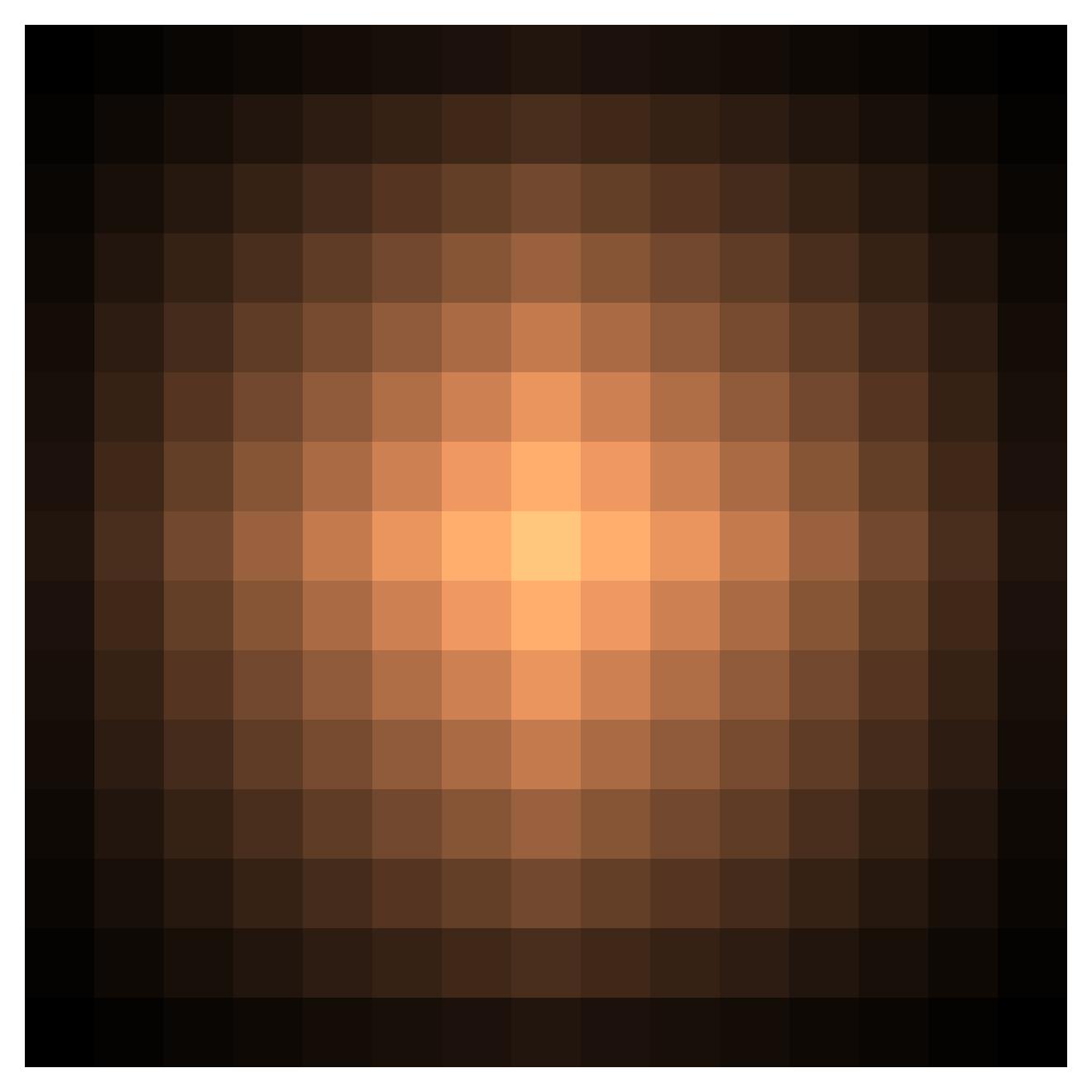} %
    \label{fig:oconv-2d-1d-1d}
  }
  \subfloat[\(\tvar{g} \oconv \tvar{H}\)]{
    \includegraphics[width=0.22\linewidth]{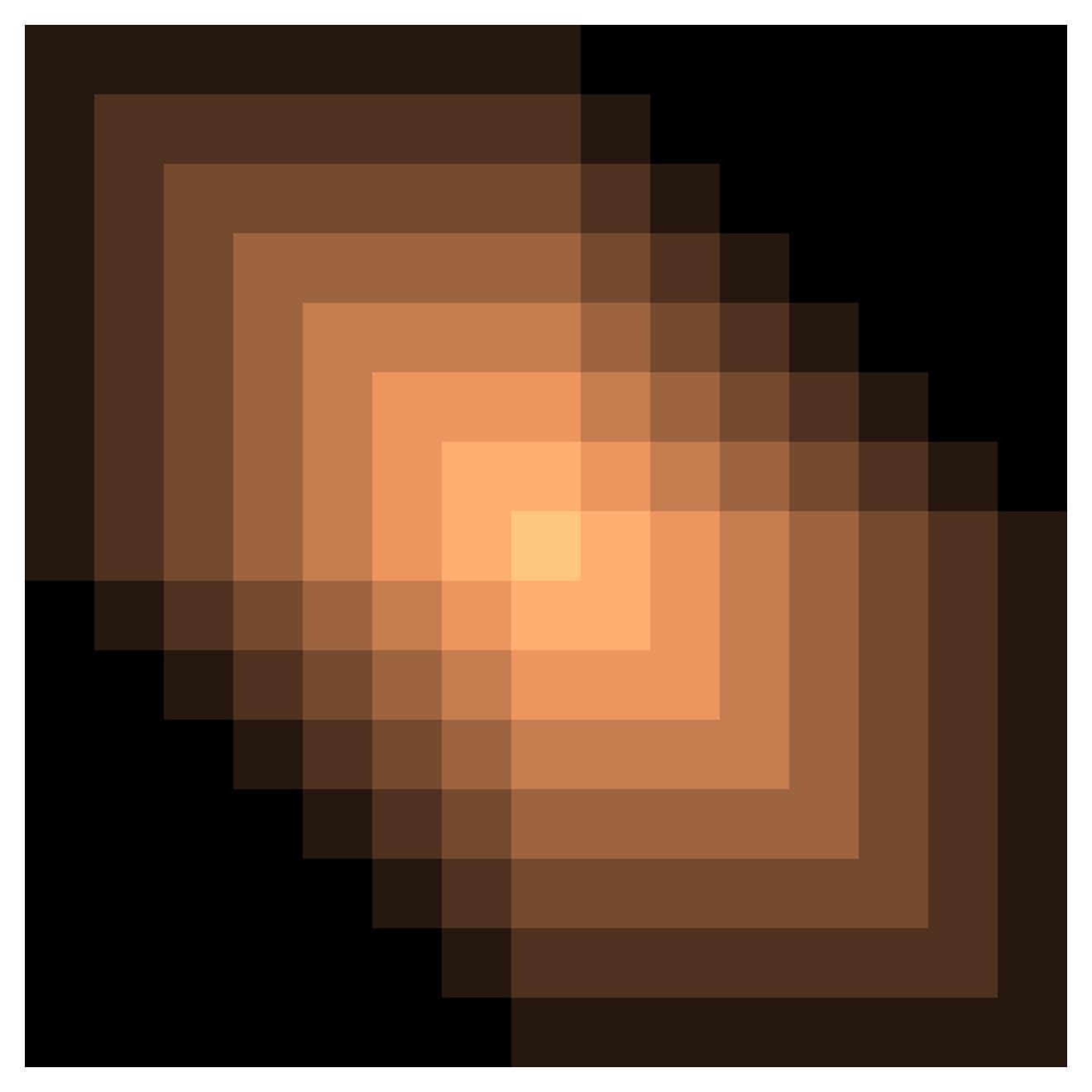} %
    \label{fig:oconv-1d-2d}
  }
  \subfloat[flattened \(\tvar{G} \oconv \lcerfl{\tvar{H}, \tvar{H}}\)]{
    \includegraphics[width=0.22\linewidth]{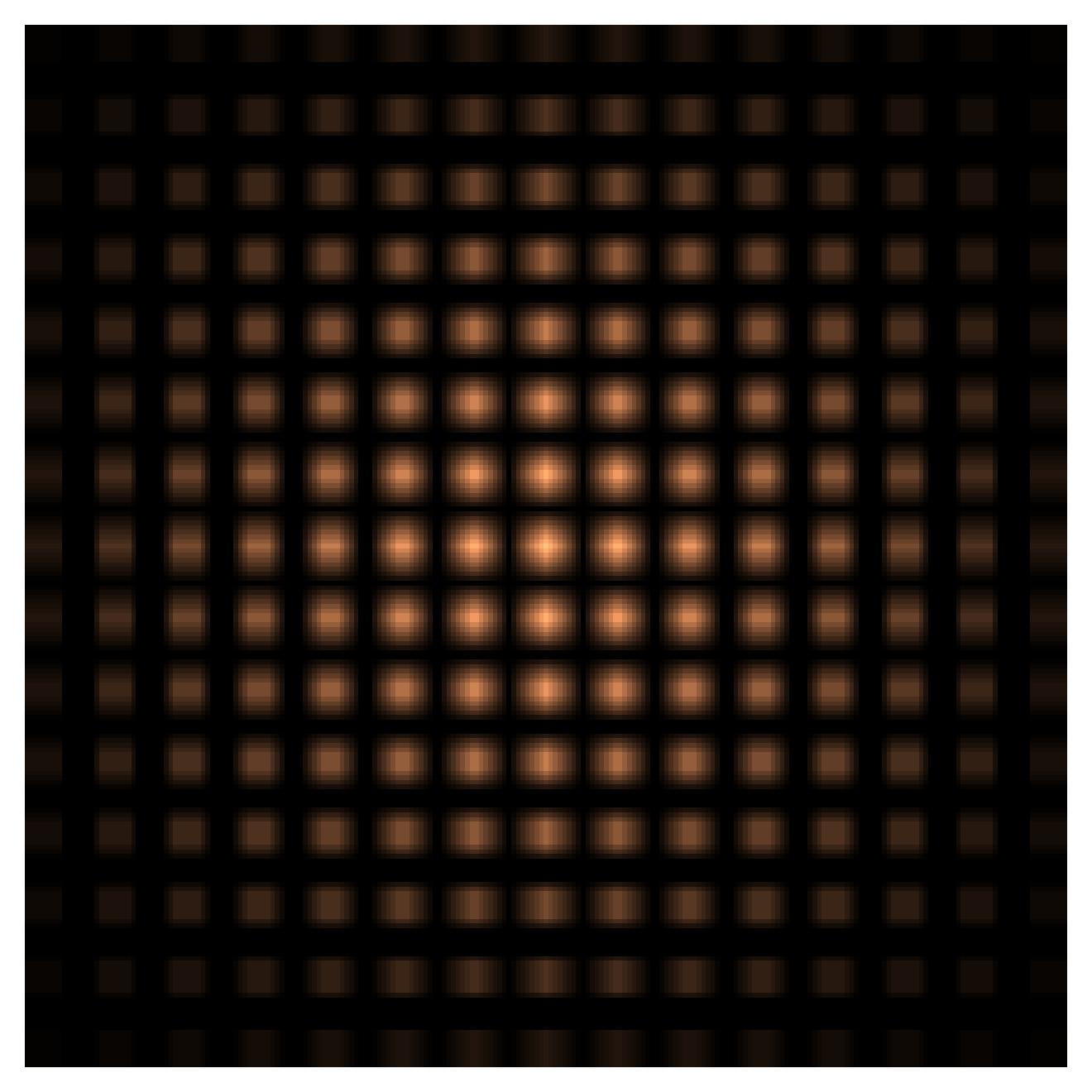}
    \label{fig:oconv-2d-2d-2d}
  }
  \caption{
    Examples of outer Convolution.
  }
\end{figure}

\subsection{Convolution for Multi-dimensional Signals}
\label{subsec:convolution-for-multi-dimensional-signals}

In this subsection, we show that the multidimensional (outer) convolution can be analyzed via one-dimensional (outer) convolution.
With this transformation, only (outer) convolution for one-dimensional signals will be studied hereafter, if not specified.

\begin{definition}[Flatten-operator]
  The flatten-operator is a bijection \(\mathcal{T}(\tvar{x}) = \tvarhat{x}\), such that \(x(t_1, t_2, \cdots) = \hat{x}(t_1 w_1 + t_2 w_2 + \cdots)\) and \(w_1, w_2\) are scalars for locating non-overlap elements. The inverse of \(\mathcal{T}\) is denoted by \(\mathcal{T}^{-1}\) with \(\mathcal{T}^{-1}\left( \mathcal{T}(\tvar{x}) \right) = \tvar{x}\).
  \label{def:flatten-operator}
\end{definition}

\begin{proposition}
  The flatten-operator is homomorphic
  \begin{equation}
    \begin{aligned}
      \mathcal{T}(\tvar{H} * \vsymb{x})
       & = \mathcal{T}(\tvar{H}) * \mathcal{T}(\vsymb{x}),      \\
      \mathcal{T}( \tvar{G} \oconv \vsymb{H} )
       & = \mathcal{T}(\tvar{G}) \oconv \mathcal{T}(\vsymb{H}).
    \end{aligned}
  \end{equation}
  \label{proposition:flatten-operator-is-homorphic}
\end{proposition}

\begin{proof}
  A continuous case of this proposition is proved here, and the discrete one can be obtained in similar way.
  Recall order-\(n\) convolution for \(m\)-dimensional signal (Equation \ref{equ:def-discrete-order-n-convolution-md}), and let flattened index \(\vsymb{\tau}_i\) be \(\upsilon_i\), flattened index \(\vsymb{t}\) be \(\iota\). We have
  \begin{equation*}
    \begin{aligned}
      \mathcal{T}(\tvar{H} * \vsymb{x})(\vsymb{t})
       & = \mathcal{T}\left(\sum_{\vsymb{\tau}_i, \cdots, \vsymb{\tau}_n} H(\vsymb{\tau}_1, \vsymb{\tau}_2, \cdots, \vsymb{\tau}_n) \prod_{i=1}^{n} x_i(\vsymb{t} - \vsymb{\tau}_i) \right) \\
       & = \sum_{\upsilon_1, \cdots, \upsilon_n} \hat{H}(\upsilon_1, \upsilon_2, \cdots, \upsilon_n) \prod_{i=1}^{n} \hat{x}_i(\iota - \upsilon_i)                                          \\
       & = \left(\mathcal{T}(\tvar{H}) * \mathcal{T}(\vsymb{x})\right)(\iota).
    \end{aligned}
  \end{equation*}
  The flatten-operator of outer convolution for \(m\)-dimensional signal is also homomorphic.
  Recall outer convolution for \(m\)-dimensional signal (Equation \ref{equ:def-discrete-volterra-convolution-md}), and let flattened index  \(\vsymb{t}_{i,j}\) be \(\iota_{i,j}\), flattened index  \(\vsymb{\tau}_{i}\) be \(\upsilon_{i}\). We have
  \begin{equation*}
    \begin{aligned}
       & \mathcal{T}(\tvar{G} \oconv \vsymb{H})(\vsymb{t}_{1,1}, \vsymb{t}_{1,2}, \cdots;, \vsymb{t}_{2,1}, \cdots; \cdots; \vsymb{t}_{n,1}, \cdots)               \\
       & = \sum_{\vsymb{\tau}_1, \cdots, \vsymb{\tau}_n} G(\vsymb{\tau}_1, \vsymb{\tau}_2, \cdots, \vsymb{\tau}_n)
      \prod_{i=1}^{n}  H_i (\vsymb{t}_{i,1} - \vsymb{\tau}_i, \vsymb{t}_{i,2} - \vsymb{\tau}_i, \cdots)                                                            \\
       & = \sum_{\upsilon_1, \cdots, \upsilon_n} \hat{G}(\upsilon_1, \upsilon_2, \cdots, \upsilon_n)
      \prod_{i=1}^{n} \hat{H}_i (\iota_{i,1} - \upsilon_1, \iota_{i,2} - \upsilon_1, \cdots)                                                                       \\
       & = \left(\mathcal{T}(\tvar{G}) \oconv \mathcal{T}(\vsymb{H}) \right) (\iota_{1,1}, \iota_{1,2}, \cdots; \iota_{2,1}, \cdots; \cdots; \iota_{n,1}, \cdots).
    \end{aligned}
  \end{equation*}
\end{proof}

With Proposition \ref{proposition:flatten-operator-is-homorphic}, it could be easily verified
\begin{equation*}
  \begin{aligned}
    \tvar{H} * \vsymb{x}      & = \mathcal{T}^{-1} \left( \mathcal{T}(\tvar{H}) * \mathcal{T}(\vsymb{x}) \right),      \\
    \tvar{G} \oconv \vsymb{H} & = \mathcal{T}^{-1} \left( \mathcal{T}(\tvar{G}) \oconv \mathcal{T}(\vsymb{H}) \right).
  \end{aligned}
\end{equation*}
To help visualize this process, Figure \ref{fig:flatten-nd-convolution-to-1d-convolution} is an example of flattening a two-dimensional signal to a one-dimensional signal.

\begin{figure}[htb]
  \centering
  \begin{tikzpicture}[scale=0.4]
    \fill[black!15, rounded corners=0.2cm]
    (-5,5) ++(0,-1) -- ++(0,1) -- ++(3,0) [rounded corners=0cm] -- ++(0,-1) ++(-3,0);
    \fill[black!30] (-5,4) rectangle ++(3,-1);
    \fill[black!45, rounded corners=0.2cm]
    (-5,3) -- ++(0,-1) -- ++(3,0) [rounded corners=0cm] -- ++(0,1) ++(0,-3);
    \draw[step=1, gray] (-6.5, 6.5) grid (-0.5, 0.5);
    \draw[rounded corners=0.2cm] (-5,5) rectangle (-2,2);

    \begin{scope}[rounded corners=0.2cm]
      \draw[fill=black!15] (-8,0) rectangle ++(3,-1);
      \draw[black!30] (-7,0) -- (-7,-1) (-6,0) -- (-6,-1);
      \draw[fill=black!30] (-1,0) rectangle ++(3,-1);
      \draw[black!45] (0, 0) -- (0, -1) (1, 0) -- (1, -1);
      \draw[fill=black!45] (6,0) rectangle ++(3,-1);
      \draw[black!65] (7, 0) -- (7, -1) (8, 0) -- (8, -1);
    \end{scope}

    \draw[gray, double distance=3, arrows = {-Stealth[length=12pt,open]}, rounded corners=0.5cm] (0,3.5) -- ++(2, 0) -- ++(0, -3);

    \draw[gray] (-10,0) -- (11,0)
    (-10,-1) -- (11,-1);

    \begin{scope}[gray]
      \draw[rotate around={60:(-9,-0.5)}, double distance=2]
      (-9,-0.5) ++(-1,-0.2) -- ++(1, 0) -- ++(0,0.4) -- ++(1,0);
      \draw[rotate around={60:(-3,-0.5)}, double distance=2]
      (-3,-0.5) ++(-1,-0.2) -- ++(1, 0) -- ++(0,0.4) -- ++(1,0);
      \draw[rotate around={60:(4,-0.5)}, double distance=2]
      (4,-0.5) ++(-1,-0.2) -- ++(1, 0) -- ++(0,0.4) -- ++(1,0);
      \draw[rotate around={60:(10,-0.5)}, double distance=2]
      (10,-0.5) ++(-1,-0.2) -- ++(1, 0) -- ++(0,0.4) -- ++(1,0);
    \end{scope}

  \end{tikzpicture}
  \caption{Flattening a two-dimensional signal to a one-dimensional signal.}
  \label{fig:flatten-nd-convolution-to-1d-convolution}
\end{figure}
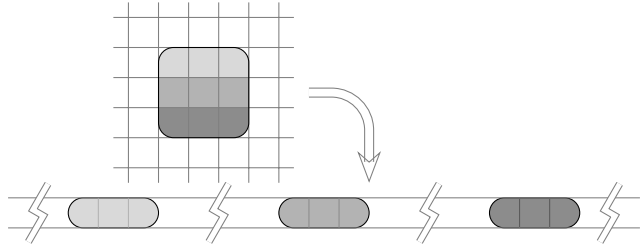

\subsection{Combination Properties}
\label{subsec:convolution-properties}

With all definitions above, some useful properties are concluded in this subsection.
All proofs are presented in Appendix \ref{appendix:proof-for-combination-properties}.

Since the multidimensional signals can be analyzed via one-dimensional signals, all signals here are set to be \textbf{one-dimensional} and represented as \(\tvar{x}\) or \(\tvar{y}\). Kernels are represented as \(\tvar{H}\) and \(\tvar{G}\). \(\alpha\) is a scalar.

The following Property \ref{prop:conv-g-plus-x-y}, \ref{prop:conv-g-add-x-a} and \ref{prop:conv-plus-g-h-x} are about linearity of order-\(n\) convolutions.
In continuous time, notation \(\sum \tvar{G}\) is replaced with \(\isintinf G(\vsymb{t}) d \vsymb{t}\).

\begin{enumerate}
  \item \label{prop:conv-g-plus-x-y} %
        \(\tvar{G} * (\vsymb{x} + \vsymb{y}) = \tvar{G} * \vsymb{x} + \tvar{G} * \vsymb{y}\);
  \item \label{prop:conv-g-add-x-a} %
        \(\tvar{G} * (\vsymb{x} + \alpha) = \tvar{G} * \vsymb{x} + \alpha \sum \tvar{G}\);
  \item \label{prop:conv-plus-g-h-x} %
        \((\tvar{G} + \tvar{H}) * \vsymb{x} = \tvar{G} * \vsymb{x} + \tvar{H} * \vsymb{x}\);
\end{enumerate}

Property \ref{prop:conv-g-square-plus-a-b}, \ref{prop:conv-g-power-n-plus-a-b} and \ref{prop:conv-g-power-n-plus-a-c-b-d} are combination properties of order-\(n\) convolution. Based on previous discussion, the kernel \(\tvar{G}\) here is symmetric.
The multinomial coefficient in Property \ref{prop:conv-g-power-n-plus-a-c-b-d} can be obtained from textbook of combinatorial mathematic \citep{brualdi_2004}.

\begin{enumerate}[resume]
  \item \label{prop:conv-g-square-plus-a-b} %
        \(\tvar{G} * (\tvar{x} + \tvar{y})^2 = \tvar{G} * \tvar{x}^2 + 2 \tvar{G} * \lcerfl{\tvar{x}, \tvar{y}} + \tvar{G} * \tvar{y}^2\);
  \item \label{prop:conv-g-power-n-plus-a-b} %
        \(\tvar{G} * (\tvar{x} + \tvar{y})^n = \sum_{k=0}^{n} \binom{n}{k} \tvar{G} * \lcerfl{\tvar{x}^k, \tvar{y}^{n-k}}\), where \(\binom{n}{k} = \dfrac{n!}{k! (n-k)!}\) is binomial coefficient;
  \item \label{prop:conv-g-power-n-plus-a-c-b-d}%
        \(\tvar{G} * (\tvar{x}_1 + \tvar{x}_2 + \cdots + \tvar{x}_m)^n = \sum \binom{n}{n_1 n_2 \cdots n_m} \tvar{G} * \lcerfl{\tvar{x}_1^{n_1}, \tvar{x}_2^{n_2}, \cdots, \tvar{x}_m^{n_m}} \), where multinomial coefficient \(\binom{n}{n_1 n_2 \cdots n_m} = \dfrac{n!}{n_1! n_2! \cdots n_m!}\), and \(\sum_{i=1}^{m} n_i = m, n_i \ge 0\), for all \(i = 1, 2, \cdots, m\).
\end{enumerate}

Properties below are for combining convolutions.
Symbol ``\(\dsmark\)'' indicates summation along a specific dimension.
For example, \(\sum_{\dsmark \alpha} \left(\tvar{G} \oconv \lcerfl{\tvar{H}, \alpha}\right) = \sum_i (\tvar{G} \oconv \lcerfl{\tvar{H}, \alpha})(:,i)\) and \(\sum_{\dsmark \alpha} \left(\tvar{G} \oconv \lcerfl{\tvar{H}_1, \alpha, \tvar{H}_2}\right) = \sum_i (\tvar{G} \oconv \lcerfl{\tvar{H}_1, \alpha, \tvar{H}_2})(:,i,:)\). In continuous space, they are replaced by \(\intinf \left(\tvar{G} \oconv \lcerfl{\tvar{H}, \alpha}\right)(:,t_{\alpha}) d t_{\alpha}\) and \(\intinf \left(\tvar{G} \oconv \lcerfl{\tvar{H}_1, \alpha, \tvar{H}_2}\right)(:, t_{\alpha}, :) d t_{\alpha}\).

\begin{enumerate}[resume]
  \item \label{prop:conv-g-conv-h-x} %
        \(\tvar{G} * (\tvar{H} * \vsymb{x}) = (\tvar{G} \oconv \tvar{H}) * \vsymb{x}\);
  \item \label{prop:conv-stride-s-g-conv-stride-z-h-x}
        \(\tvar{G} *_s (\tvar{H} *_z \tvar{x}) = (\tvar{G} \oconv_z \tvar{H}) *_{sz} \tvar{x}\);
  \item \label{prop:conv-g-mul-conv-h1-x-conv-h2-y} %
        \(\tvar{G} * \lcerfl{\tvar{H}_1 * \vsymb{x}, \tvar{H}_2 * \vsymb{y}} = (\tvar{G} \oconv \lcerfl{\tvar{H}_1, \tvar{H}_2}) * \lcerfl{\vsymb{x}, \vsymb{y}}\);
  \item \label{prop:conv-g-prod-n-conv-h-x} %
        \(\tvar{G} * \lcerfl{\tvar{H}_1 * \vsymb{x}_1, \tvar{H}_2 * \vsymb{x}_2, \cdots} = (\tvar{G} \oconv \lcerfl{\tvar{H}_1, \tvar{H}_2, \cdots}) * \lcerfl{\vsymb{x}_1, \vsymb{x}_2, \cdots}\);
  \item \label{prop:oconv-g1-oconv-g2-g3} %
        \(\tvar{G}_1 \oconv (\tvar{G}_2 \oconv \tvar{G}_3) = (\tvar{G}_1 \oconv \tvar{G}_2) \oconv \tvar{G}_3\);
  \item \label{prop:conv-g-conv-alpha-conv-h-x} %
        \(\tvar{G} * \lcerfl{\alpha, \tvar{H} * \vsymb{x}} = \left(\sum_{\dsmark \alpha} \left(\tvar{G} \oconv \lcerfl{\alpha, \tvar{H}}\right) \right) * \vsymb{x}\);
  \item \label{prop:conv-g-conv-h-conv-alpha-x}%
        \(\tvar{G} * \lcerfl{\tvar{H} * \vsymb{x}, \alpha} = \left(\sum_{\dsmark \alpha} \left(\tvar{G} \oconv \lcerfl{\tvar{H}, \alpha}\right) \right) * \vsymb{x}\);
  \item \label{prop:conv-g-conv-conv-h1-x-alpha-conv-h2-y} %
        \(\tvar{G} * \lcerfl{\tvar{H}_1 * \vsymb{x}, \alpha, \tvar{H}_2 * \vsymb{y}}
        = \left(\sum_{\dsmark \alpha} \left(\tvar{G} \oconv \lcerfl{\tvar{H}_1, \alpha, \tvar{H}_2}\right) \right) * \lcerfl{\vsymb{x}, \vsymb{y}}\);
\end{enumerate}

Property \ref{prop:conv-h-element-power-n-x} and \ref{prop:conv-g-power-conv-h-x} focus on convolution with signal that is element-wise power \(n\), \(\left( [\tvar{x}]^n \right)(t) = \left(x(t)\right)^n\).

\begin{enumerate}[resume]
  \item \label{prop:conv-h-element-power-n-x}
        \(\tvar{h} * [\tvar{x}]^n = \diag(n, \tvar{h}) * \tvar{x}^n\);
  \item \label{prop:conv-g-power-conv-h-x}
        \(\tvar{g} * [\tvar{h} * \tvar{x}]^n = \left(\diag(n, \tvar{g}) \oconv \tvar{h}^n\right) * \tvar{x}^n\);
\end{enumerate}

In addition to all properties above, some special properties could be obtained via setting special kernels. For example, by setting \(\tvar{G}\) as the identity matrix, we have
\begin{equation*}
  \left(
  \tvar{G} \oconv \lcerfl{\tvar{H}_1, \tvar{H}_2}
  \right)(\vsymb{t}_1, \vsymb{t}_2)
  = \sum_{i_1, i_2} G(i_1, i_2) H_1(\vsymb{t}_1 - i_1) H_2(\vsymb{t}_2 - i_2)
  = H_1(\vsymb{t}_1) H_2(\vsymb{t}_2).
\end{equation*}

\section{Transformation from Neural Networks to Volterra Convolutions}
\label{sec:connect-to-neural-network}

Previous section discussed the definition of Volterra convolution and some useful properties of combining two order-\(n\) convolutions.
In this section, we will go further and try to represent some common convolutional networks in the form of Volterra convolutions.

\begin{theorem}
  Most convolutional neural networks can be represented in the form of Volterra convolutions.
  \label{thm:conv-net-to-vconv}
\end{theorem}

Both convolutional networks and Volterra convolutions have the operation of convolution.
The convolutional neural network extend this operation by stacking layers and the Volterra convolution extend this by increasing the order.
Apart from the convolution, a convolutional neural network is a universal approximator, as it happens, a Volterra convolution is also a universal approximator.
Theoretically, if two approximators can approximate the same function, it is possible to approximate one by the other. In light of this, roughly speaking, most convolutional neural network can be approximated in the form of Volterra convolution, and vice versa.

The proof contains two major parts. The first part is about the small neural network structures, and the second part is about the combination of multiple layers, i.e., the whole network.
Since both the small structures and their combinations can be represented in this form, we conclude that most convolutional neural networks build of these structures can also be represented in the form of Volterra convolution.

\subsection{Conversion of Small Structures}

\subsubsection{Conv-Act-Conv Structure}
\label{subsubsec:combine-layers}

The ``conv | act | conv'' structure means the stacking of a convolutional layer, an activation layer, and a convolutional layer.

\begin{lemma}
  The ``conv | act | conv'' structure can be converted to the form of Volterra convolution. 
  \label{lemma:convolution-with-activation-is-Volterra}
\end{lemma}

\begin{proof}
  Suppose this structure has the form \(\tvar{g} * \sigma(\tvar{h} * \tvar{x})\), where \(\sigma(\cdot)\) is a nonlinear activation function.

  A polynomial approximation, i.e., Taylor expansion, of this activation function \(\sigma(t)\) at \(\alpha\) is
  \begin{equation*}
    \sigma(t) = \sigma(\alpha) + \sigma'(\alpha) (t - \alpha) + \dfrac{\sigma''(\alpha)}{2!} (t - \alpha)^2 + \dfrac{\sigma'''(\alpha)}{3!} (t - \alpha)^3 + \cdots.
  \end{equation*}

  We can assume without loss of generality that \(\alpha = 0\), and \(\tvar{g} * \sigma(\tvar{h} * \tvar{x})\) becomes
  \begin{equation}
    \tvar{g} * \left(
    \sigma(0) + \sigma'(0) [\tvar{h} * \tvar{x}] +
    \dfrac{\sigma''(0)}{2!} [\tvar{h} * \tvar{x}]^2                \\
    + \dfrac{\sigma'''(0)}{3!} [\tvar{h} * \tvar{x}]^3 + \cdots
    \right),
    \label{equ:proof-expand-conv-g-polynomial-conv-h-x}
  \end{equation}
  where square brackets stand for
  \begin{equation*}
    [\tvar{h} * \tvar{x}]^n(t) = \left( \sum_{\tau} h(\tau) x(t - \tau) \right)^n.
  \end{equation*}

  By linearity of convolution, Equation \ref{equ:proof-expand-conv-g-polynomial-conv-h-x} can be separated by terms.
  The first term is \(\sum \tvar{g}\).
  The second term is \(\tvar{g} * [\tvar{h} * \tvar{x}]^1 = \left(\tvar{g} \oconv \tvar{h}\right) * \tvar{x}\).
  For the third term and above, recall Property \ref{prop:conv-g-power-conv-h-x}, we have the \(n\)-th term,
  \begin{equation*}
    \tvar{g} * [\tvar{h} * \tvar{x}]^n
    = \left(\diag(n, \tvar{g}) \oconv \tvar{h}^n \right) * \tvar{x}^n.
  \end{equation*}

  If \(\alpha = 0\), we conclude that
  \begin{equation}
    \begin{aligned}
      \tvar{g} * \sigma(\tvar{h} * \tvar{x})
       & = \sigma(0) \sum \tvar{g}
      + \sigma'(0) (\tvar{g} \oconv \tvar{h}) * \tvar{x} \\
       & \!\!
      + \dfrac{\sigma''(0)}{2!} (\diag(2, \tvar{g}) \oconv \tvar{h}^2) * \tvar{x}^2
      + \dfrac{\sigma'''(0)}{3!} (\diag(3, \tvar{g}) \oconv \tvar{h}^3) * \tvar{x}^3
      + \cdots.
    \end{aligned}
    \label{equ:conv-act-conv-to-vconv-expand}
  \end{equation}

  More generally, if \(\alpha \ne 0\), the \(n\)-th term is
  \begin{equation*}
    \begin{aligned}
       & \tvar{g} * [\tvar{h} * \tvar{x} - \alpha]^n                                                                                 \\
       & = \diag(n, \tvar{g}) * \left(\tvar{h} * \tvar{x} - \alpha\right)^n ~~ (\text{Property \ref{prop:conv-h-element-power-n-x}}) \\
       & = \sum_{k=0}^{n} \binom{n}{k} \diag(n, \tvar{g}) * \lcerfl{(\tvar{h} * \tvar{x})^k,  (- \alpha)^{n-k}}
      ~~ (\text{Property \ref{prop:conv-g-power-n-plus-a-b}})                                                                        \\
       & = \sum_{k=0}^{n} \binom{n}{k}
      \left( \sum_{\dsmark (-\alpha)^{n-k}} \left(\diag(n, \tvar{g}) \oconv \lcerfl{\tvar{h}^k, (-\alpha)^{n-k}} \right)  \right) * \tvar{x}^k
      ~~ (\text{Property \ref{prop:conv-g-conv-h-conv-alpha-x}}),                                                                    \\
    \end{aligned}
  \end{equation*}
  which implies that the sum from the \(0\)-th term to the \(\infty\)-th term is also the form of Volterra convolution.
  This proof is completed.
\end{proof}

The Taylor expansion of a function often has infinite terms. However, if small truncation errors are allowed in applications, we can truncate the infinite term Taylor expansion to a finite term Taylor expansion.
The idea of truncation can also be applied to the Volterra convolution.
According to the universal approximation property of Volterra convolution with fading memory \citep{Boyd1985}, for any given \(\epsilon > 0\), there always exists \(n\) such that
\begin{equation*}
  \left\| f(\tvar{x}) - \sum_{i=0}^{n} \tvar{F}_i * \tvar{x}^i \right\|_2 < \epsilon,
\end{equation*}
where \(f(\cdot)\) is a time invariant operation, and \(\tvar{F}_i, i = 0, 1, \cdots n\) are kernels.
The fading memory theory means that the outputs are close if two inputs are close in the recent past, but not necessarily close in the remote past \citep{Boyd1985}.
If \(\tvar{g} * \sigma(\tvar{h} * \tvar{x})\) is time invariant and small truncation errors are allowed, it is reasonable to approximate \(\tvar{g} * \sigma(\tvar{h} * \tvar{x})\) by a finite term Volterra convolution.

In the following, we will consider the width of a two-layer network.
Suppose a two-layer network has \(M\) hidden neurons, and the neurons of distinct channels are independent and identically distributed,
\begin{equation*}
  y(t) = \sum_{i=0}^{M-1} W_2(i) \sigma\left(\sum_{\tau} W_1(i,\tau) x(t-\tau)\right).
\end{equation*}

Recall Lemma \ref{lemma:convolution-with-activation-is-Volterra}, each channel can be represented in the form of Volterra convolution, and this network can also be approximated by the sum of \(M\) Volterra convolutions, the order-\(n\) term is
\begin{equation*}
  \sum_{n=0} \tvar{F}_{0,n} * \tvar{x}^n +
  \sum_{n=0} \tvar{F}_{1,n} * \tvar{x}^n +
  \cdots +
  \sum_{n=0} \tvar{F}_{M-1,n} * \tvar{x}^n,
\end{equation*}
where \(\tvar{F}_{i,n}, ~ i=0, 1, \cdots, M-1;\) is the proxy kernel of each channel.

Based on the previous independent assumptions, these proxy kernels of different channels are also independent and identically.
Recall Hoeffding's inequality \citep{vershynin_2018}, we have
\begin{equation*}
  \mathbb{P}\left\{ \left| \dfrac{1}{M} \sum_{i=0}^{M-1} \tvar{F}_{i,n} - \mu_{n} \right| \ge \eta \right\}
  \le 2 \exp\left(\dfrac{-2 M \eta^2}{(b-a)^2}\right),
\end{equation*}
where \(a,b\) are the minimum and maximum values of all proxy kernels, \(a \le F_{i,n}(\cdots) \le b\).
We can observe that for any \(\epsilon \in (0, 1)\), \(1/M \sum_{i=0}^{M-1} \tvar{F}_{i,n}\) will converge to \(\mu_{n}\), with probability at least \(1 - \epsilon\) as long as
\begin{equation}
  \label{equ:upperbound-of-proxy-kernels}
  M \ge \dfrac{1}{2 \eta^2} \ln\left(\dfrac{2}{\epsilon}\right) (b-a)^2.
\end{equation}

\subsubsection{Other Structures}
\label{subsubsec:extension-on-other-Structures}

Some commonly used structures can also be represented in the form of Volterra convolution, including some activations, normalize layers, inception modules, residual connection and pooling layers.

\textit{ReLU activation:} The ReLU activation, \(\max(x,0)\) \citep{Nair2010}, is not differentiable at point \(0\). This is quite a panic when taking Taylor expansion at that position. Nonetheless, it can be approximated by
\begin{equation*}
  \text{ReLU}(x)
  = \lim_{\alpha \rightarrow \infty} \dfrac{1}{\alpha} \ln\left( 1 + e^{\alpha x} \right).
\end{equation*}
Other activations of the ReLU family can be approximated in the same way.

\textit{Fully connected layers:} A fully connected layer is a matrix multiplication with bias. It can be thought as discrete convolution with equal kernel length and signal length.
\begin{equation*}
  \sum_{j} W(i,j) x(j) + c = \sum_{j} \overline{W}(i,0-j) x(j) + c,
\end{equation*}
where \(\overline{W}(i, 0-j) = W(i,j)\).

\textit{Normalization:} The normalization layer \citep{Ioffe2015, Ba2016} scales and shifts the input signal
\begin{equation*}
  \tvar{x} \rightarrow a \tvar{x} + b.
\end{equation*}
This is a linear transformation, and this will not change the generality.
Nevertheless, \(a\) and \(b\) are input related, which implies that the corresponded proxy kernels are also input related.
In other words, the Volterra convolution is dynamic and will be updated if input differs.
We will pause here and left this dynamic Volterra convolution to future work.

\textit{Inception:} Main idea of inception module \citep{InceptionNet} is to apply convolution to different sizes of kernels parallelly and then concatenate, which is
\begin{equation*}
  \tvar{g} * \left( \tvar{h}_1 * \tvar{x} + \tvar{h}_2 * \tvar{x} + \cdots + \tvar{h}_n * \tvar{x}\right).
\end{equation*}
If we zero pad those kernels to the same size, recalling Property \ref{prop:conv-g-conv-h-x}, it becomes a convolutional layer,
\begin{equation*}
  \left( \tvar{g} \oconv (\tvar{h}_1 + \tvar{h}_2 + \cdots + \tvar{h}_n) \right) * \tvar{x}.
\end{equation*}

\textit{Residual connection:} A residual connection \citep{He2016} proposes \(f(\tvar{x}) + \tvar{x}\), where \(f(\cdot)\) is a neural network.
If \(f(\cdot)\) can be transformed to Volterra convolution, we have
\begin{equation*}
  \begin{aligned}
    f(\tvar{x}) + \tvar{x}
     & = \sum_{n=0}^{N} \tvar{H}_n * \tvar{x}^n + \tvar{x}                                                 \\
     & = \tvar{H}_0 + (\tvar{H}_1 * \tvar{x} + \delta * \tvar{x}) + \sum_{n=2}^{N} \tvar{H}_n * \tvar{x}^n \\
     & = \tvar{H}_0 + (\tvar{H}_1 + \delta) * \tvar{x} + \sum_{n=2}^{N} \tvar{H}_n * \tvar{x}^n,           \\
  \end{aligned}
\end{equation*}
where \(\delta\) is the Dirac delta and
\(\tvar{H}_1 + \delta
= \left\{\begin{array}{ll}
  H_1(t) + 1, & t = 0   \\
  H_1(t),     & t \ne 0
\end{array} \right. \).

\textit{Pooling:} Another family is the pooling layers. They are down sample or up sample operations.
Average pooling is convolution with kernel filled by one.
Max pooling picks the maximum value in a small region. It is data dependent, the equivalent kernels changes synchronously with input signal, and the proxy kernels are also dynamic.

\subsection{Conversion of Layer Combination}
\label{subsec:conversion-of-layer-combination}

In the previous subsection, we focus on small structures, and in this section, we combine Volterra convolution layers, showing that the combinations also have the same format.

\subsubsection{Order-Two-Order-Two Structure}
\label{subsubsec:transform-o2-o2-vconv}

Before going further, a simple structure is presented in this subsection.
The ``order-\(2\) | order-\(2\)'' structure is stacking two order-two Volterra convolutions.

\begin{lemma}
  The ``order-\(2\) | order-\(2\)'' structure can be converted to the form of order-\(4\) Volterra convolution.
  \label{lemma:stacking-o2-vconv-o2-vconv-approximate-o4-vconv}
\end{lemma}

\begin{proof}
  Suppose \(\tvar{H}_i, \tvar{G}_j\) are two groups of kernels, \(\tvar{x}, \tvar{y}, \tvar{z}\) are the signals, these two Volterra convolutions are
  \begin{equation*}
    \tvar{y} = \sum_{i=0}^{2} \tvar{H}_i * \tvar{x}^{i}, ~~~
    \tvar{z} = \sum_{j=0}^{2} \tvar{G}_j * \tvar{y}^{j}.
  \end{equation*}

  Combining these two convolutions, we have
  \begin{equation*}
    \tvar{z} = \sum_{j=0}^{2} \tvar{G}_j * \left( \tvar{H}_0 + \tvar{H}_1 * \tvar{x}^{1} + \tvar{H}_2 * \tvar{x}^{2} \right)^j.
  \end{equation*}

  The first term is \(\tvar{G}_0 * \tvar{y}^0  = \tvar{G}_0\).

  Recall Property \ref{prop:conv-g-plus-x-y} and \ref{prop:conv-g-conv-h-x}, the second term is
  \begin{equation*}
    \begin{aligned}
      \tvar{G}_1 * \tvar{y}^1
       & = \tvar{H}_0 \sum \tvar{G}_1  + \tvar{G}_1 * (\tvar{H}_1 * \tvar{x}) + \tvar{G}_1 * (\tvar{H}_2 * \tvar{x}^2)           \\
       & = \tvar{H}_0 \sum \tvar{G}_1 + (\tvar{G}_1 \oconv \tvar{H}_1) * \tvar{x} + (\tvar{G}_1 \oconv \tvar{H}_2) * \tvar{x}^2.
    \end{aligned}
  \end{equation*}

  Recall Property \ref{prop:conv-g-square-plus-a-b}, the last term is
  \begin{equation*}
    \begin{aligned}
      \tvar{G}_2 * \tvar{y}^2
       & = \tvar{G}_2 * \left(\tvar{H}_0 + \tvar{H}_1 * \tvar{x} + \tvar{H}_2 * \tvar{x}^2 \right)^2 \\
       & = \tvar{G}_2 * \lcerfl{\tvar{H}_0, \tvar{H}_0}
      + \tvar{G}_2 * \lcerfl{\tvar{H}_0, \tvar{H}_1 * \tvar{x}}
      + \tvar{G}_2 * \lcerfl{\tvar{H}_1 * \tvar{x}, \tvar{H}_0}                                      \\
       & ~~
      + \tvar{G}_2 * \lcerfl{\tvar{H}_0, \tvar{H}_2 * \tvar{x}^2}
      + \tvar{G}_2 * \lcerfl{\tvar{H}_2 * \tvar{x}^2, \tvar{H}_0}
      + \tvar{G}_2 * \lcerfl{\tvar{H}_1 * \tvar{x}, \tvar{H}_1 * \tvar{x}}                           \\
       & ~~
      + \tvar{G}_2 * \lcerfl{\tvar{H}_1 * \tvar{x}, \tvar{H}_2 * \tvar{x}^2}
      + \tvar{G}_2 * \lcerfl{\tvar{H}_2 * \tvar{x}^2, \tvar{H}_1 * \tvar{x}}                         \\
       & ~~
      + \tvar{G}_2 * \lcerfl{\tvar{H}_2 * \tvar{x}^2, \tvar{H}_2 * \tvar{x}^2}.                      \\
    \end{aligned}
  \end{equation*}
  And recall Property
  \ref{prop:conv-g-conv-h-x},
  \ref{prop:conv-g-mul-conv-h1-x-conv-h2-y},
  \ref{prop:conv-g-conv-alpha-conv-h-x},
  \ref{prop:conv-g-conv-h-conv-alpha-x},
  \ref{prop:conv-g-conv-conv-h1-x-alpha-conv-h2-y}, it becomes
  \begin{equation*}
    \begin{aligned}
      \tvar{G}_2 * \tvar{y}^2
       & = \tvar{H}^2_0 \sum \tvar{G}_2
      + \left(\sum_{\dsmark \tvar{H}_0} \tvar{G}_2 \oconv \lcerfl{\tvar{H}_0, \tvar{H}_1} \right) * \tvar{x}
      + \left(\sum_{\dsmark \tvar{H}_0} \tvar{G}_2 \oconv \lcerfl{\tvar{H}_1, \tvar{H}_0} \right) * \tvar{x}   \\
       & \!\!\!\!\!\!\!
      + (\tvar{G}_2 \oconv \lcerfl{\tvar{H}_1, \tvar{H}_1}) * \tvar{x}^2
      + \left(\sum_{\dsmark \tvar{H}_0} \tvar{G}_2 \oconv \lcerfl{\tvar{H}_0, \tvar{H}_2} \right) * \tvar{x}^2
      + \left(\sum_{\dsmark \tvar{H}_0} \tvar{G}_2 \oconv \lcerfl{\tvar{H}_2, \tvar{H}_0} \right) * \tvar{x}^2 \\
       & \!\!\!\!\!\!\!
      + (\tvar{G}_2 \oconv \lcerfl{\tvar{H}_2, \tvar{H}_1}) * \tvar{x}^3
      + (\tvar{G}_2 \oconv \lcerfl{\tvar{H}_1, \tvar{H}_2}) * \tvar{x}^3
      + (\tvar{G}_2 \oconv \lcerfl{\tvar{H}_2, \tvar{H}_2}) * \tvar{x}^4.
    \end{aligned}
  \end{equation*}

  In summary, we have
  \begin{equation*}
    \tvar{z} = \sum_{k=0}^{4} \tvar{F}_k * \tvar{x}^k,
  \end{equation*}
  where the new kernels are
  \begin{equation}
    \label{equ:combination-two-volterra-kernels}
    \begin{aligned}
      \tvar{F}_0 & = \tvar{G}_0 + \tvar{H}_0 \sum \tvar{G}_1 + \tvar{H}_0^2 \sum \tvar{G}_2 \\
      \tvar{F}_1 & = \tvar{G}_1 \oconv \tvar{H}_1
      + \sum_{\dsmark \tvar{H}_0} \tvar{G}_2 \oconv \lcerfl{\tvar{H}_0, \tvar{H}_1}
      + \sum_{\dsmark \tvar{H}_0} \tvar{G}_2 \oconv \lcerfl{\tvar{H}_1, \tvar{H}_0}         \\
      \tvar{F}_2 & = \tvar{G}_1 \oconv \tvar{H}_2
      + \tvar{G}_2 \oconv \lcerfl{\tvar{H}_1, \tvar{H}_1}
      + \sum_{\dsmark \tvar{H}_0} \tvar{G}_2 \oconv \lcerfl{\tvar{H}_0, \tvar{H}_2}
      + \sum_{\dsmark \tvar{H}_0} \tvar{G}_2 \oconv \lcerfl{\tvar{H}_2, \tvar{H}_0}         \\
      \tvar{F}_3 & = \tvar{G}_2 \oconv \lcerfl{\tvar{H}_1, \tvar{H}_2}
      + \tvar{G}_2 \oconv \lcerfl{\tvar{H}_2, \tvar{H}_1}                                   \\
      \tvar{F}_4 & = \tvar{G}_2 \oconv \lcerfl{\tvar{H}_2, \tvar{H}_2}.
    \end{aligned}
  \end{equation}
\end{proof}

Notice that it is possible for some \(i\) such that \(\|\tvar{F}_i * \tvar{x}^i \| = 0\), but we will also call it the form of order-four Volterra convolution.

\subsubsection{Order-n-order-m Structure}
\label{subsubsec:transform-on-om-vconv}

In the following, we generalized ``order-2 | order-2'' structure (Lemma \ref{lemma:stacking-o2-vconv-o2-vconv-approximate-o4-vconv}) to ``order-\(n\) | order-\(m\)'' structure. This structure is to stack order-\(n\) Volterra convolution and order-\(m\) Volterra convolution.

\begin{lemma}
  The ``order-\(n\) | order-\(m\)'' structure with \(n,m > 0\) can be converted to the form of order-\(nm\) Volterra convolution.
  \label{lemma:stacking-on-vconv-om-vconv-approximate-onm-vconv}
\end{lemma}

\begin{proof}
  Suppose \(\tvar{H}_i, \tvar{G}_j\) are two groups of kernels, \(\tvar{x}, \tvar{y}, \tvar{z}\) are signals, these two Volterra convolutions have the form of
  \begin{equation*}
    \tvar{y} = \sum_{i=0}^{n} \tvar{H}_i * \tvar{x}^i, ~~~~
    \tvar{z} = \sum_{j=0}^{m} \tvar{G}_j * \tvar{y}^j.
  \end{equation*}

  Combining these two Volterra convolutions, we have
  \begin{equation*}
    \tvar{z} = \sum_{j=0}^{m} \tvar{G}_j * \left(\sum_{i=0}^{n} \tvar{H}_i * \tvar{x}^i\right)^j.
  \end{equation*}

  Recall Property \ref{prop:conv-g-power-n-plus-a-c-b-d}, we have
  \begin{equation*}
    \tvar{z}
    = \sum_{j=0}^{m} \sum \binom{j}{j_0 j_1 \cdots j_n} \tvar{G}_j * \lcerfl{
      (\tvar{H}_0 * \tvar{x}^0)^{j_0},
      (\tvar{H}_1 * \tvar{x}^1)^{j_1},
      \cdots,
      (\tvar{H}_n * \tvar{x}^n)^{j_n}
    },
  \end{equation*}
  where \(\binom{j}{j_0 j_1 \cdots j_n} = \dfrac{j!}{j_0! j_1! \cdots j_n!}\) is multinomial coefficient and \(\sum_{k=0}^{n} j_k = j, j_k \ge 0\), for all \(k = 0, 1, \cdots, j\).

  Recall Property \ref{prop:conv-g-prod-n-conv-h-x}, we have
  \begin{equation*}
    \begin{aligned}
       & \tvar{G}_j * \lcerfl{
        (\tvar{H}_0 * \tvar{x}^0)^{j_0},
        (\tvar{H}_1 * \tvar{x}^1)^{j_1},
        \cdots,
        (\tvar{H}_n * \tvar{x}^n)^{j_n}
      }                                                                                                           \\
       & = \left( \tvar{G}_j \oconv \lcerfl{\tvar{H}_0^{j_0}, \tvar{H}_1^{j_1}, \cdots, \tvar{H}_n^{j_n}} \right)
      * \lcerfl{\tvar{x}^0, \tvar{x}^{j_1}, \tvar{x}^{2 j_2} \cdots, \tvar{x}^{n j_n}}                            \\
       & = \left( \tvar{G}_j \oconv \lcerfl{\tvar{H}_0^{j_0}, \tvar{H}_1^{j_1}, \cdots, \tvar{H}_n^{j_n}} \right)
      * \tvar{x}^{j_1 + 2 j_2 + \cdots + n j_n}.                                                                  \\
    \end{aligned}
  \end{equation*}

  In conclusion, the combination is
  \begin{equation}
    \tvar{z}
    = \sum_{j=0}^{m} \left( \sum \binom{j}{j_0 j_1 \cdots j_n}
    \left( \tvar{G}_j \oconv \lcerfl{\tvar{H}_0^{j_0}, \tvar{H}_1^{j_1}, \cdots, \tvar{H}_n^{j_n}} \right)
    * \tvar{x}^{j_1 + 2 j_2 + \cdots + n j_n} \right).
    \label{equ:transform-on-om-structure-combined}
  \end{equation}

  Clearly, \(j_1 + 2 j_2 + \cdots n j_n\) is sequential values from \(0\) to \(nm\), which implies that this structure can also be converted to the form of order-\(nm\) Volterra convolution.
\end{proof}

This conversion does not convince that all terms are non-zero. It is possible that the \(nm\) term is zero, \(\|\tvar{F}_{nm} * \tvar{x}^{nm} \| = 0\), where \(\tvar{F}_{nm}\) is the proxy kernel. To keep the coherent, we would prefer to call it the form of order-\(nm\) Volterra convolution.

Apart from the order, we also care about the number of terms, \(\tvar{G}_j \oconv \lcerfl{\tvar{H}_0^{j_0}, \tvar{H}_1^{j_1}, \cdots, \tvar{H}_n^{j_n}}\), that added to a proxy kernel.
For a given order \(o, 0 \le p \le nm\), how many combinations of \(j_0, j_1, \cdots, j_n, \sum_{k=0}^{n} j_k = 0, 1, \cdots, m\) such that \(j_1 + 2j_2 + \cdots + nj_n = o\)? This number can be obtained by counting the terms in Equation \ref{equ:transform-on-om-structure-combined}. We plot a figure for some combinations of \(n,m\) in Figure \ref{fig:conv-multinomial}.

\begin{figure}[htb]
  \centering
  \includegraphics[width=0.9\linewidth]{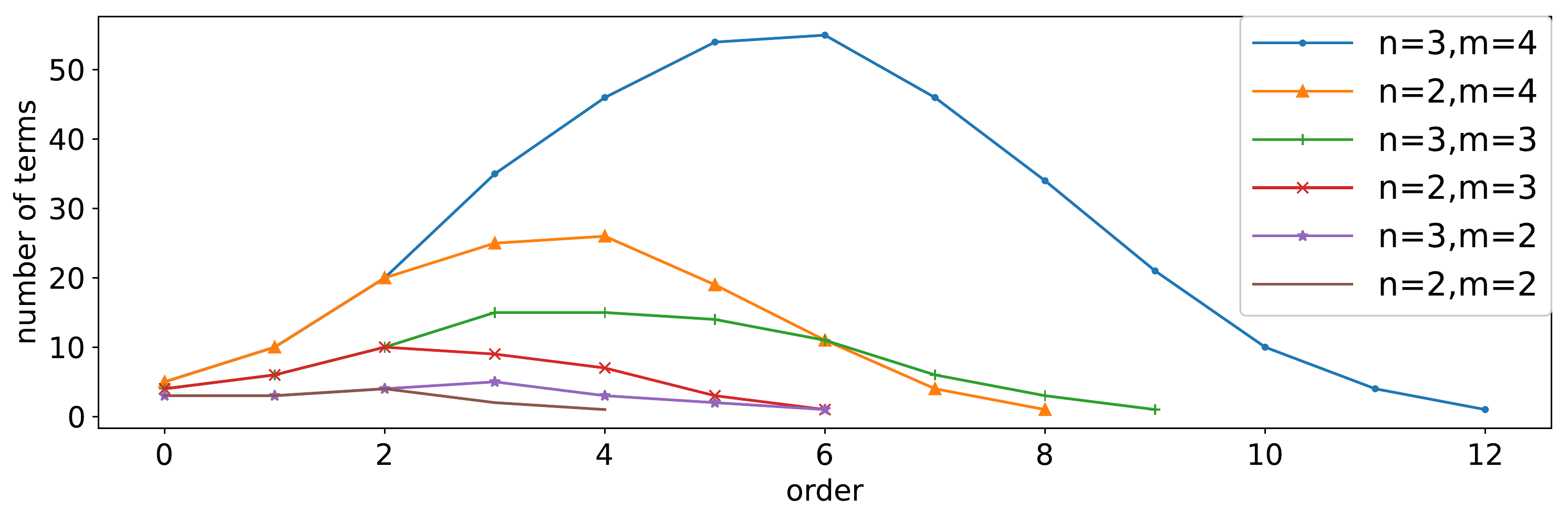}
  \caption{The number of terms that added to proxy kernels. The horizontal axis is the order and the vertical axis is the number of terms.}
  \label{fig:conv-multinomial}
\end{figure}

\subsubsection{Multiple Channels and Layers}
\label{subsubsec:combine-many-layers}

In the following, we will show that multichannel and multi-layer structure can also be represented in the form of Volterra convolution.
Moreover, we also measure the change of kernel size during this conversion. 

\begin{lemma}
  Multichannel convolution can be represented in the form of Volterra convolution.
  \label{lemma:multi-channel-convolution-vconv}
\end{lemma}
\begin{proof}
  Suppose that a multichannel convolution for signal \(\tvar{x}\) and kernel \(\tvar{h}\) is
  \begin{equation*}
    y(c,t)
    = \sum_{u, \tau} h(c,u,\tau) x(u, t - \tau)
    = \sum_{u} h(c, u) * x(u).
  \end{equation*}
  It can be considered as sum of multiple convolutions, which have exactly the same form.
\end{proof}

\begin{lemma}
  Stacking order \(o_1, o_2, \cdots, o_n\) Volterra convolutions can be converted to the form of order-\(\prod_{i=1}^{n} o_i\) Volterra convolution, where \(o_1, o_2, \cdots, o_n > 0\).
  \label{lemma:stacking-low-order-vconv-to-high-order-vconv}
\end{lemma}

\begin{proof}
  To prove this, we recursively apply Lemma \ref{lemma:stacking-on-vconv-om-vconv-approximate-onm-vconv}.
  Taking the first two layers into consideration, we have order-\((o_1 o_2)\) Volterra convolution.
  Appending the third layer, we have order-\((o_1 o_2 o_3)\) Volterra convolution.
  Recursively, after appending the \(n\)-th layer, order becomes \(\prod_{i=1}^{k+1}o_i\), which completes this proof.
\end{proof}

\begin{remark}
  With number of stacked layers increasing, the order of converted Volterra convolution increase exponentially.
  \label{remark:number-of-layers-exp-order-of-volterra}
\end{remark}

Previous lemmas show us the change of orders. The following lemma will measure the change of kernel size. This lemma will be helpful if we want to compute the overall sizes, strides, and paddings of proxy kernels.

\begin{lemma}
  Suppose that size of kernels are \(z_1, z_2, \cdots, z_n\), strides are \(s_1, s_2, \cdots, s_n\), and paddings are \(p_1, p_2, \cdots, p_n\), size of combined kernel is \(z_1 + (z_2 - 1) s_1 + (z_3 - 1) s_1 s_2 + \cdots + (z_n - 1) \prod_{k=1}^{n-1} s_k\), combined stride is \(\prod_{k=1}^{n} s_k\) and combined padding is \(p_1 + p_2 s_1 + \cdots + p_n \prod_{k=1}^{n-1} s_k\).
  \label{lemma:size-of-n-layer-combined-kernel}
\end{lemma}
\begin{proof}
  Size of convolution is \footnote{https://pytorch.org/docs/stable/generated/torch.nn.Conv1d.html}
  \begin{equation*}
    \text{out\_size} = \dfrac{\text{in\_size} + 2 \text{ padding} - \text{kernel\_size}}{\text{stride}} + 1.
  \end{equation*}
  Stacking the first two layers, we have
  \begin{equation*}
    \begin{aligned}
      \text{out\_size}
       & = \dfrac{\left(\text{in\_size} + 2 p_1 - z_1\right)/s_1 + 1 + 2p_2 - z_2}{s_2} + 1  \\
       & = \dfrac{\text{in\_size} + 2 (p_1 + p_2 s_1) - (z_1 + (z_2 - 1) s_1)}{s_1 s_2} + 1.
    \end{aligned}
  \end{equation*}
  Equivalent size is \(z_1 + (z_2 - 1) s_1\), stride is \(s_1 s_2\), and padding is \(p_1 + p_2 s_1\).
  Stacking the first three layers, we have
  \begin{equation*}
    \text{out\_size}
    = \dfrac{\text{in\_size} + 2 (p_1 + p_2 s_1 + p_3 s_1 s_2) - (z_1 + (z_2 - 1) s_1 + (z_3 - 1) s_1 s_2)}{s_1 s_2 s_3} + 1.
  \end{equation*}
  Equivalent size is \(z_1 + (z_2 - 1) s_1 + (z_3 - 1) s_1 s_2\), stride is \(s_1 s_2 s_3\), and padding is \(p_1 + p_2 s_1 + p_3 s_1 s_2\).
  Recursively, this proof is completed.
\end{proof}

\subsection{Validation of Lemma \ref{lemma:convolution-with-activation-is-Volterra}
  and Lemma \ref{lemma:stacking-o2-vconv-o2-vconv-approximate-o4-vconv}}
\label{subsec:validate-gluing-convolutions}

In this subsection, we will validate the approximation of two simple structures (truncated version of Lemma \ref{lemma:convolution-with-activation-is-Volterra} and Lemma \ref{lemma:stacking-o2-vconv-o2-vconv-approximate-o4-vconv}), which are two fundamental structures in our proof and which can show the effectiveness and correctness of our results.

The set \(\mathbb{M}\) is defined for numerical validation in this subsection and below,
\begin{equation}
  \mathbb{M} = \left\{
  \tvar{x} : \tvar{x} = \dfrac{\tvar{y}}{\|\tvar{y}\|_{2}}, y(\cdot) \sim \mathcal{N}(0, 1)
  \right\},
  \label{equ:define-the-set-of-gaussian-l2-norm-set-to-1}
\end{equation}
where \(\mathcal{N}(0,1)\) is standard Gaussian distribution with mean zero and variance one and \(\|\tvar{y}\|_{2}\) is \(L_2\) norm of \(\tvar{y}\). The upper script of these symbols indicates the shape, i.e., \(\tvar{H} \in \mathbb{M}^{9 \times 9}\) is a matrix taken from \(\mathbb{M}\) and the shape is \((9, 9)\).

\textbf{Validate Lemma \ref{lemma:convolution-with-activation-is-Volterra}}: We will check whether the ``conv | act | conv'' structure can be approximated in the form of Volterra convolution.
Take the first four terms into consideration and activation function \(\sigma(t) = 1/(1 + e^{-t})\), we have
\begin{equation}
  \label{equ:validate-conv-g-sigoid-conv-h-x}
  \tvar{g} * \sigma(\tvar{h} * \tvar{x})
  \approx \dfrac{1}{2} \sum \tvar{g}
  + \dfrac{1}{4} (\tvar{g} \oconv \tvar{h}) * \tvar{x}
  - \dfrac{(\diag(3, \tvar{g}) \oconv \tvar{h}^3) * \tvar{x}^3}{48}
  + \dfrac{(\diag(5, \tvar{g}) \oconv \tvar{h}^5) * \tvar{x}^5}{480}.
\end{equation}

This approximation relies on the Taylor expansion, which is valid only in a small neighbor of zero. Therefore, we need to make sure that \(\tvar{h} * \tvar{x}\) is located in such neighbor.
We randomly generate \(\tvar{x} \in \mathbb{M}^{64}\), \(\tvar{h} \in \mathbb{M}^{9}\) and \(\tvar{g} \in \mathbb{M}^{5}\). Convolution with sigmoid activation is plotted in Figure \ref{fig:sub-convolution-taylor-target} and the approximated order-five Volterra convolution is plotted in Figure \ref{fig:sub-convolution-taylor-estimate}. 

\begin{figure}[htb]
  \centering
  \subfloat[\(\tvar{g} * \sigma(\tvar{h} * \tvar{x})\)]{
    \centering
    \includegraphics[width=0.45\linewidth]{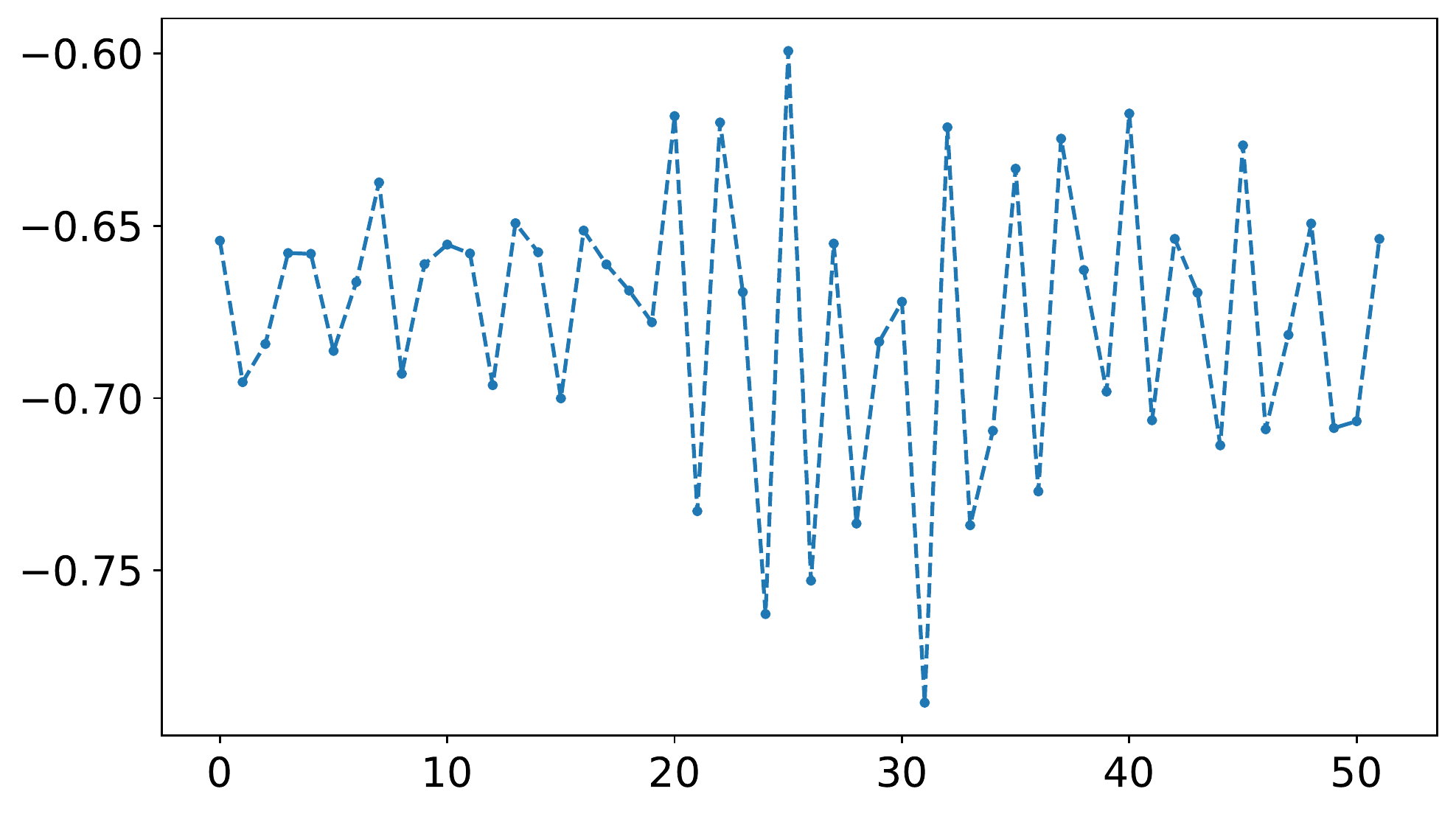}%
    \label{fig:sub-convolution-taylor-target}
  }
  \subfloat[\(\dfrac{1}{2} \sum \tvar{g} + \dfrac{1}{4} (\tvar{g} \oconv \tvar{h}) * \tvar{x} + \cdots\)]{
    \centering
    \includegraphics[width=0.45\linewidth]{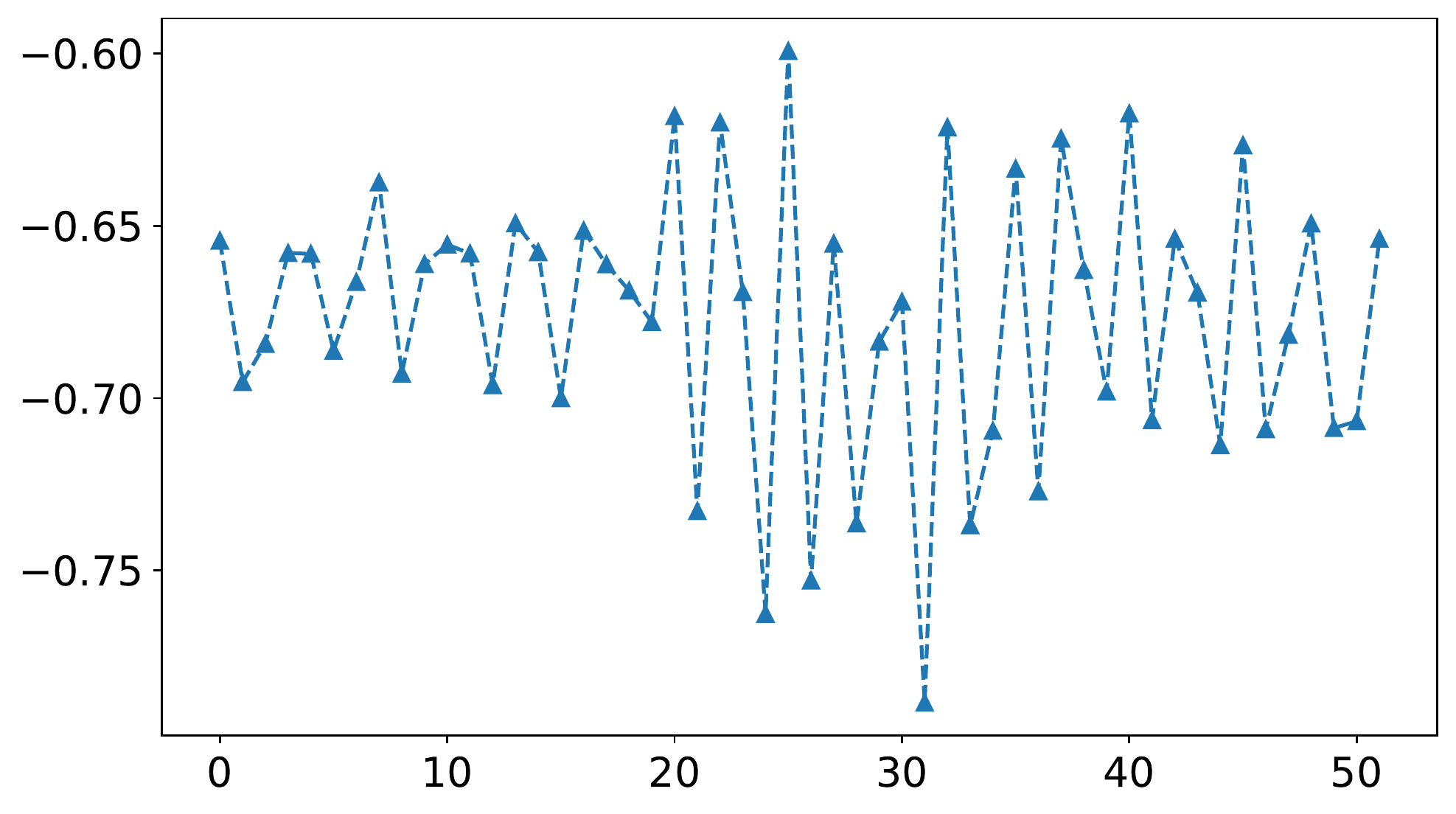}%
    \label{fig:sub-convolution-taylor-estimate}
  }
  \caption{(\ref{fig:sub-convolution-taylor-target}) The output of convolution with activation (left-hand side of Equation \ref{equ:validate-conv-g-sigoid-conv-h-x}); (\ref{fig:sub-convolution-taylor-estimate}) The output of the approximated Volterra convolution (right-hand side of Equation \ref{equ:validate-conv-g-sigoid-conv-h-x}). The Reconstruct error (\(L_2\)-norm) between (\ref{fig:sub-convolution-taylor-target}) and (\ref{fig:sub-convolution-taylor-estimate}) is \syncrecord{conv_taylor_err}{\(5.64877e^{-07}\)}.}
  \label{fig:check-convolution-taylor}
\end{figure}

\textbf{Validate Lemma \ref{lemma:stacking-o2-vconv-o2-vconv-approximate-o4-vconv}}: We will check whether the ``order-2 | order-2'' structure can be converted to the form of order-four Volterra convolution,
\begin{equation}
  \sum_{i=0}^{2} \tvar{G}_i * \left( \sum_{j=0}^{2} \tvar{H}_j * \tvar{x}^j \right)^i
  = \sum_{k=0}^{4} \tvar{F}_k * \tvar{x}^k,
  \label{equ:validate-vconv-g-vocnv-h-x}
\end{equation}
where \(\tvar{F}_k\) are kernels from Equation \ref{equ:combination-two-volterra-kernels}.

We randomly generate \(H_0, G_0 \sim \mathcal{N}(0,1)\), \(\tvar{H}_1, \tvar{G}_1 \in \mathbb{M}^{5}\) and \(\tvar{H}_2, \tvar{G}_2 \in \mathbb{M}^{5 \times 5}\), and \(\tvar{x} \in \mathbb{M}^{64}\).
The left-hand side of Equation \ref{equ:validate-vconv-g-vocnv-h-x} is plotted in Figure \ref{fig:sub-volterra-22-target} and the right-hand side is plotted in Figure \ref{fig:sub-volterra-22-estimate}. It shows that these two are exactly the same.

\begin{figure}[htb]
  \centering
  \subfloat[\(\sum_{i=0}^{2} \tvar{G}_i * \left( \sum_{j=0}^{2} \tvar{H}_j * \tvar{x}^j \right)^i\)]{
    \centering
    \includegraphics[width=0.45\linewidth]{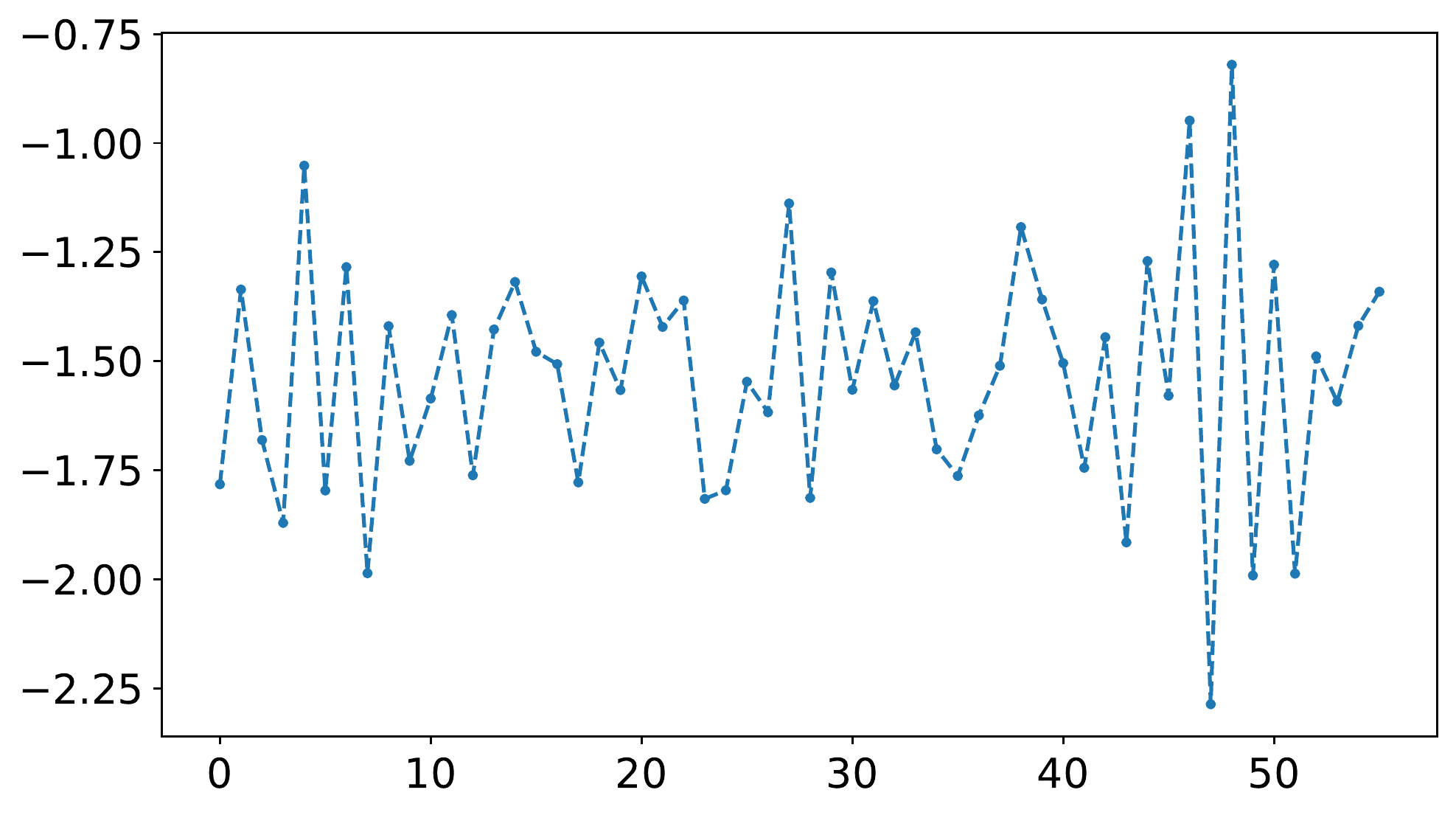}%
    \label{fig:sub-volterra-22-target}
  }
  \subfloat[\(\sum_{k=0}^{4} \tvar{F}_k * \tvar{x}^k\)]{
    \centering
    \includegraphics[width=0.45\linewidth]{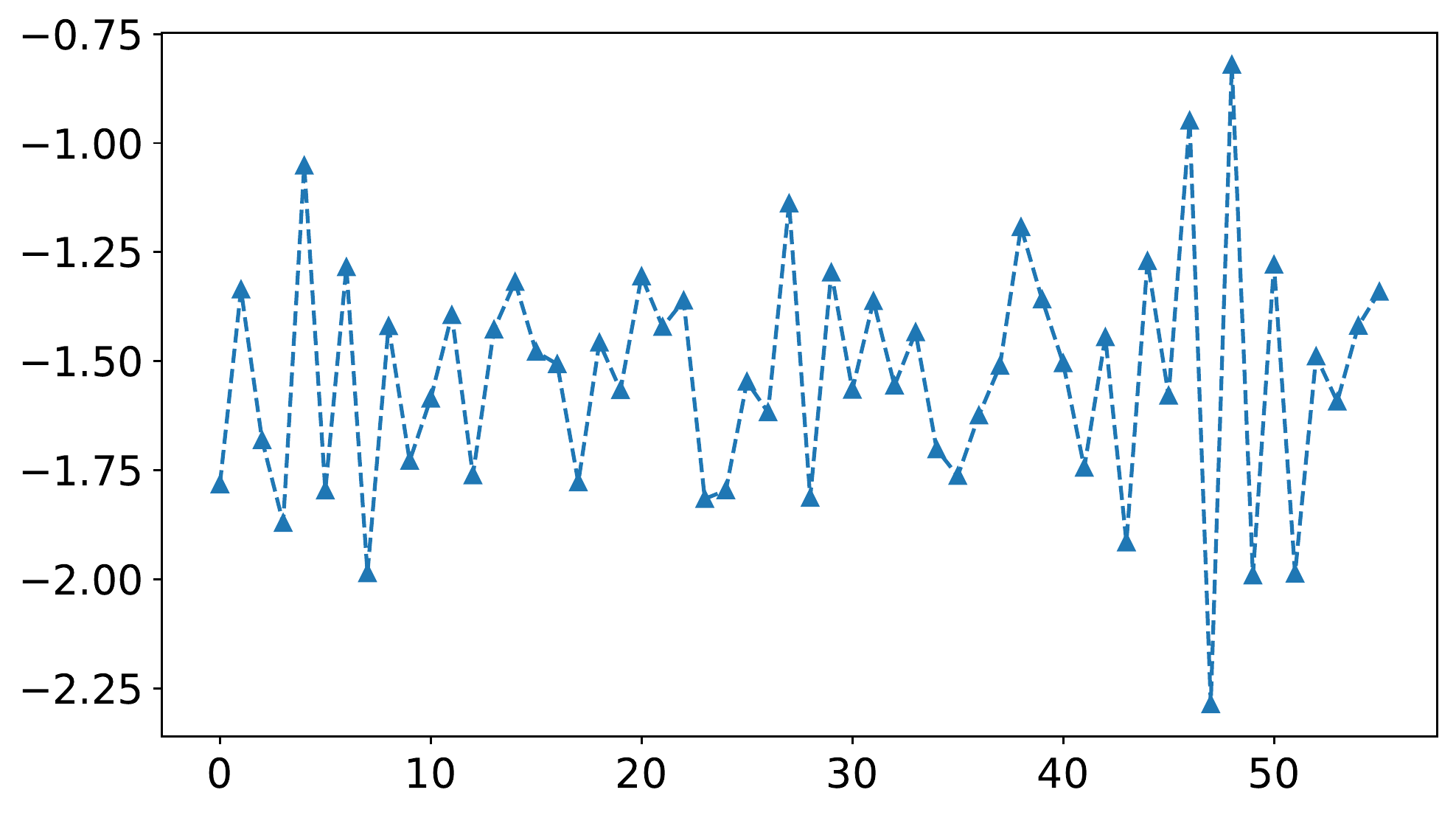}%
    \label{fig:sub-volterra-22-estimate}
  }
  \caption{(\ref{fig:sub-volterra-22-target}) The output of stacking two order-two Volterra convolutions (left-hand side of Equation \ref{equ:validate-vconv-g-vocnv-h-x}); (\ref{fig:sub-volterra-22-estimate}) The output of order-four Volterra convolutions (right-hand side of Equation \ref{equ:validate-vconv-g-vocnv-h-x}). The Reconstruct error (\(L^2\)-norm) between (\ref{fig:sub-volterra-22-target}) and (\ref{fig:sub-volterra-22-estimate}) is \syncrecord{conv_volterra_22_err}{\(1.79705e^{-15}\)}.}
  \label{fig:check-volterra-22}
\end{figure}

\section{Inferring the Proxy Kernels}
\label{sec:inferring-the-proxy-kernels}

All we need to approximate a convolutional neural network to finite term Volterra convolution are the proxy kernels. If the structures and all parameters of a well-trained network is known, i.e., white box, the proxy kernels can be explicitly computed, just as demonstrated in Subsection \ref{subsec:validate-gluing-convolutions}.
However, in some cases, the structure or parameters are not accessible and only the input output pairs are known, i.e., the network is a black box. This may happen when we want to attack a network using black-box mode. In such cases, the proxy kernels can be approximate by training a hacking network, which will be detailed below.

Basically, a hacking network is a network that has the structure of finite term Volterra convolution.
The number of terms of the hacking network controls the order of Volterra convolution and hence the approximation precision.
To train the hacking network, we feed the input-output pairs generated by target network (or its parts) to the hacking network, and minimize the mean square error between the output of the hacking network and that of the target network.

To validate the effectiveness of hacking network, we build a network and compute its order-zero and order-one proxy kernels manually. Then, we train a hacking network to infer these kernels and compare whether they are the same.
To infer the order-zero and order-one terms, we need to build a hacking network with the structure of order-one Volterra convolution, i.e., \(\tvar{w} * \tvar{x} + b\).
This structure can be implemented by a single convolutional layer.
If the hacking network is well-trained, the order-zero proxy kernel is the bias of the convolutional layer, and the order-one proxy kernel is the weight.

\subsection{Two Examples}
\label{subsec:two-examples}

As our first example, we consider a simple two-layer network with sigmoid activation. Suppose that a two-layer network is represented as \(\tvar{g} * \sigma(\tvar{h} * \tvar{x})\).
Without loss of generality, we use the same settings as Equation 28, 
\begin{equation*}
  \tvar{g} * \sigma(\tvar{h} * \tvar{x})
  \approx \dfrac{1}{2} \sum \tvar{g}
  + \dfrac{1}{4} (\tvar{g} \oconv \tvar{h}) * \tvar{x}
  - \dfrac{(\diag(3, \tvar{g}) \oconv \tvar{h}^3) * \tvar{x}^3}{48}
  + \dfrac{(\diag(5, \tvar{g}) \oconv \tvar{h}^5) * \tvar{x}^5}{480}.
\end{equation*}
It is clear that, the order-zero proxy kernel is \(\dfrac{1}{2} \sum \tvar{g}\), and the order-one proxy kernel is \(\dfrac{1}{4} (\tvar{g} \oconv \tvar{h})\).

We randomly pick two kernels
\begin{small}
  \syncrecord{oconv_l2_kernels}{
    \begin{equation*}
      \begin{aligned}
        \tvar{h} & = \begin{bmatrix}0.0381&-0.2047&\phantom{-}0.3097&\phantom{-}0.0693&-0.3405&\phantom{-}0.7618&-0.1190&-0.1089&\phantom{-}0.3657\end{bmatrix}; \\
        \tvar{g} & = \begin{bmatrix}0.5280&-0.3684&-0.2644&-0.3412&-0.2461&-0.1377&-0.3077&\phantom{-}0.4540&-0.1369\end{bmatrix}. \\
      \end{aligned}
    \end{equation*}
  }
\end{small}

The hacking network is a single layer convolutional network.
While training, we randomly generate the input data \(\tvar{x} \in \mathbb{M}^{\text{batch\_size} \times 1 \times \text{length}}\) and minimize mean square error of the output of the hacking network and the output of \(\tvar{g} * \sigma(\tvar{h} * \tvar{x})\).

By straightforward calculation, the order-zero proxy kernel is \syncrecord{oconv_l2_computed_w0}{\(-4.10236e^{-01}\)}, and that obtained from the hacking network is \syncrecord{oconv_l2_estimated_w0}{\(-4.10295e^{-01}\)}, and the inferred order-one proxy kernels respectively obtained by direct calculation and by training the hacking network are compared in Figure \ref{fig:oconv-conv-approximate-two-layer-net-w1-g}.
In summary, the hacking network can approximate the true order-zero and order-one proxy kernels very well.

\begin{figure}[htbp]
  \centering
  \subfloat[\(\dfrac{1}{4} (\tvar{g} \oconv \tvar{h})\)]{
    \includegraphics[width=0.45\linewidth]{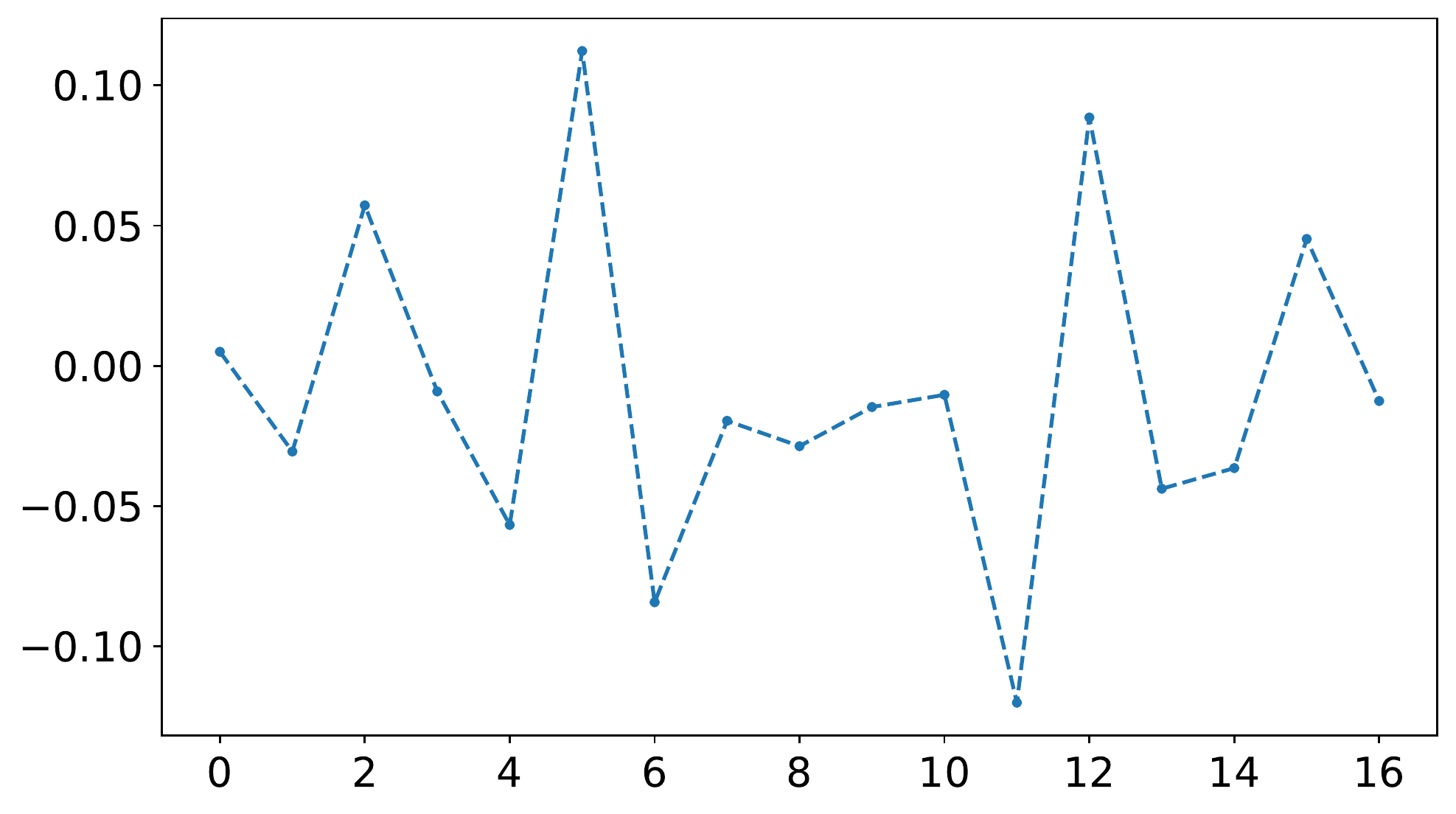}
    \label{fig:oconv-conv-approximate-two-layer-net-w1-computed-g}
  }
  \subfloat[kernel from the trained hacking network]{
    \includegraphics[width=0.45\linewidth]{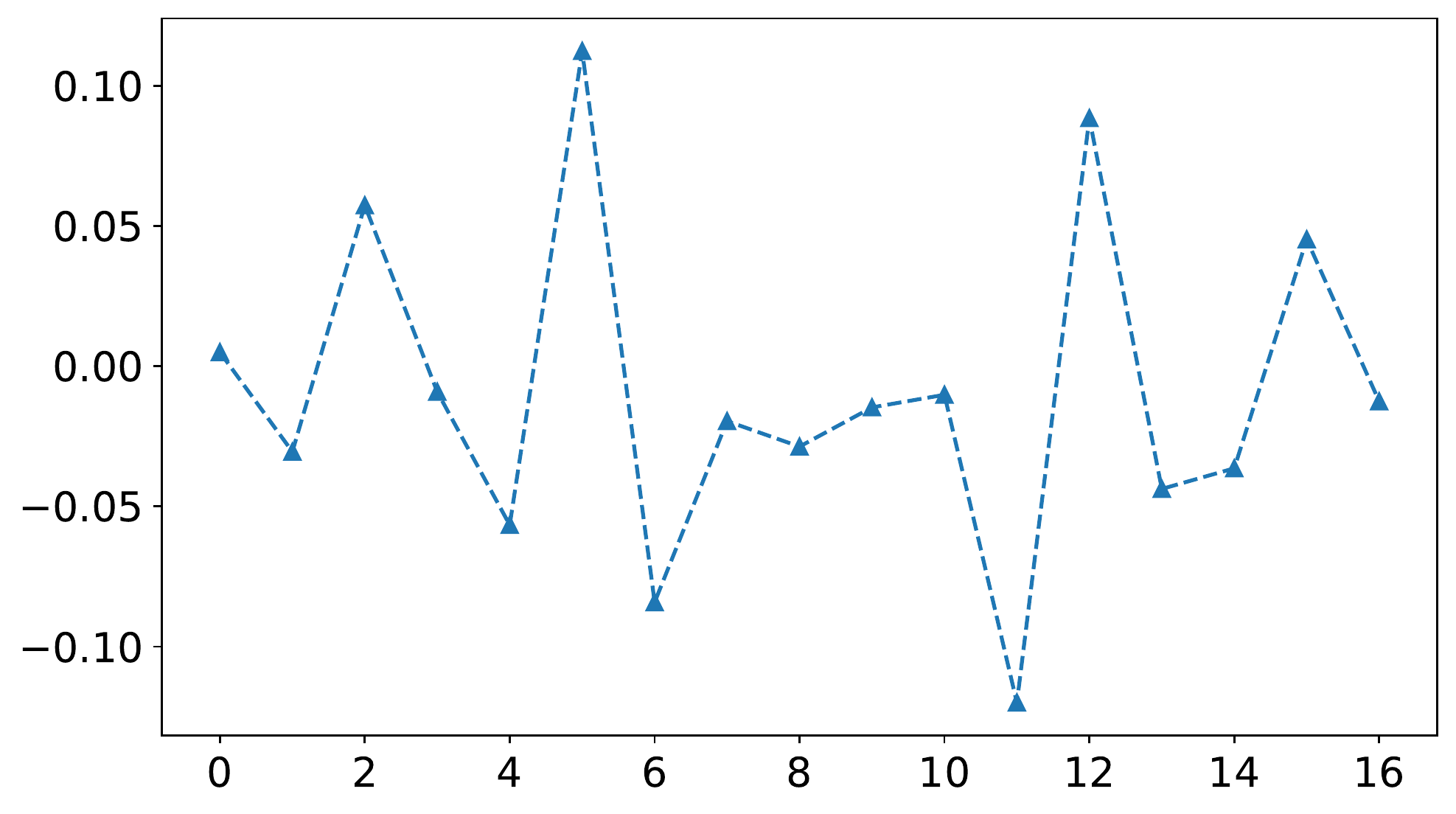}
    \label{fig:oconv-conv-approximate-two-layer-net-w1-estimated-g}
  }
  \caption{The inferred order-one proxy kernels from the two-layer network. The reconstruction error (\(L_2\)-norm) between left and right is \syncrecord{oconv_l2_err_w1}{\(3.04194e^{-04}\)}.}
  \label{fig:oconv-conv-approximate-two-layer-net-w1-g}
\end{figure}

For the second example, we consider a three-layer network with sigmoid activation. Suppose that a three-layer network is represented as \(\tvar{f} * \sigma(\tvar{g} * \sigma(\tvar{h} * \tvar{x}))\).
Using the same settings as previous example and approximating this three-layer network into the form of Volterra convolution, we have
\begin{equation*}
  \begin{aligned}
     & \tvar{f} * \sigma(\tvar{g} * \sigma(\tvar{h} * \tvar{x}))                      \\
     & = \dfrac{1}{2} \sum \tvar{f}
    + \dfrac{1}{4} \tvar{f} * [\tvar{g} * \sigma(\tvar{h} * \tvar{x})]
    - \dfrac{1}{48} \tvar{f} * [\tvar{g} * \sigma(\tvar{h} * \tvar{x})]^{3}
    + \dfrac{1}{480} \tvar{f} * [\tvar{g} * \sigma(\tvar{h} * \tvar{x})]^{5}          \\
     & = \dfrac{1}{2} \sum \tvar{f}
    + \dfrac{1}{4} \tvar{f} * (\tvar{g} * \sigma(\tvar{h} * \tvar{x}))
    - \dfrac{1}{48} \diag(3, \tvar{f}) * (\tvar{g} * \sigma(\tvar{h} * \tvar{x}))^{3} \\
     & ~~~~~
    + \dfrac{1}{480} \diag(5, \tvar{f}) * (\tvar{g} * \sigma(\tvar{h} * \tvar{x}))^{5}.
    ~~~~~~~~ (\text{Property 15})
    \\
  \end{aligned}
\end{equation*}
Substituting Equation 28 
and reordering the terms, we obtain the order-zero proxy kernel as
\begin{equation*}
  \left(\sum \tvar{f} \right)
  \left(
  \dfrac{1}{2}
  + \dfrac{1}{4} \left(\dfrac{1}{2} \sum \tvar{g}\right)
  - \dfrac{1}{48} \left(\dfrac{1}{2} \sum \tvar{g}\right)^3
  + \dfrac{1}{480} \left(\dfrac{1}{2} \sum \tvar{g}\right)^5
  \right),
\end{equation*}
and the order-one proxy kernel as
\begin{equation}
  \label{equ:three-layer-order-one-proxy-kernel}
  \dfrac{1}{4}
  \left(
  \dfrac{1}{4}
  - \dfrac{3}{48} \left(\dfrac{1}{2} \sum \tvar{g}\right)^2
  + \dfrac{5}{480} \left(\dfrac{1}{2} \sum \tvar{g}\right)^4
  \right)
  \left( \tvar{f} \oconv \tvar{g} \oconv \tvar{h} \right).
\end{equation}

We randomly pick three kernels
\begin{small}
  \syncrecord{oconv_l3_kernels}{
    \begin{equation*}
      \begin{aligned}
        \tvar{h} & = \begin{bmatrix}\phantom{-}0.4830&-0.3142&-0.2219&\phantom{-}0.0361&0.2659&\phantom{-}0.4040&\phantom{-}0.3978&\phantom{-}0.3628&0.3061\end{bmatrix}; \\
        \tvar{g} & = \begin{bmatrix}-0.1271&\phantom{-}0.1521&\phantom{-}0.6264&-0.2576&0.3027&-0.1574&\phantom{-}0.1009&\phantom{-}0.3923&0.4705\end{bmatrix}; \\
        \tvar{f} & = \begin{bmatrix}\phantom{-}0.7160&-0.3866&-0.1870&-0.0566&0.3057&-0.0062&-0.4463&-0.0395&0.0735\end{bmatrix}. \\
      \end{aligned}
    \end{equation*}
  }
\end{small}

The training process of the hacking network is the same as the previous example.
The order-zero proxy kernel by direct calculation is \syncrecord{oconv_l3_computed_w0}{\(-1.83091e^{-02}\)}, and that by training the hacking network is \syncrecord{oconv_l3_estimated_w0}{\(-1.82950e^{-02}\)}, and the order-one proxy kernels obtained by these two ways shown in Figure \ref{fig:oconv-conv-approximate-three-layer-net-w1-g2}.
The figure shows again that the proxy kernels obtained training the hacking network is almost the same as that obtained by direct calculation.

\begin{figure}[htbp]
  \centering
  \subfloat[kernel form Equation \ref{equ:three-layer-order-one-proxy-kernel}]{
    \includegraphics[width=0.45\linewidth]{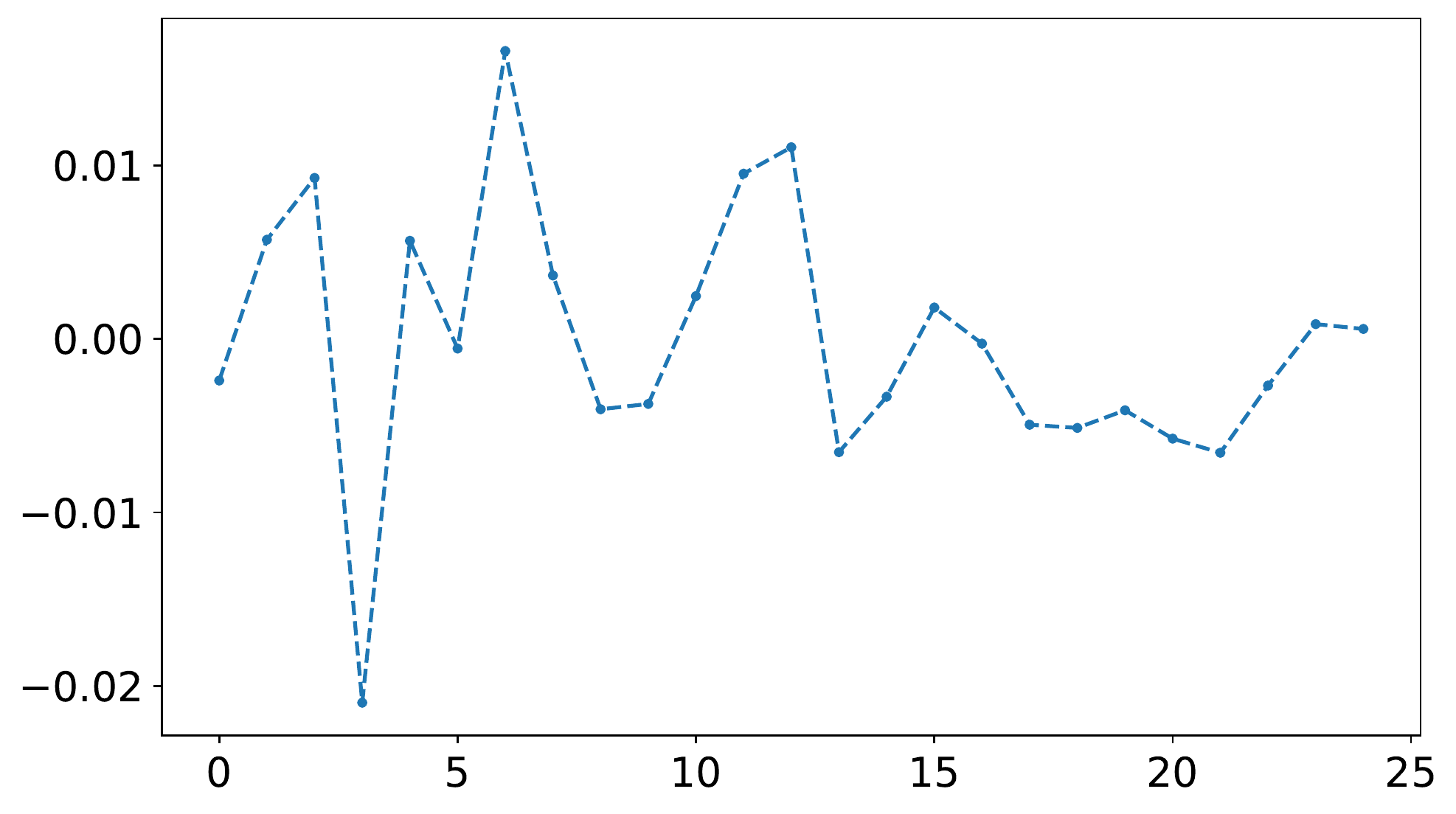}
  }
  \subfloat[kernel from the trained hacking network]{
    \includegraphics[width=0.45\linewidth]{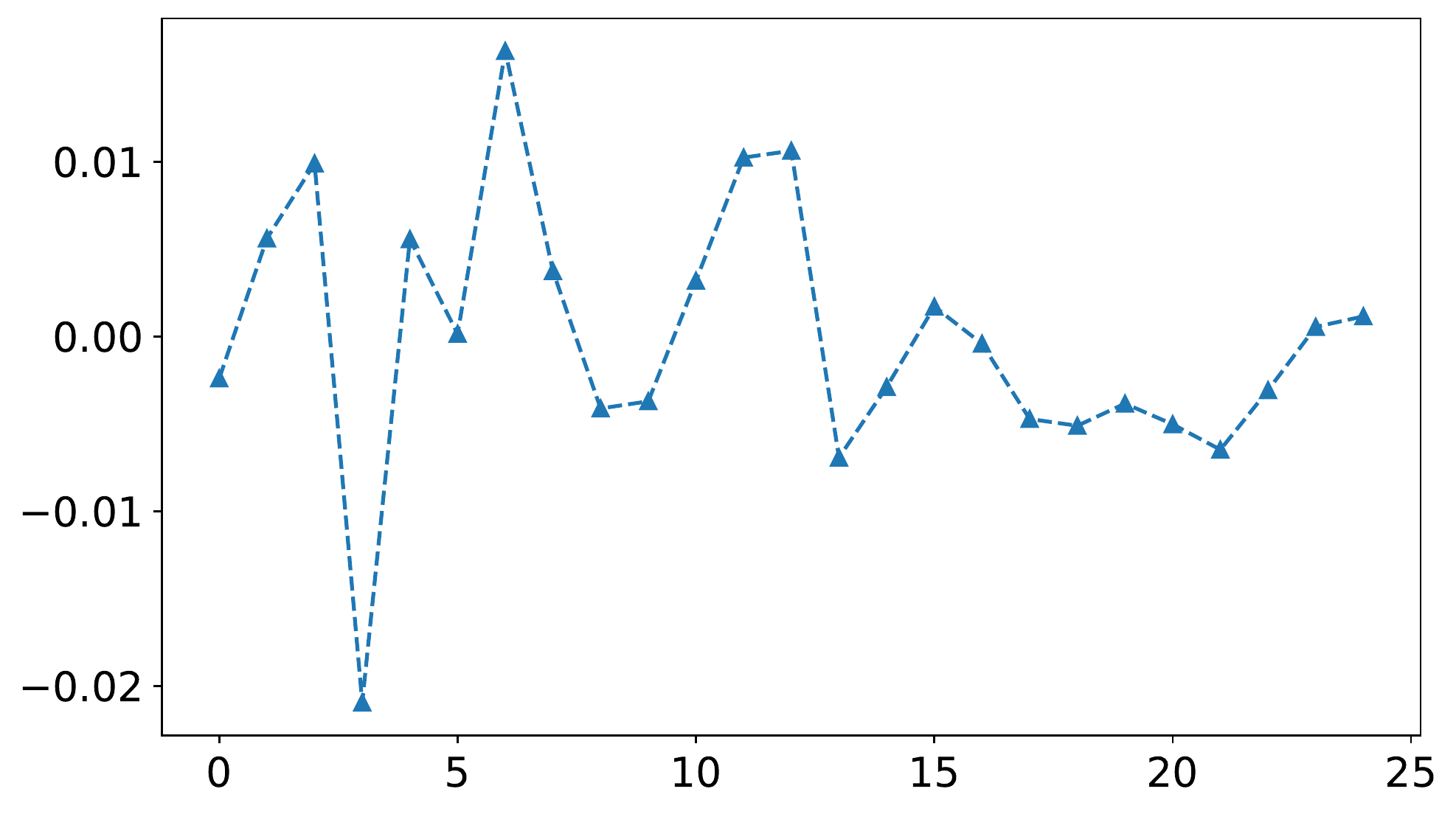}
  }
  \caption{The inferred order-one proxy kernels of the three-layer network. The reconstruction error (\(L_2\)-norm) between left and right is \syncrecord{oconv_l3_err_w1}{\(1.94667e^{-03}\)}.}
  \label{fig:oconv-conv-approximate-three-layer-net-w1-g2}
\end{figure}

These two examples show that the order-zero and order-one proxy kernels inferred by direct calculation and by training the hacking network are very close to each other.
To cover a broader selection of parameters, we analyze the statistics of the reconstruction errors of the proxy kernels inferred from these two methods. The results are illustrated in Figure \ref{fig:oconv-conv-approximate-err-histogram}.
It shows that the choice of parameters has a less effect on the reconstruction error and these two methods are comparable.

\begin{figure}[htbp]
  \centering
  \includegraphics[width=0.6\linewidth]{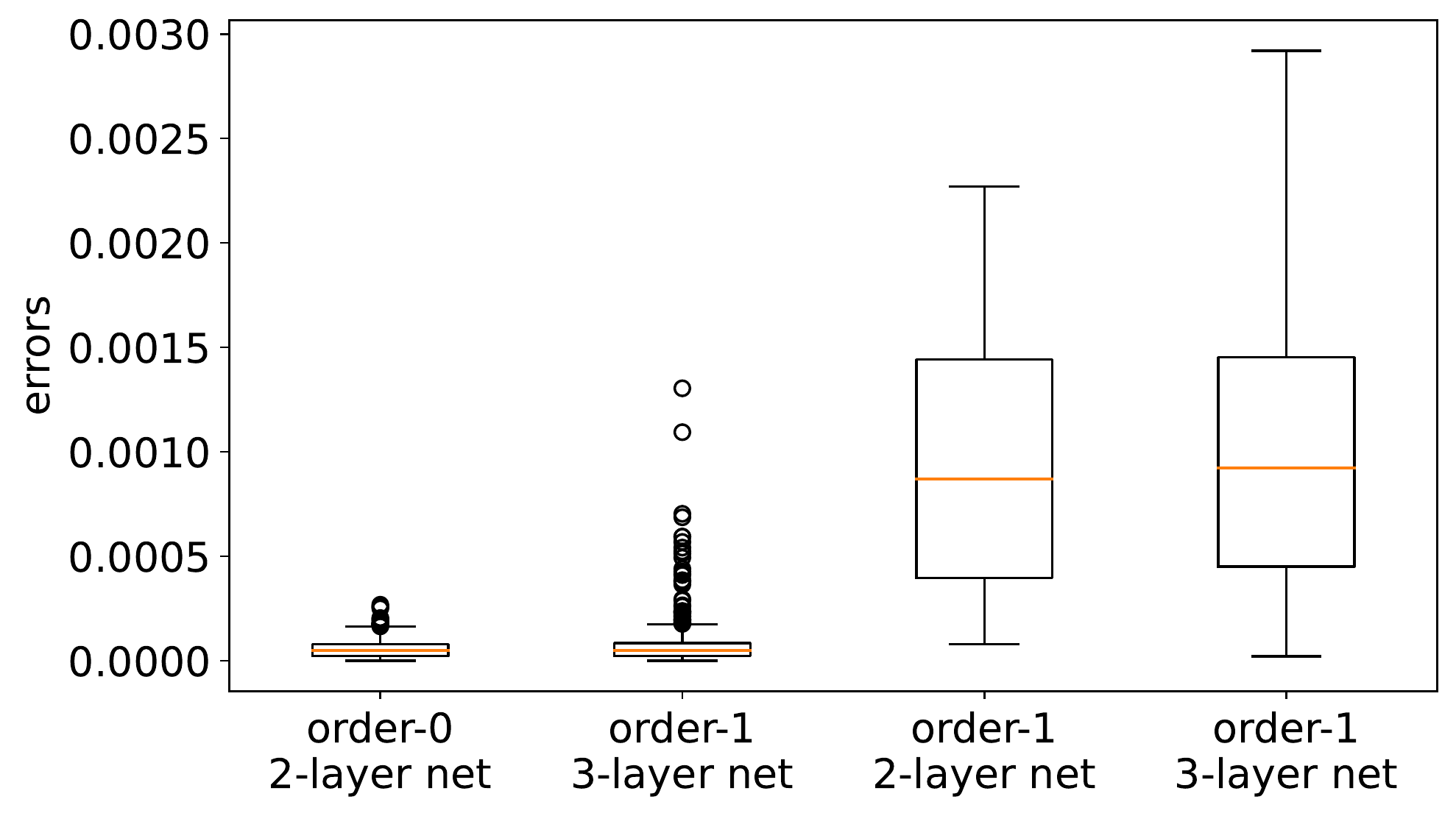}
  \caption{The boxplot of reconstruction errors between two methods.}
  \label{fig:oconv-conv-approximate-err-histogram}
\end{figure}

\subsection{An Application of the Order-one Proxy Kernels}
\label{subsec:hacking-network}

Why do we need to infer the proxy kernels?

We think the proxy kernels shall contain some useful information about the original network.
It is possible to carefully design special inputs according to the inferred proxy kernels to change the behavior of the original network.
Below we provide a toy example to illustrate this interesting application of order-one proxy kernels.
First, we build a hacking network to approximate the order-one proxy kernel of a classifier network trained on the MNIST dataset \citep{LeCun1998}. Then, we add visually imperceptible perturbation to the inputs to cheat the classifier to give wrong labels.

The structure of the classifier network is shown in the left of Table \ref{tab:structure-of-classifier-network}. Training images are all scaled to \([0,1]\). To obtain a network that is robust to noise, uniform noise \(\mathcal{U}(0,0.2)\) and Gaussian noise \(\mathcal{N}(0, 0.2)\) are respectively added to training images with a probability of 5\%.
After 512 training episodes, the network achieved \syncrecord{mnist_classifier_accuracy}{98.280\%} accuracy on test set.

The hacking network is illustrated in the right of Table \ref{tab:structure-of-classifier-network}.
We only approximate the first six layers of the classifier network.
This is because approximating the entire classifier network to a linear network will result in high fitting errors, and these errors will make it harder for hacking network to converge. From Table \ref{tab:structure-of-classifier-network}, the hacking network has ten output channels. Hence, we have ten proxy kernels \(\tvar{h}_{1}, \cdots, \tvar{h}_{10}\), corresponding to the ten output channels.

\begin{table}[htb]
  \caption{Structure of classifier network and hacking network. Parameters of conv2d layers are input channels, output channels, kernel size, stride, and padding.}
  \label{tab:structure-of-classifier-network}
  \centering
  \begin{tabular}{c||c|c}
    \hline
              & classifier network            & hacking network                                \\ \hline
              & input                         & input                                          \\ \hline
    \small{1} & conv2d(1, 10, k=3, s=2, p=1)  & \multirow{6}{*}{conv2d(1, 10, k=15, s=8, p=3)} \\
    \cline{1-2}
    \small{2} & sigmoid                       &                                                \\
    \cline{1-2}
    \small{3} & conv2d(10, 10, k=3, s=2, p=1) &                                                \\
    \cline{1-2}
    \small{4} & sigmoid                       &                                                \\
    \cline{1-2}
    \small{5} & conv2d(10, 10, k=3, s=2, p=0) &                                                \\
    \cline{1-2}
    \small{6} & sigmoid                       &                                                \\
    \cline{1-3}
    \small{7} & conv2d(10, 10, k=3, s=1, p=0) & \multirow{4}{*}{output}                        \\
    \cline{1-2}
    \small{8} & sigmoid                       &                                                \\
    \cline{1-2}
    \small{9} & flatten, linear(10, 10)       &                                                \\
    \cline{1-2}
              & output                        &                                                \\ \hline
  \end{tabular}
\end{table}

The training of the hacking network is similar to the procedure discussed in the previous subsection.
The training images and their outputs of the classifier network form input-label pairs, and the hacking network is trained with these pairs until convergence.
Compared to the outputs of the first six layers of the classifier network, the mean square error of the hacking network is about \syncrecord{mnist_hack_err}{\(3.19016e^{-02}\)}. In other words, these two networks have the similar behavior.

Suppose that some neurons in the network are suppressed when feeding input image \(\tvar{x}\).
To change the behavior of this network, we need to activate some extra neurons.
Specifically, in this application, we try to increase the energy of these extra neurons to activate them.

One important advantage of the hacking network is its simple structure: only convolutions are involved. Hence, to largely change the outputs, for given order-one proxy kernel \(\tvar{h}\) and input signal \(\tvar{x}\), we search for perturbation \(\epsilon\) such that
\begin{equation}
  \argmax_{\epsilon} \| \tvar{h} * (\tvar{x} + \epsilon) \|_2, ~~ \text{s.t. } \|\epsilon \|_2 \le c,
\end{equation}
where \(c\) is a constant.
If \(\tvar{h}\) and \(\tvar{x} + \epsilon\) are close to each other in the frequency domain, the energy \(\| \tvar{h} * (\tvar{x} + \epsilon) \|_2\) should be large.

Recall Parseval's Theorem and Convolution Theorem, convolution in time domain equals multiplication in frequency domain and the energy is preserved, \(\epsilon\) can be obtained via Fourier transform \(\mathcal{F}(\cdot)\),
\begin{equation}
  \mathcal{F}(\epsilon) \propto \mathcal{F}(\tvar{h}) - \mathcal{F}(\tvar{x}).
\end{equation}

Applying the bounded condition, we scale it by \(\alpha\) and take inverse Fourier transform \(\mathcal{F}^{-1}(\cdot)\),
\begin{equation}
  \epsilon = \alpha \mathcal{F}^{-1}\left(\mathcal{F}(\tvar{h}) - \mathcal{F}(\tvar{x})\right).
  \label{equ:resonance-condition}
\end{equation}

Recall Equation \ref{equ:resonance-condition}, we can compute a perturbation as
\begin{equation*}
  \epsilon_{i} = \mathcal{F}^{-1} \left(\mathcal{F}(\tvar{w}_{i}) - \mathcal{F}(\tvar{x})\right).
\end{equation*}
Since each pixel of the input image is in the range of \([0, 1]\), the perturbation should be adjusted to the same range. In this application, to make the perturbation hard to be seen, \(\epsilon_i\) is scaled by a factor \(\alpha\).
The value of \(\alpha\) is \(0.55\) in the experiment below.
In short, the perturbation \(\epsilon_i\) is normalized as
\begin{equation*}
  \epsilon_i \leftarrow \alpha \dfrac{\epsilon_i - \min(\epsilon_i)}{\max(\epsilon_i) - \min(\epsilon_i)}.
\end{equation*}

Besides the ten proxy kernels estimated by the hacking network, we also append two fake kernels to check whether the fake kernels can behave similarly. The fake kernels are \(\tvar{h}_{11} \sim \mathcal{U}(0, 1)\) and \(\tvar{h}_{12} \sim \mathcal{N}(0, 1)\).
We randomly choose twelve images from the test set, \(\tvar{x}_i, i = 1, 2, \cdots, 12\), and feed both image \(\tvar{x}_i\) and \(\tvar{x}_i + \epsilon_i\) into classifier network, where \(\epsilon_1, \epsilon_2, \cdots, \epsilon_{10}\) are computed from the approximated order-one proxy kernels, and \(\epsilon_{11}, \epsilon_{12}\) are computed from two fake kernels.
Results are illustrated in Figure \ref{fig:mnist-digits-perturbed-by-specialized-perturbations}.
The images with patches computed from the approximated order-one proxy kernels are more likely to change the output than the images with patches computed from fake kernels.

\begin{figure}[htb!]
  \centering
  
\begin{tikzpicture}
    \node[inner sep=0pt] (I) at (0,0) {%
    \includegraphics[width=0.08\linewidth]{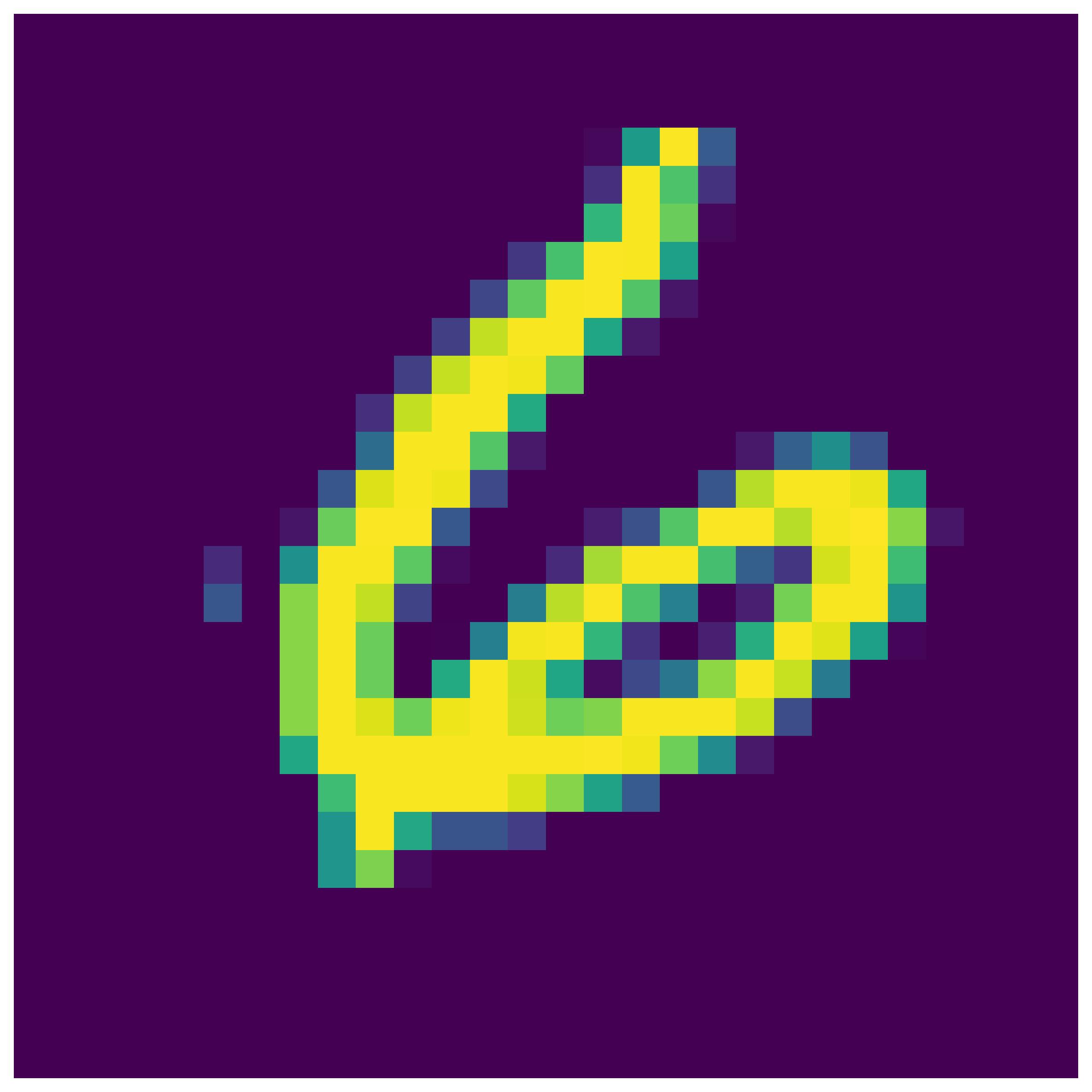}%
    \includegraphics[width=0.08\linewidth]{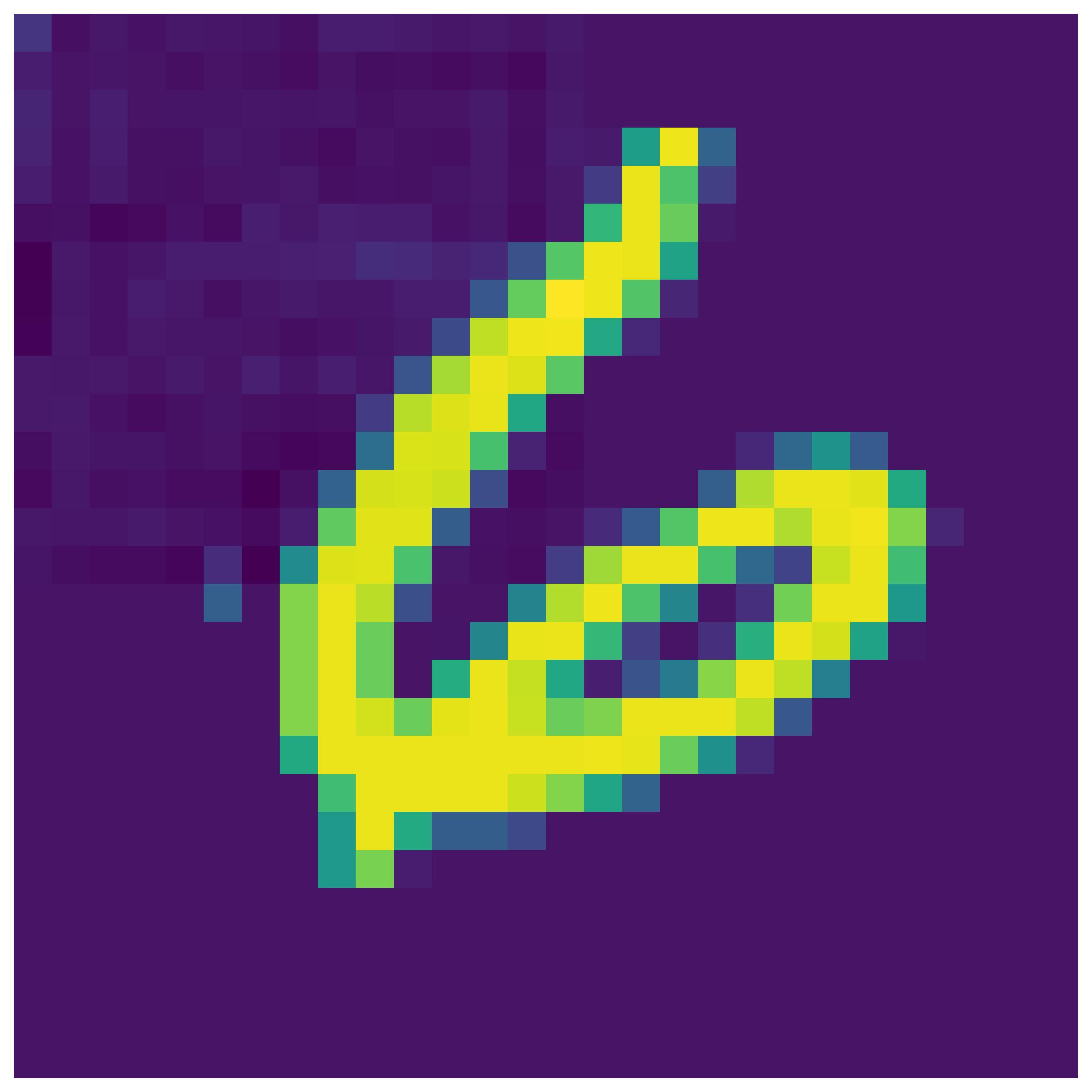}%
    };%
    \node[inner xsep=0pt, inner ysep=4pt, yshift=2pt, anchor=north, scale=0.70] at (I.south) {%
    \begin{tabular}{c|ccc}
    \hline
    \multirow{2}{*}{C} & L & 6 & 0 \\
    \cline{2-4}
    & V & 8.7471 & 2.4500 \\
    \hline
    \multirow{2}{*}{P} & L & 8 & 6 \\
    \cline{2-4}
    & V & 3.0685 & 2.2283 \\
    \hline
    \end{tabular}
    };%
\end{tikzpicture}
\begin{tikzpicture}
    \node[inner sep=0pt] (I) at (0,0) {%
    \includegraphics[width=0.08\linewidth]{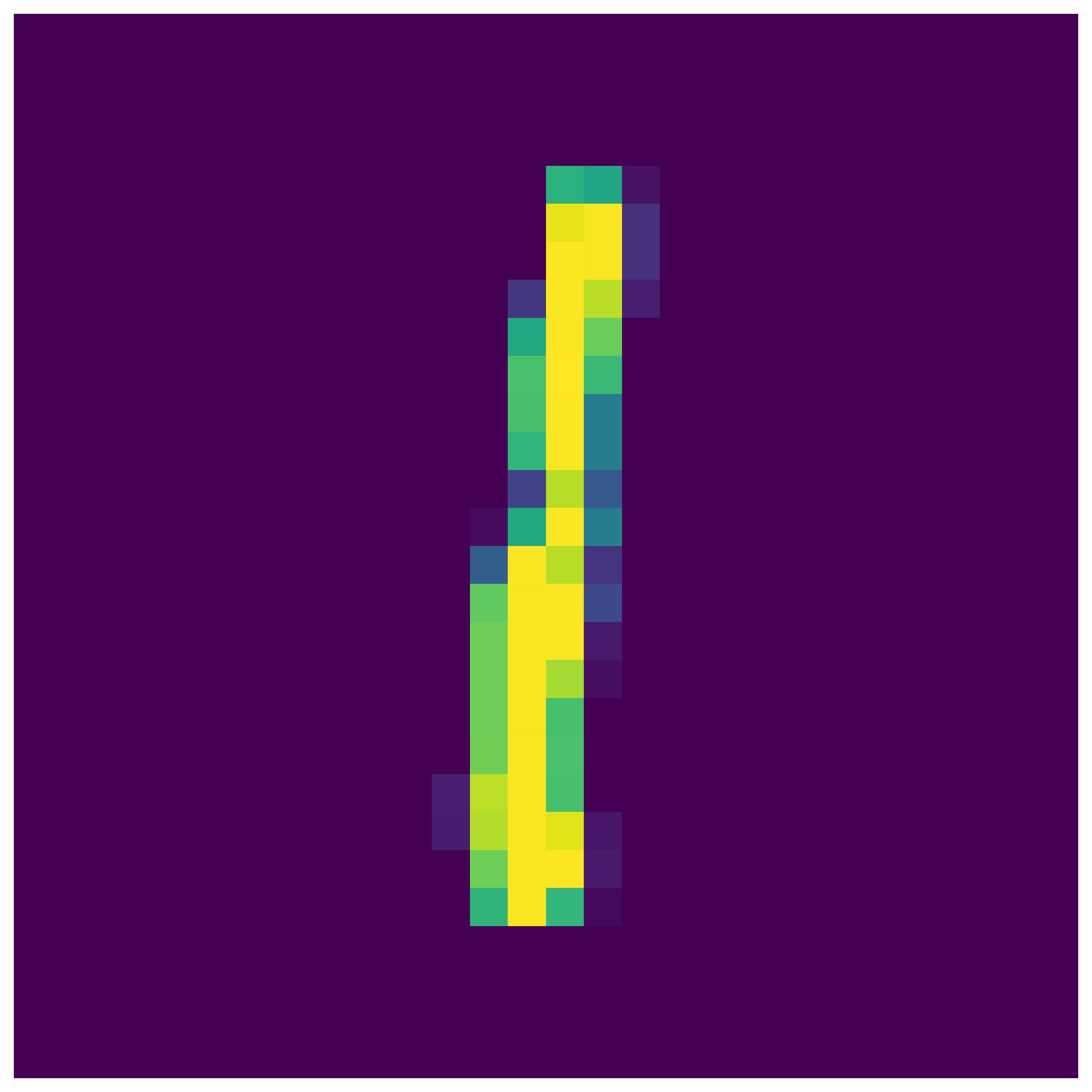}%
    \includegraphics[width=0.08\linewidth]{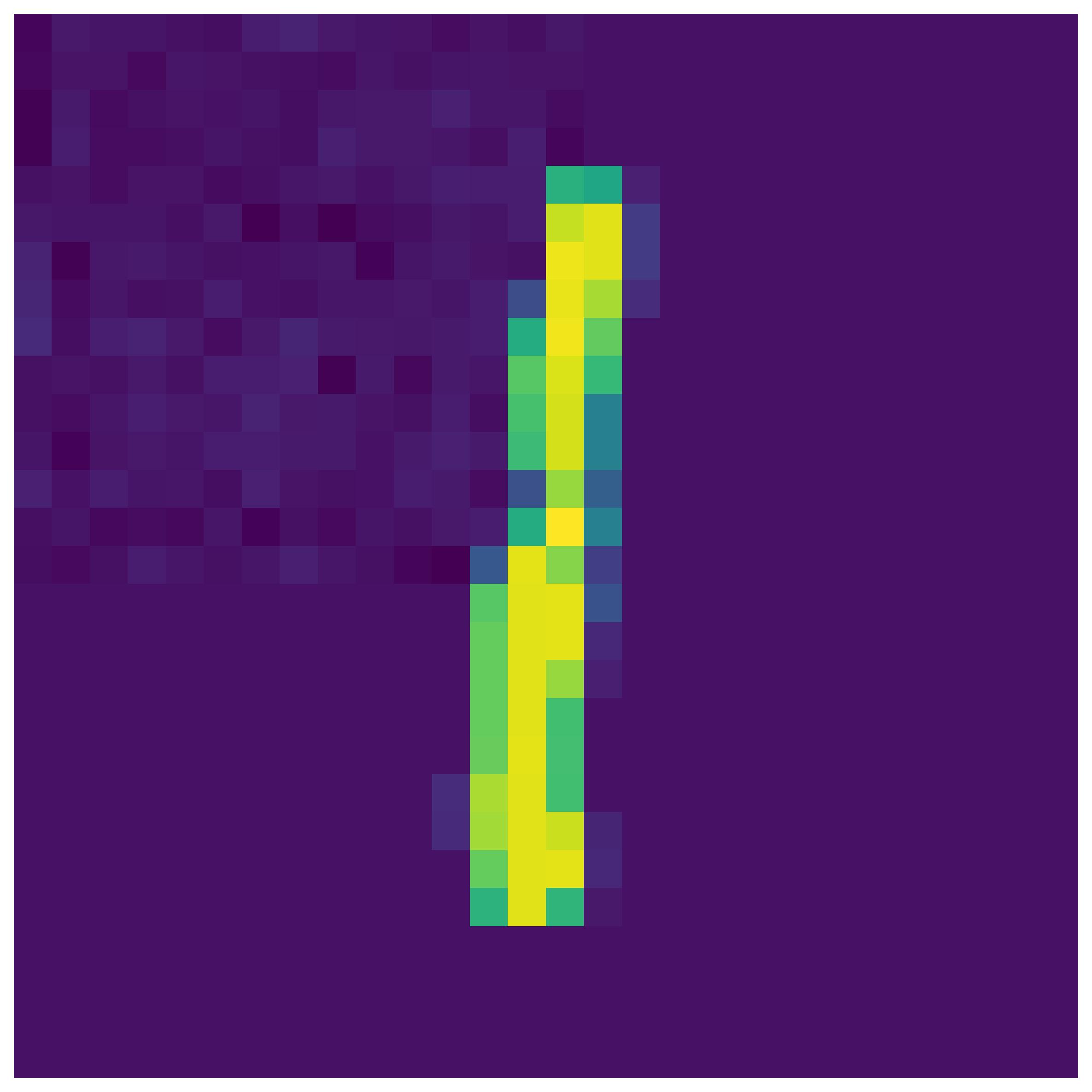}%
    };%
    \node[inner xsep=0pt, inner ysep=4pt, yshift=2pt, anchor=north, scale=0.70] at (I.south) {%
    \begin{tabular}{c|ccc}
    \hline
    \multirow{2}{*}{C} & L & 1 & 5 \\
    \cline{2-4}
    & V & 7.7386 & 0.0578 \\
    \hline
    \multirow{2}{*}{P} & L & 8 & 5 \\
    \cline{2-4}
    & V & 4.8160 & 1.3297 \\
    \hline
    \end{tabular}
    };%
\end{tikzpicture}
\begin{tikzpicture}
    \node[inner sep=0pt] (I) at (0,0) {%
    \includegraphics[width=0.08\linewidth]{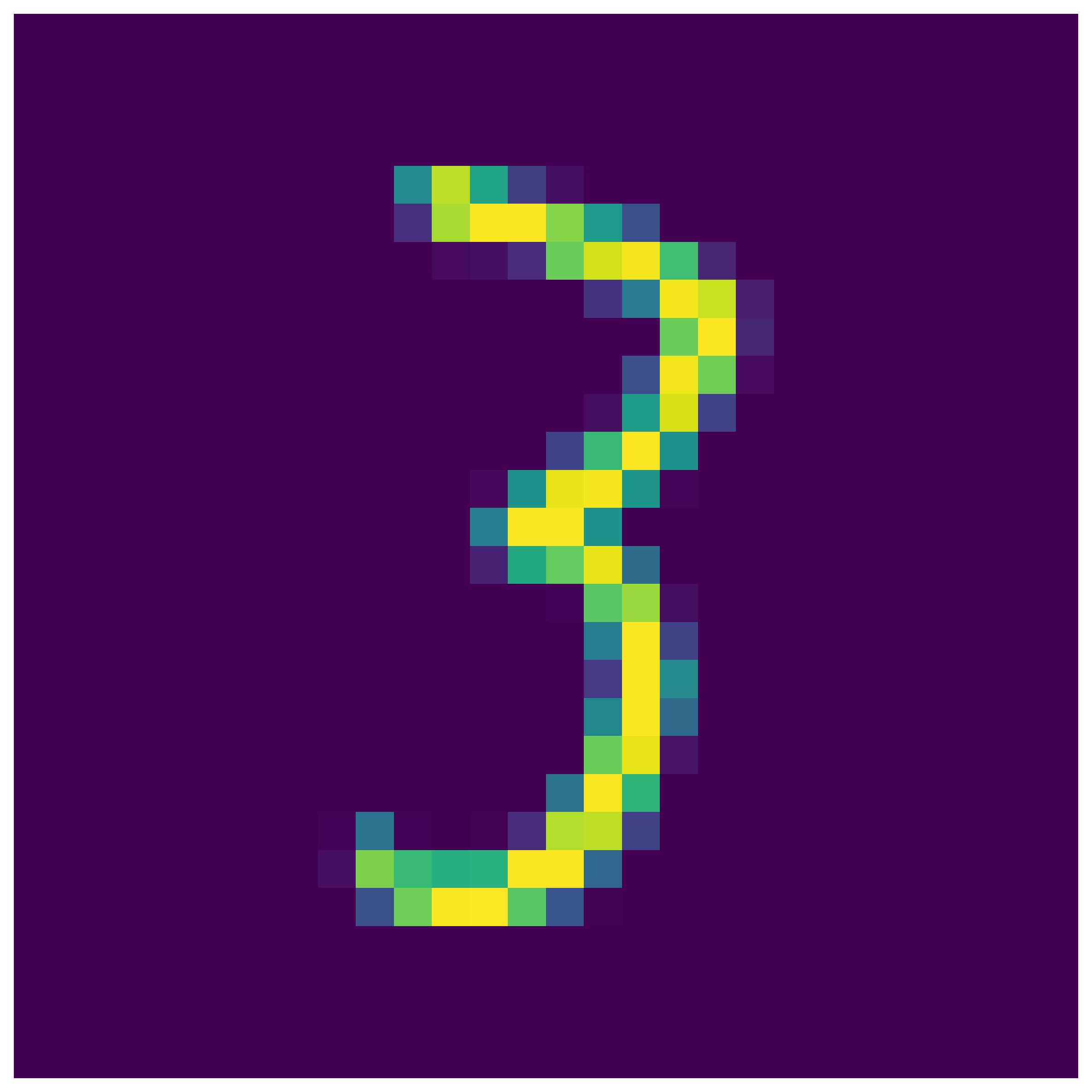}%
    \includegraphics[width=0.08\linewidth]{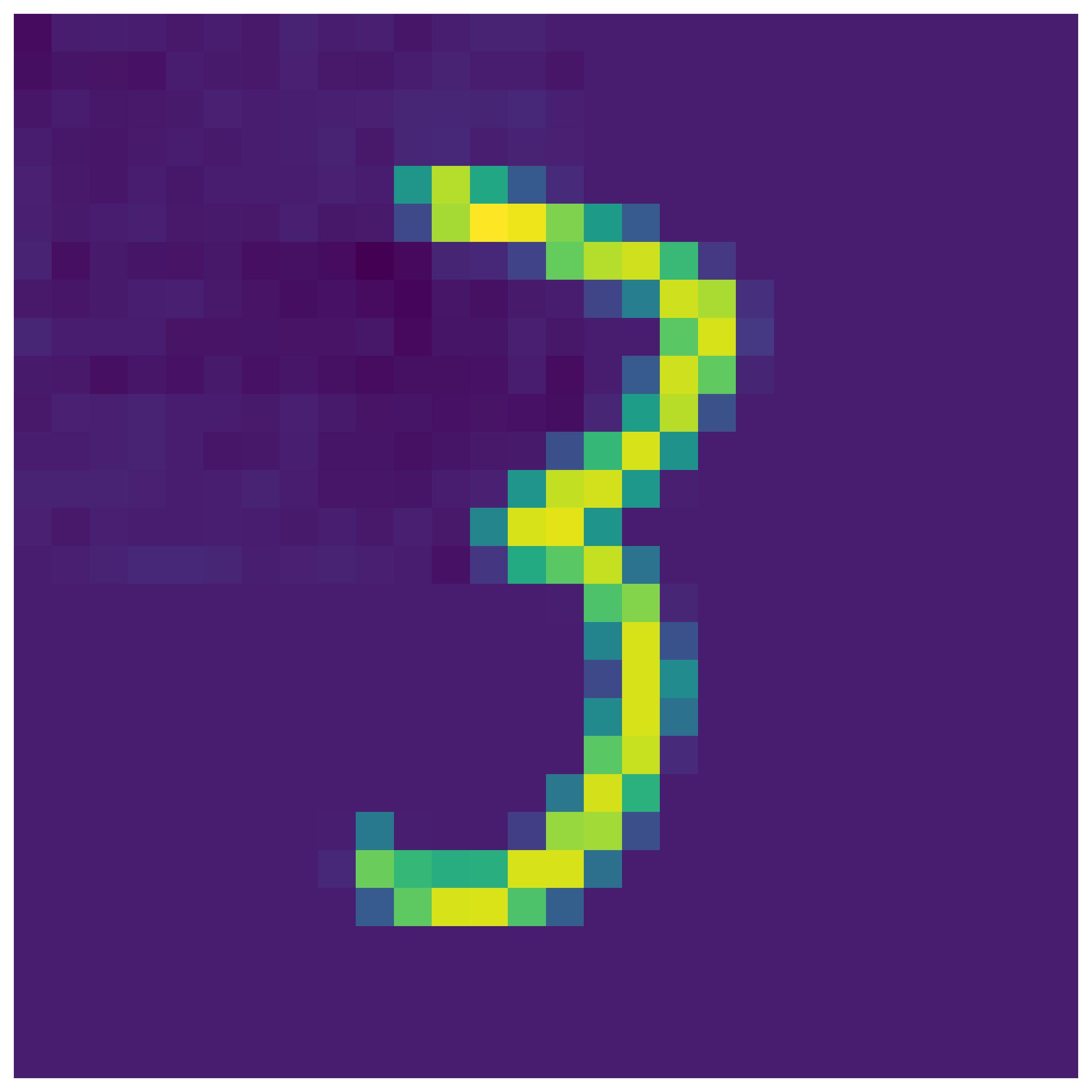}%
    };%
    \node[inner xsep=0pt, inner ysep=4pt, yshift=2pt, anchor=north, scale=0.70] at (I.south) {%
    \begin{tabular}{c|ccc}
    \hline
    \multirow{2}{*}{C} & L & 3 & 2 \\
    \cline{2-4}
    & V & 9.0018 & 1.1149 \\
    \hline
    \multirow{2}{*}{P} & L & 8 & 3 \\
    \cline{2-4}
    & V & 4.8440 & 1.5732 \\
    \hline
    \end{tabular}
    };%
\end{tikzpicture}
\begin{tikzpicture}
    \node[inner sep=0pt] (I) at (0,0) {%
    \includegraphics[width=0.08\linewidth]{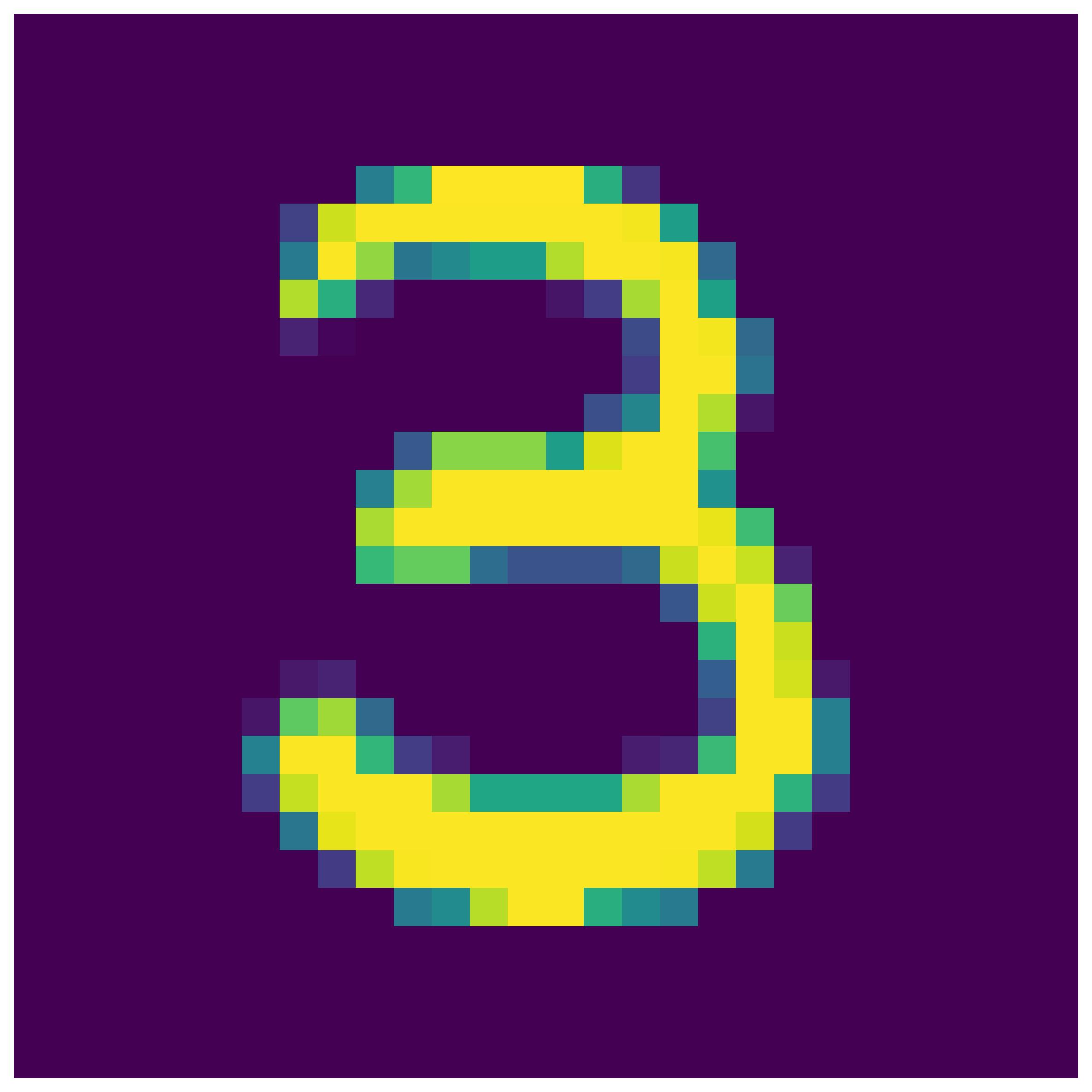}%
    \includegraphics[width=0.08\linewidth]{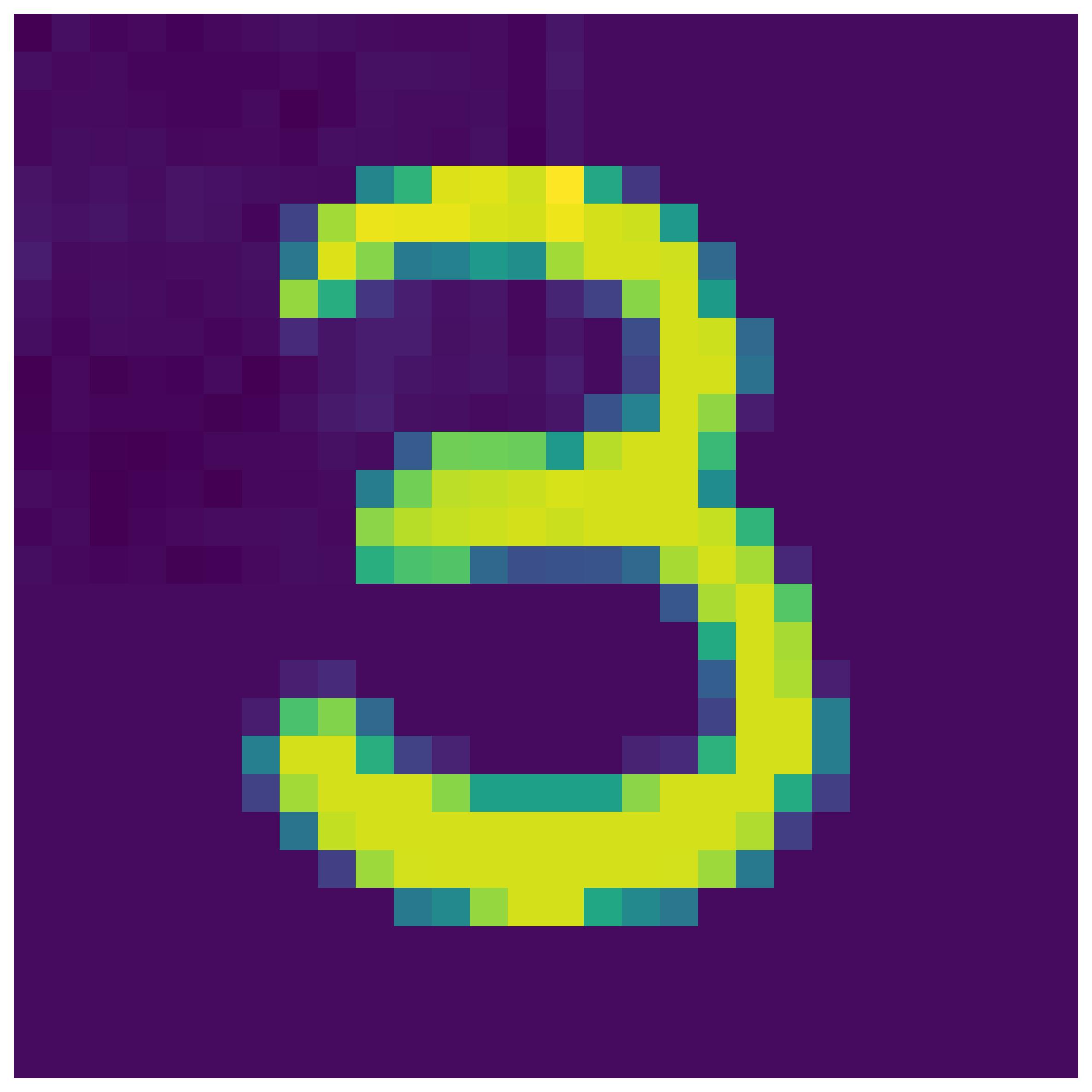}%
    };%
    \node[inner xsep=0pt, inner ysep=4pt, yshift=2pt, anchor=north, scale=0.70] at (I.south) {%
    \begin{tabular}{c|ccc}
    \hline
    \multirow{2}{*}{C} & L & 3 & 5 \\
    \cline{2-4}
    & V & 9.7701 & 1.5166 \\
    \hline
    \multirow{2}{*}{P} & L & 8 & 3 \\
    \cline{2-4}
    & V & 2.7900 & 2.7868 \\
    \hline
    \end{tabular}
    };%
\end{tikzpicture}
\begin{tikzpicture}
    \node[inner sep=0pt] (I) at (0,0) {%
    \includegraphics[width=0.08\linewidth]{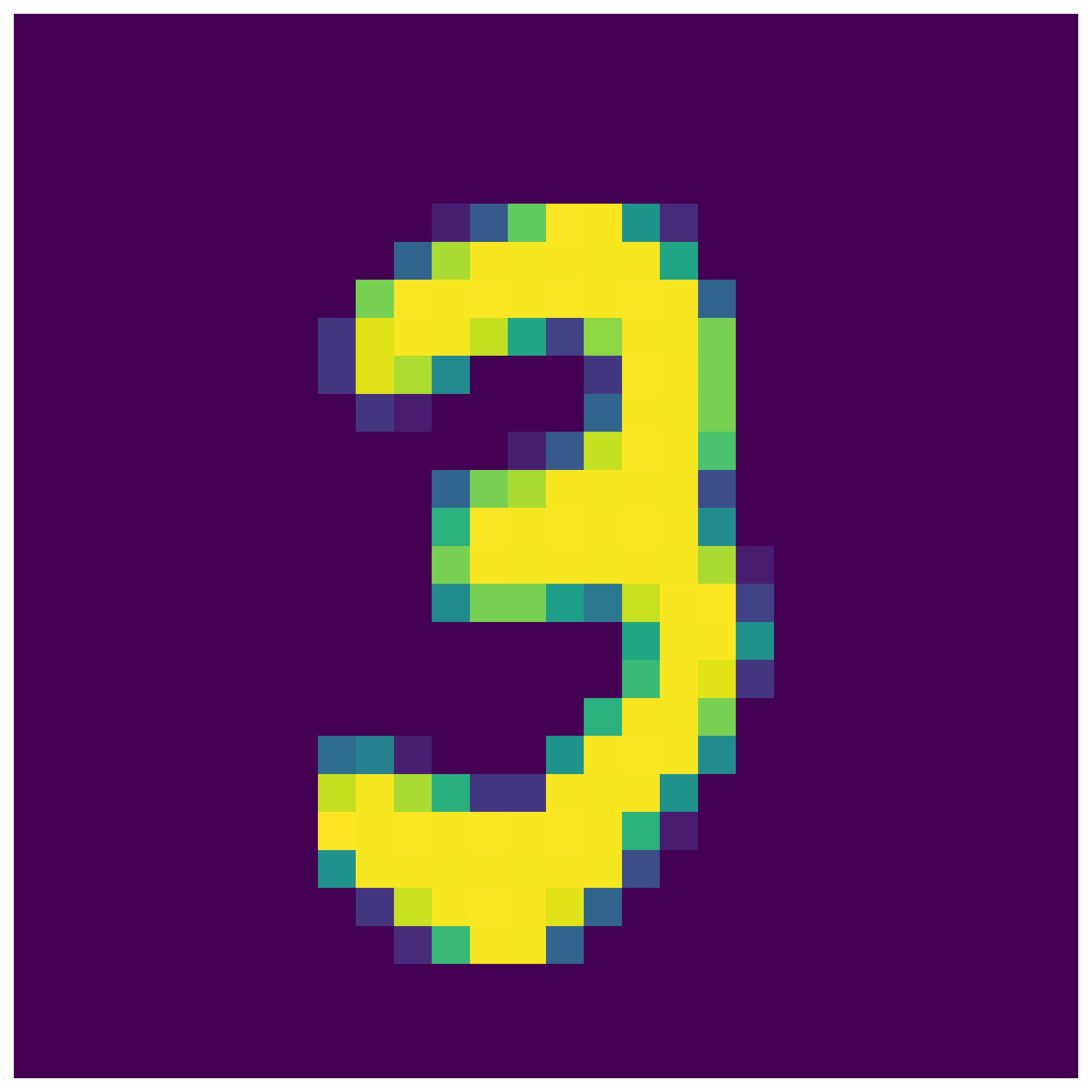}%
    \includegraphics[width=0.08\linewidth]{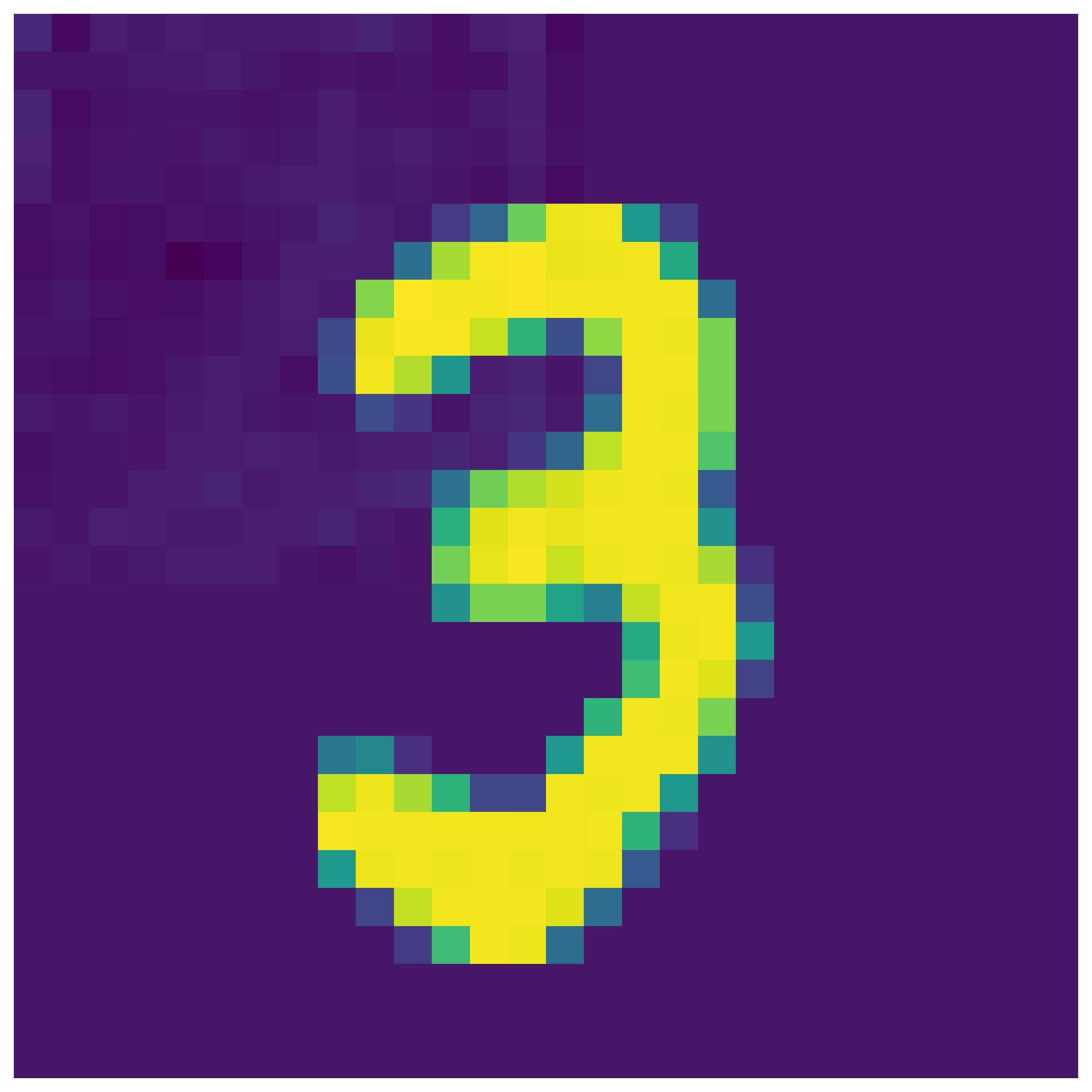}%
    };%
    \node[inner xsep=0pt, inner ysep=4pt, yshift=2pt, anchor=north, scale=0.70] at (I.south) {%
    \begin{tabular}{c|ccc}
    \hline
    \multirow{2}{*}{C} & L & 3 & 9 \\
    \cline{2-4}
    & V & 8.5626 & 2.1267 \\
    \hline
    \multirow{2}{*}{P} & L & 8 & 5 \\
    \cline{2-4}
    & V & 2.8069 & 1.7993 \\
    \hline
    \end{tabular}
    };%
\end{tikzpicture}
\begin{tikzpicture}
    \node[inner sep=0pt] (I) at (0,0) {%
    \includegraphics[width=0.08\linewidth]{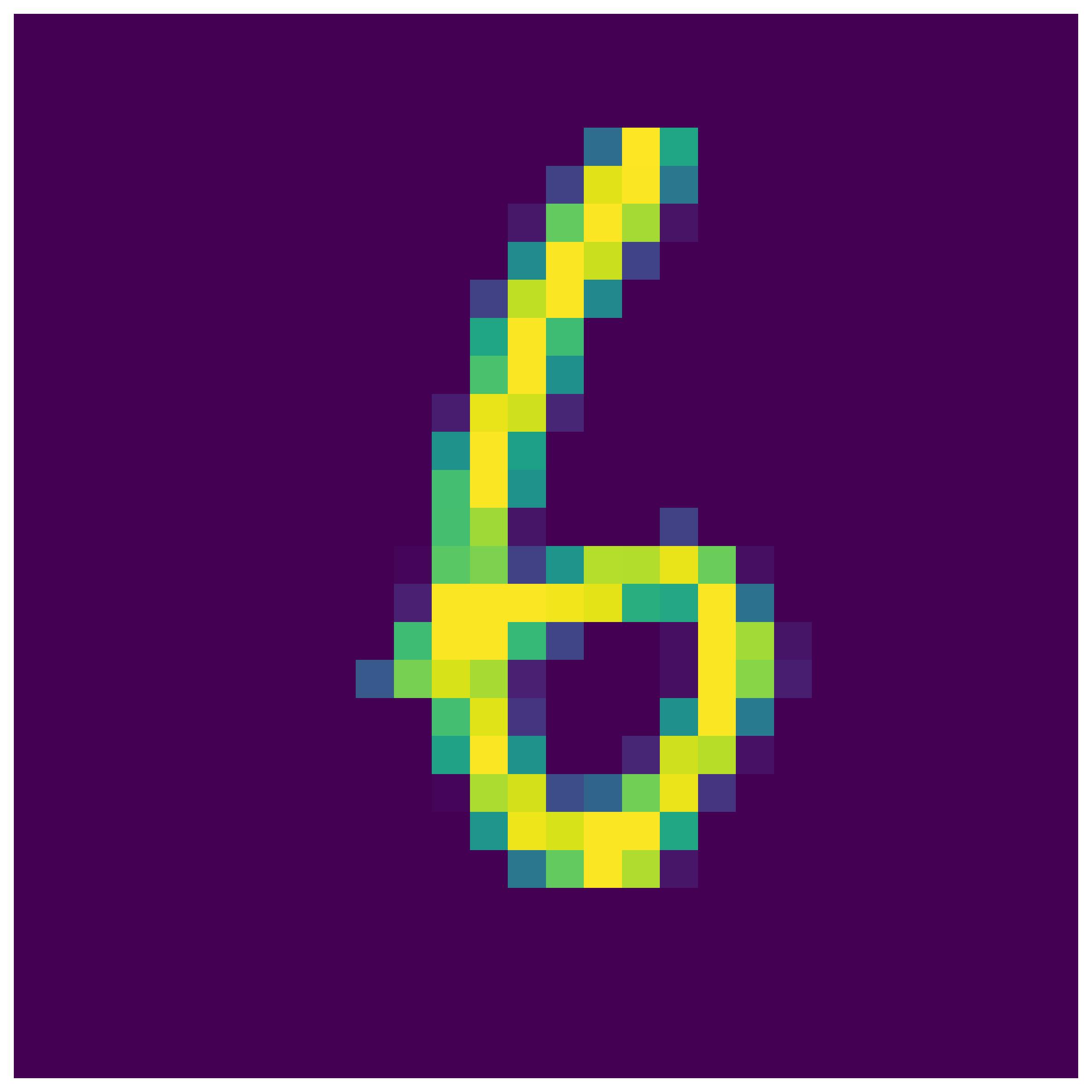}%
    \includegraphics[width=0.08\linewidth]{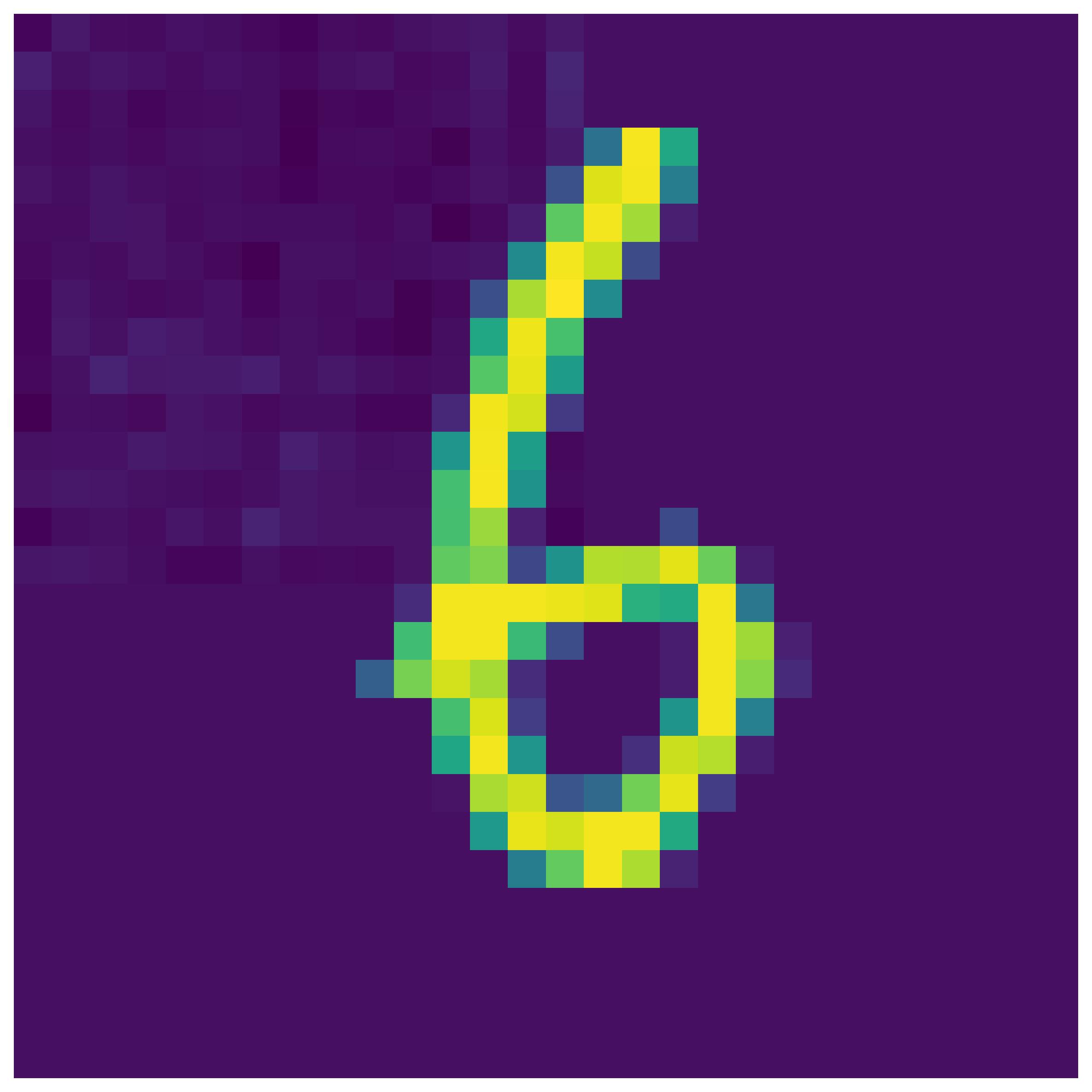}%
    };%
    \node[inner xsep=0pt, inner ysep=4pt, yshift=2pt, anchor=north, scale=0.70] at (I.south) {%
    \begin{tabular}{c|ccc}
    \hline
    \multirow{2}{*}{C} & L & 6 & 4 \\
    \cline{2-4}
    & V & 7.6841 & 2.4219 \\
    \hline
    \multirow{2}{*}{P} & L & 8 & 5 \\
    \cline{2-4}
    & V & 4.1646 & 2.1322 \\
    \hline
    \end{tabular}
    };%
\end{tikzpicture}
\begin{tikzpicture}
    \node[inner sep=0pt] (I) at (0,0) {%
    \includegraphics[width=0.08\linewidth]{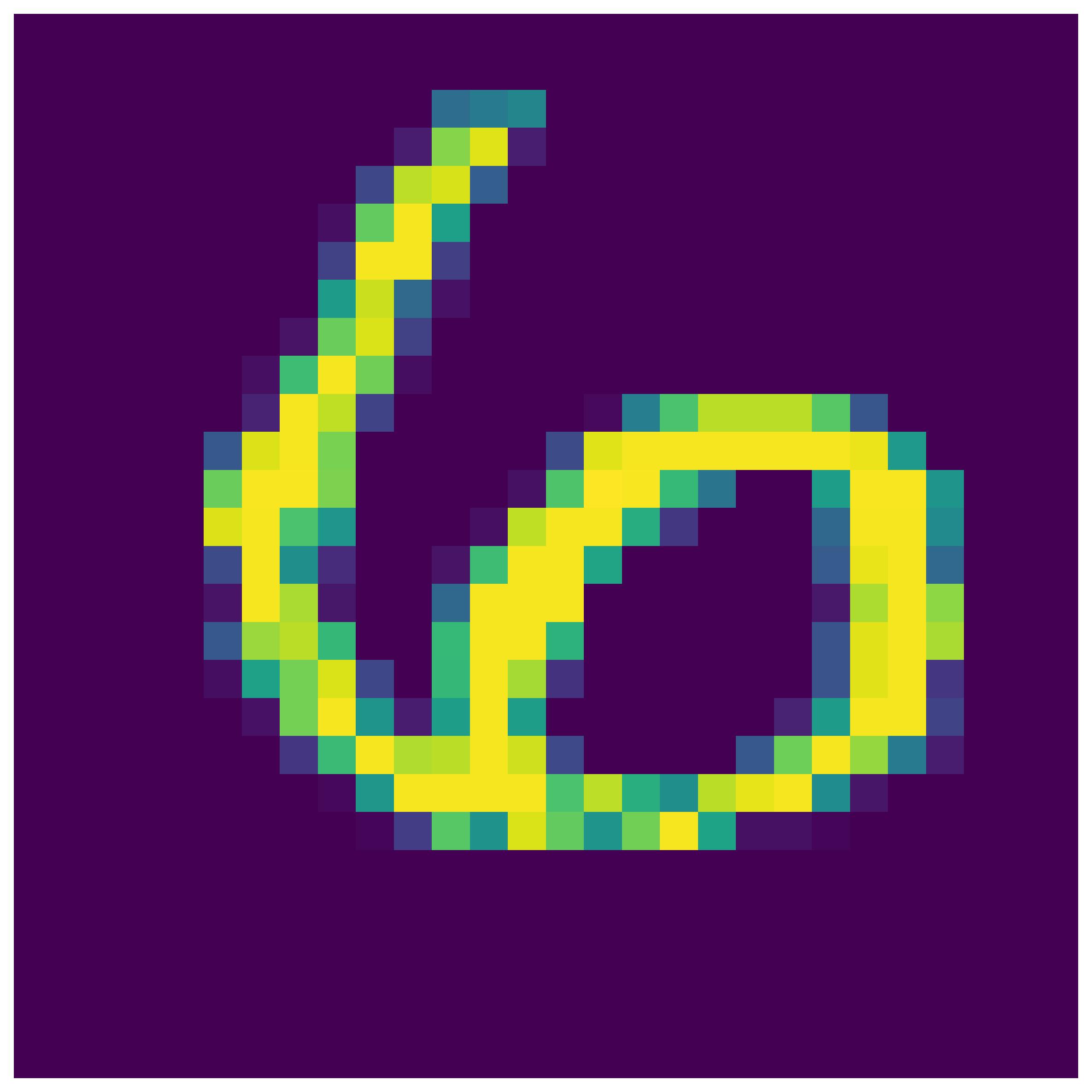}%
    \includegraphics[width=0.08\linewidth]{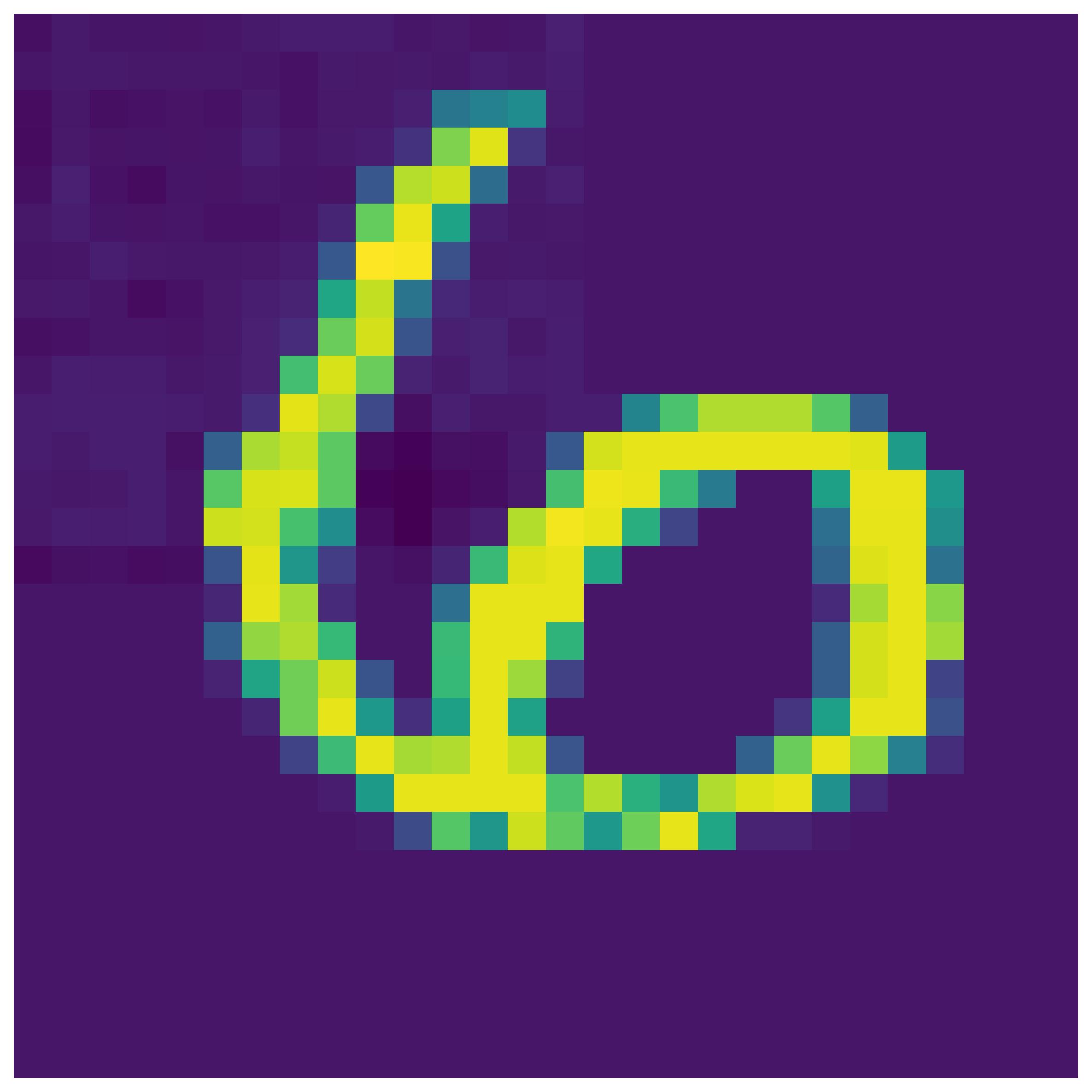}%
    };%
    \node[inner xsep=0pt, inner ysep=4pt, yshift=2pt, anchor=north, scale=0.70] at (I.south) {%
    \begin{tabular}{c|ccc}
    \hline
    \multirow{2}{*}{C} & L & 6 & 0 \\
    \cline{2-4}
    & V & 8.3944 & 3.0419 \\
    \hline
    \multirow{2}{*}{P} & L & 8 & 2 \\
    \cline{2-4}
    & V & 4.2719 & 1.8364 \\
    \hline
    \end{tabular}
    };%
\end{tikzpicture}
\begin{tikzpicture}
    \node[inner sep=0pt] (I) at (0,0) {%
    \includegraphics[width=0.08\linewidth]{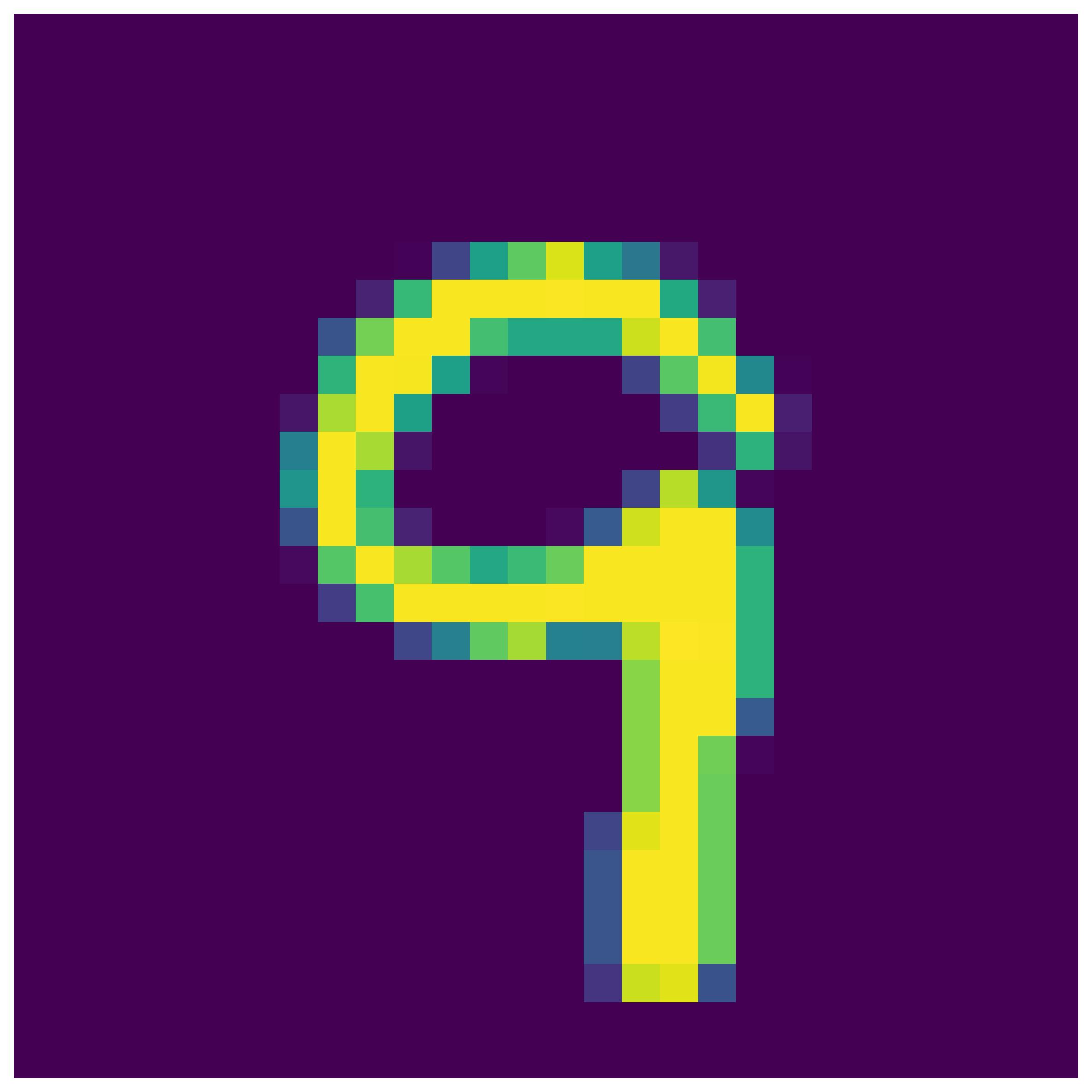}%
    \includegraphics[width=0.08\linewidth]{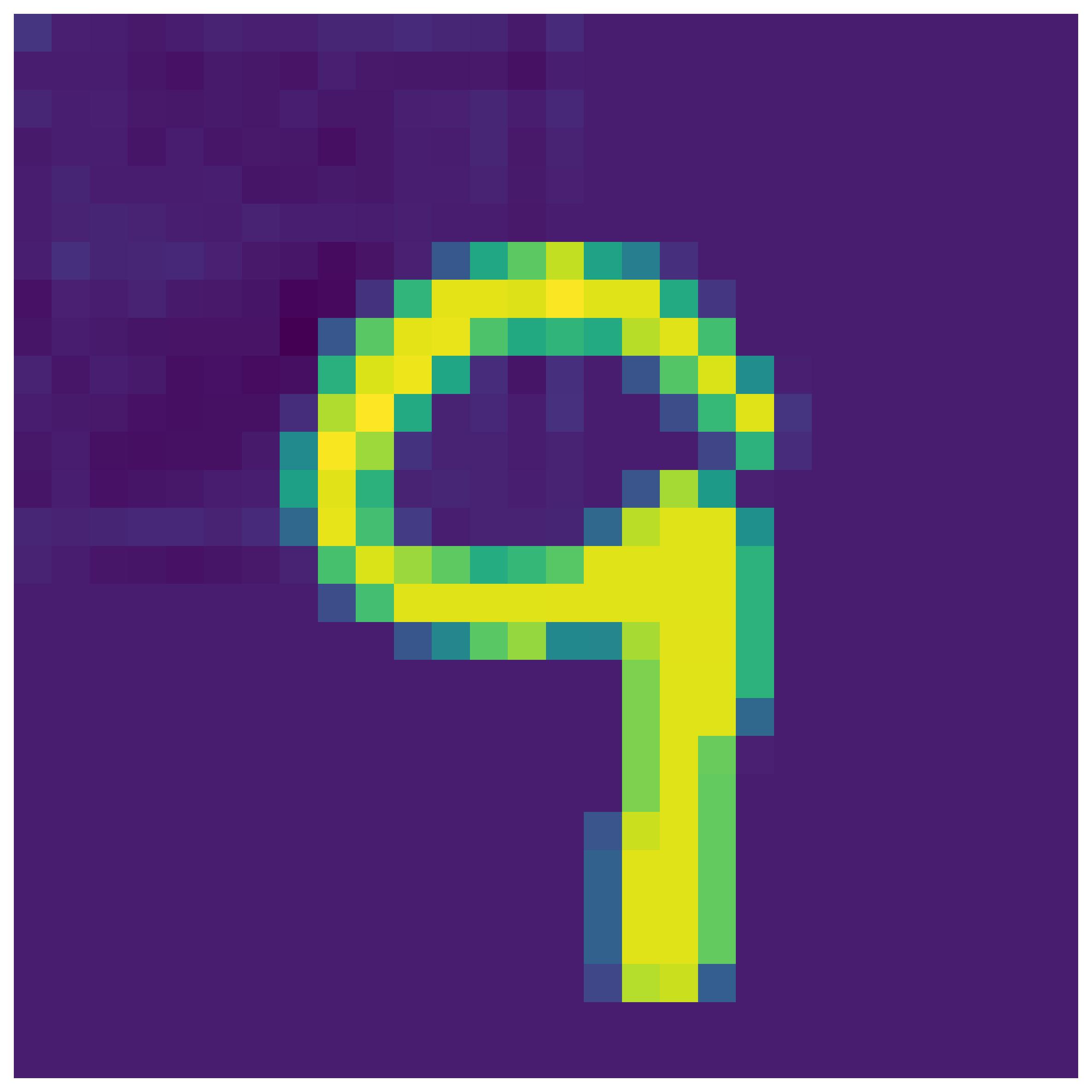}%
    };%
    \node[inner xsep=0pt, inner ysep=4pt, yshift=2pt, anchor=north, scale=0.70] at (I.south) {%
    \begin{tabular}{c|ccc}
    \hline
    \multirow{2}{*}{C} & L & 9 & 4 \\
    \cline{2-4}
    & V & 8.8827 & 1.5762 \\
    \hline
    \multirow{2}{*}{P} & L & 8 & 5 \\
    \cline{2-4}
    & V & 3.8951 & 2.1197 \\
    \hline
    \end{tabular}
    };%
\end{tikzpicture}
\begin{tikzpicture}
    \node[inner sep=0pt] (I) at (0,0) {%
    \includegraphics[width=0.08\linewidth]{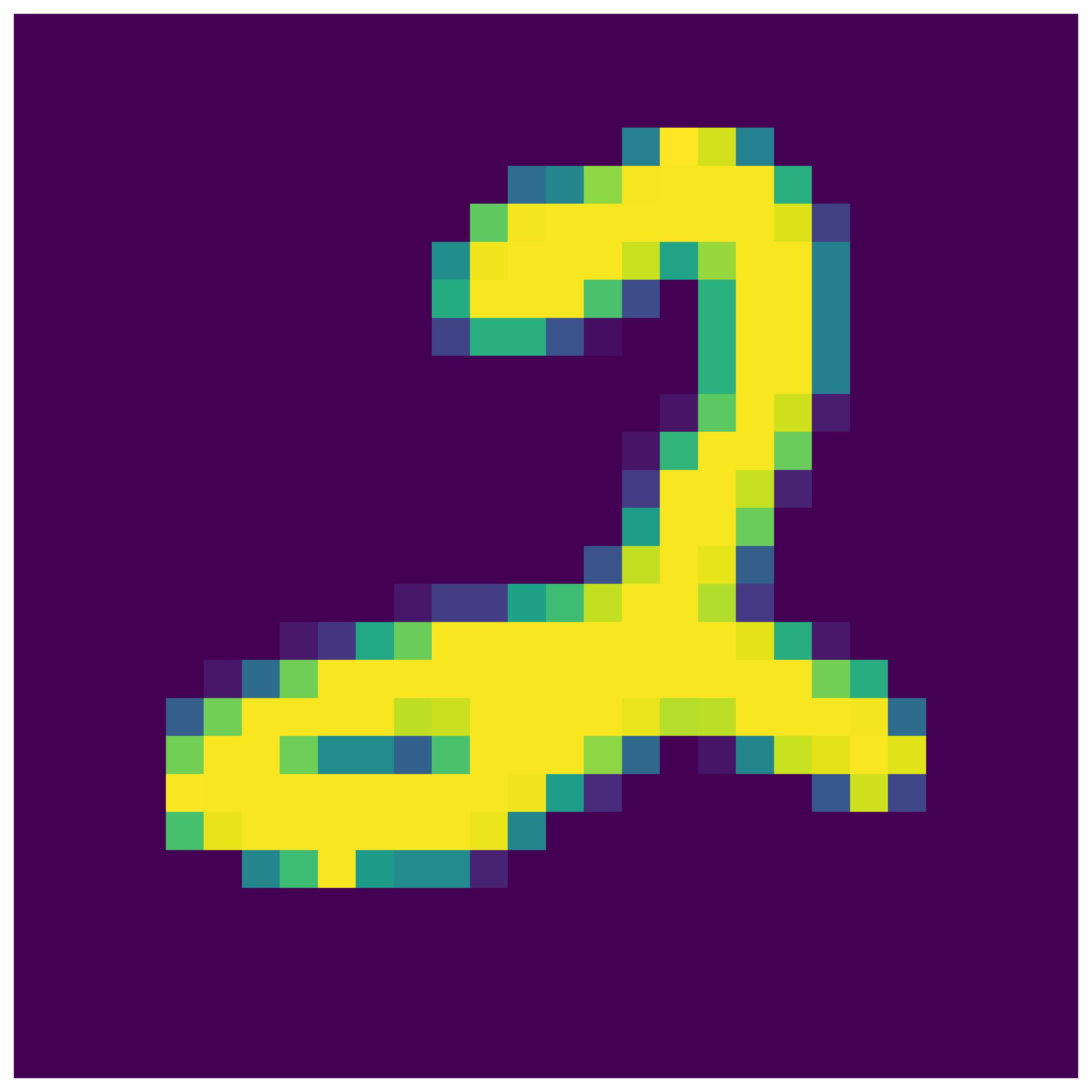}%
    \includegraphics[width=0.08\linewidth]{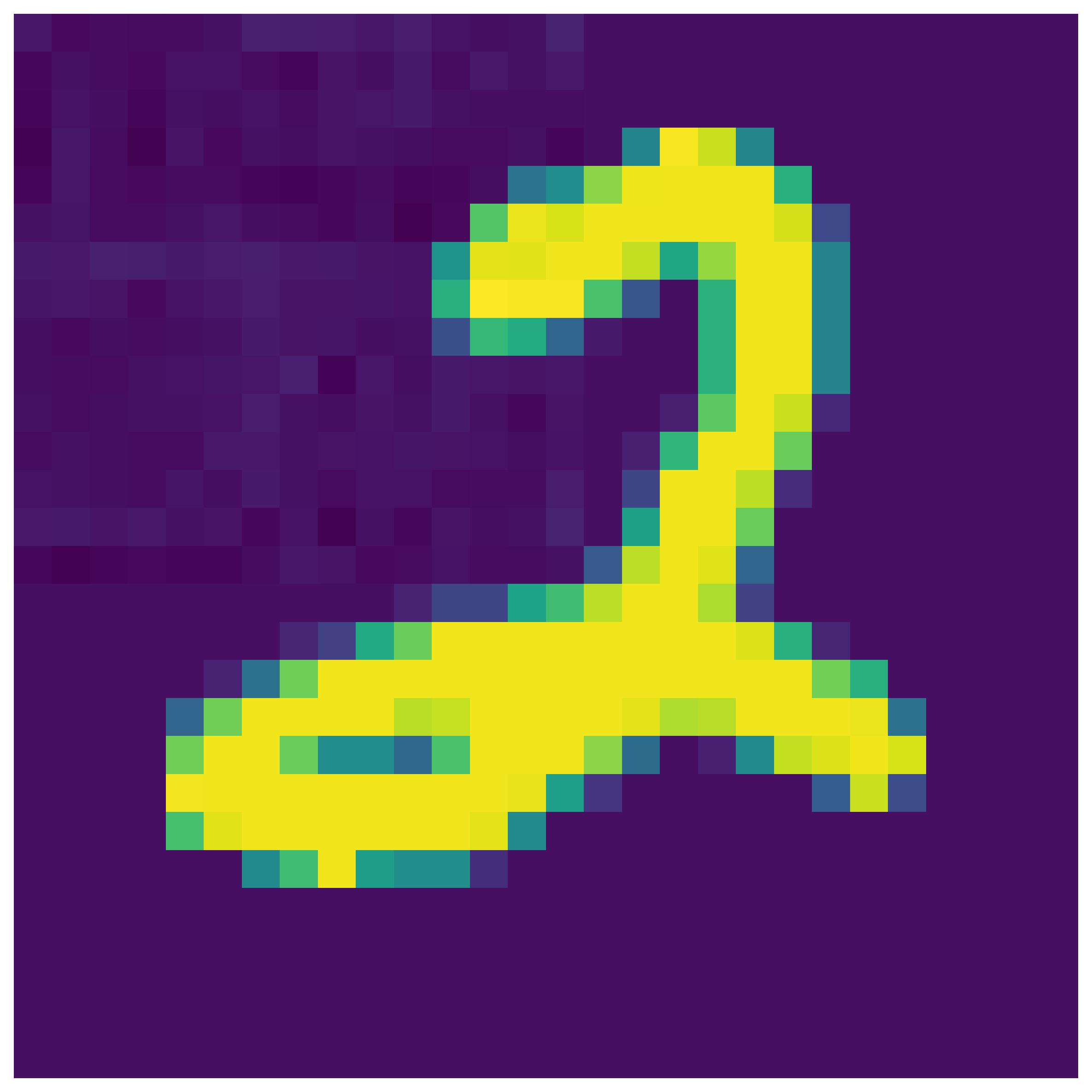}%
    };%
    \node[inner xsep=0pt, inner ysep=4pt, yshift=2pt, anchor=north, scale=0.70] at (I.south) {%
    \begin{tabular}{c|ccc}
    \hline
    \multirow{2}{*}{C} & L & 2 & 7 \\
    \cline{2-4}
    & V & 9.5912 & 1.4444 \\
    \hline
    \multirow{2}{*}{P} & L & 8 & 2 \\
    \cline{2-4}
    & V & 1.8954 & 1.7449 \\
    \hline
    \end{tabular}
    };%
\end{tikzpicture}
\begin{tikzpicture}
    \node[inner sep=0pt] (I) at (0,0) {%
    \includegraphics[width=0.08\linewidth]{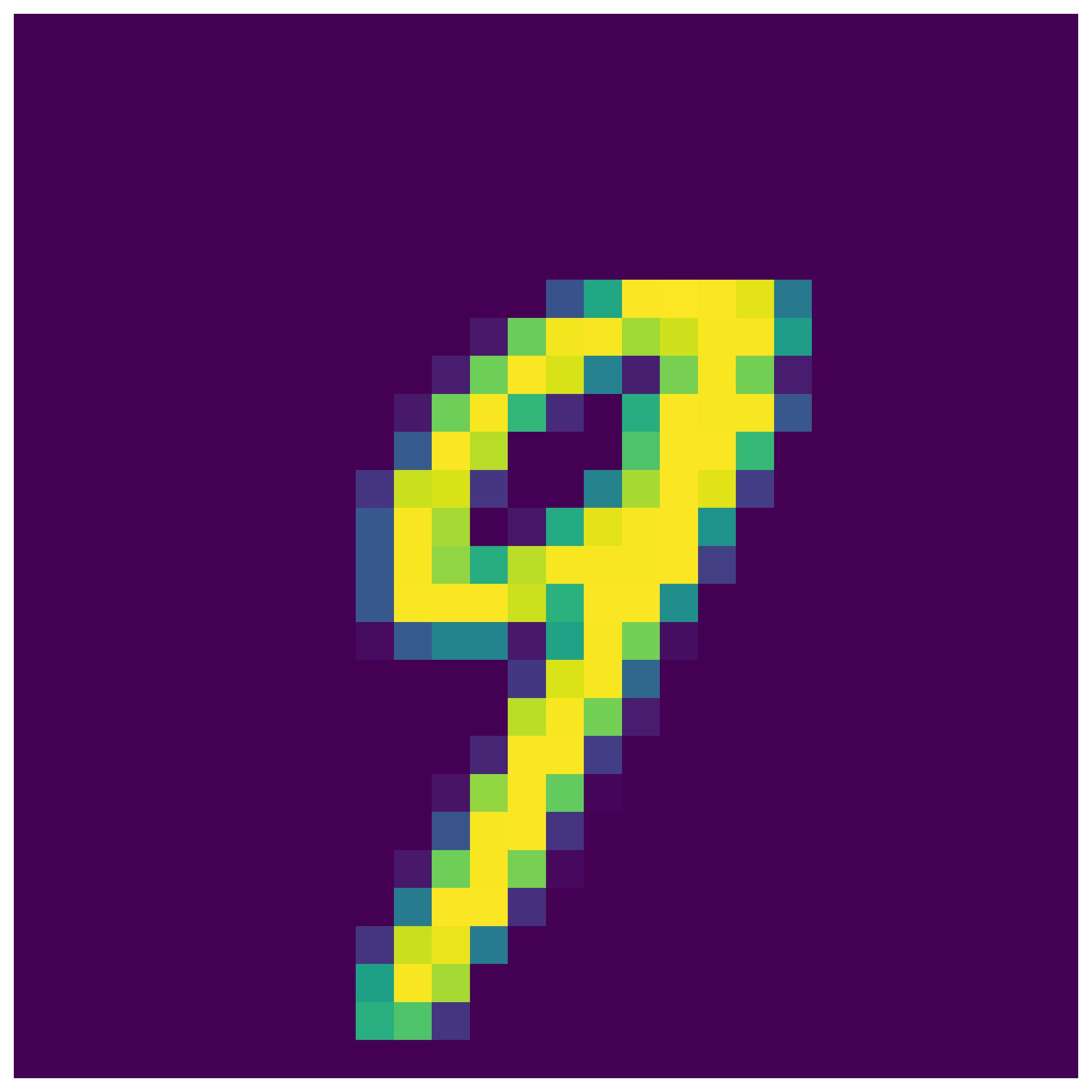}%
    \includegraphics[width=0.08\linewidth]{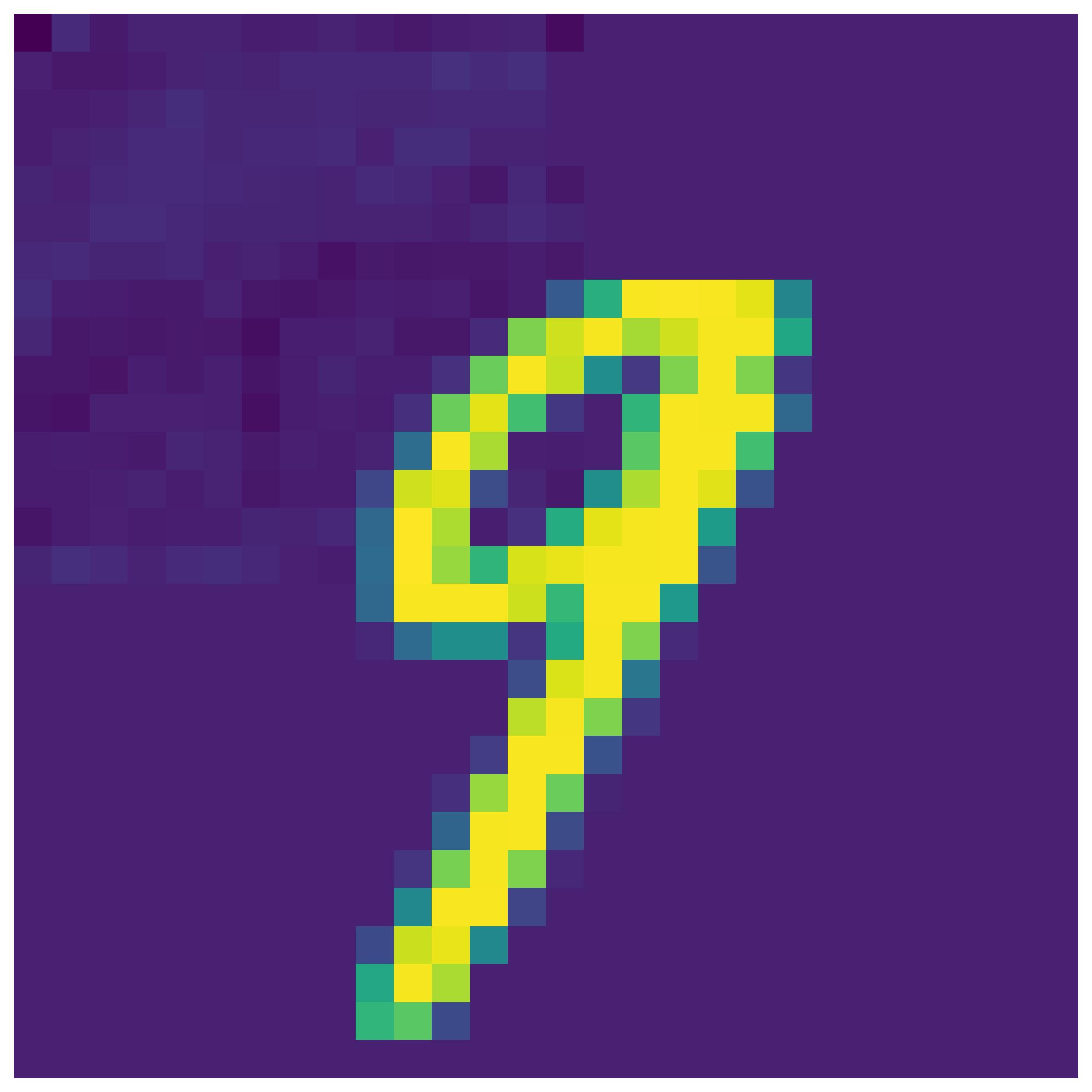}%
    };%
    \node[inner xsep=0pt, inner ysep=4pt, yshift=2pt, anchor=north, scale=0.70] at (I.south) {%
    \begin{tabular}{c|ccc}
    \hline
    \multirow{2}{*}{C} & L & 9 & 4 \\
    \cline{2-4}
    & V & 7.8130 & 1.9318 \\
    \hline
    \multirow{2}{*}{P} & L & 8 & 5 \\
    \cline{2-4}
    & V & 3.9808 & 1.4411 \\
    \hline
    \end{tabular}
    };%
\end{tikzpicture}
\begin{tikzpicture}
    \node[inner sep=0pt] (I) at (0,0) {%
    \includegraphics[width=0.08\linewidth]{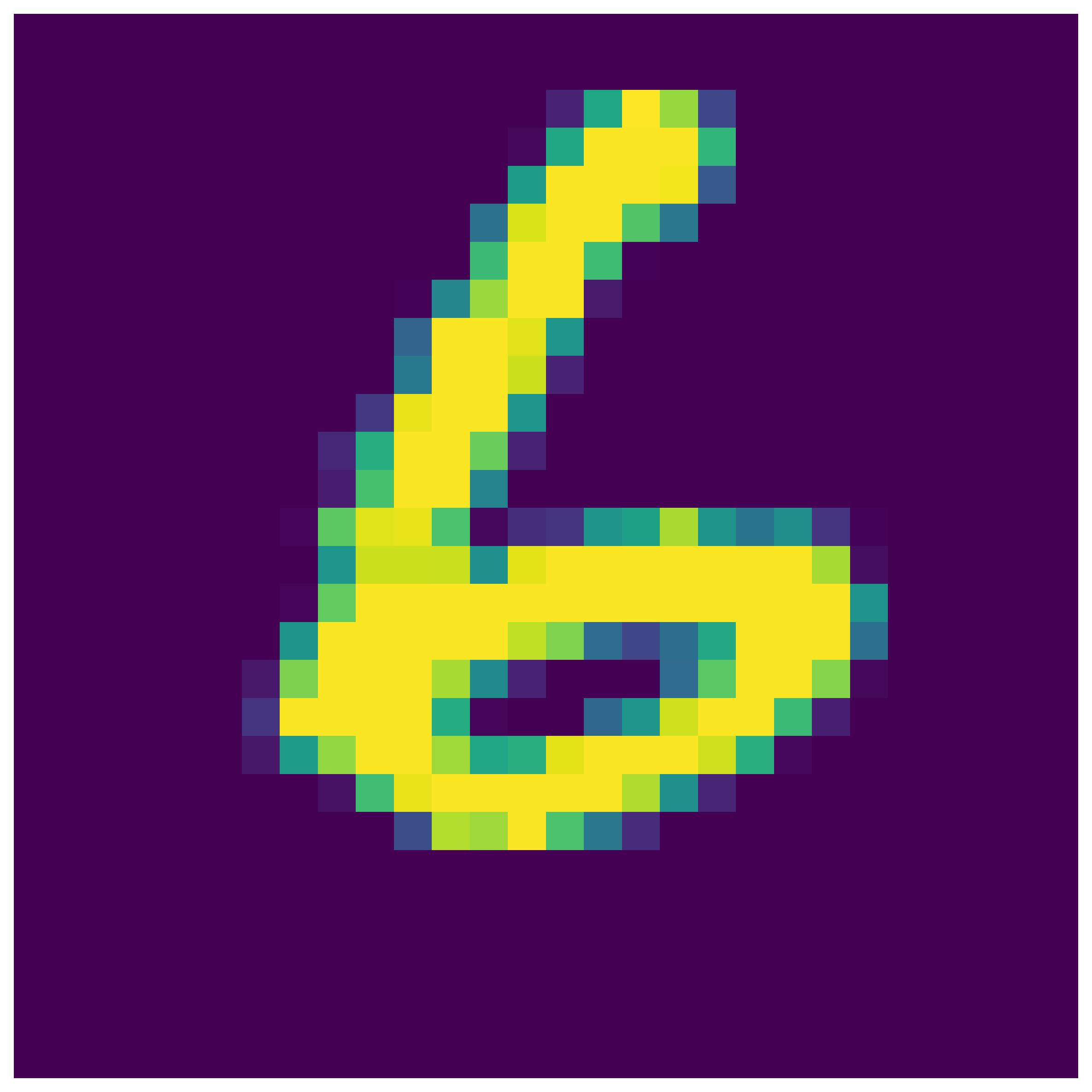}%
    \includegraphics[width=0.08\linewidth]{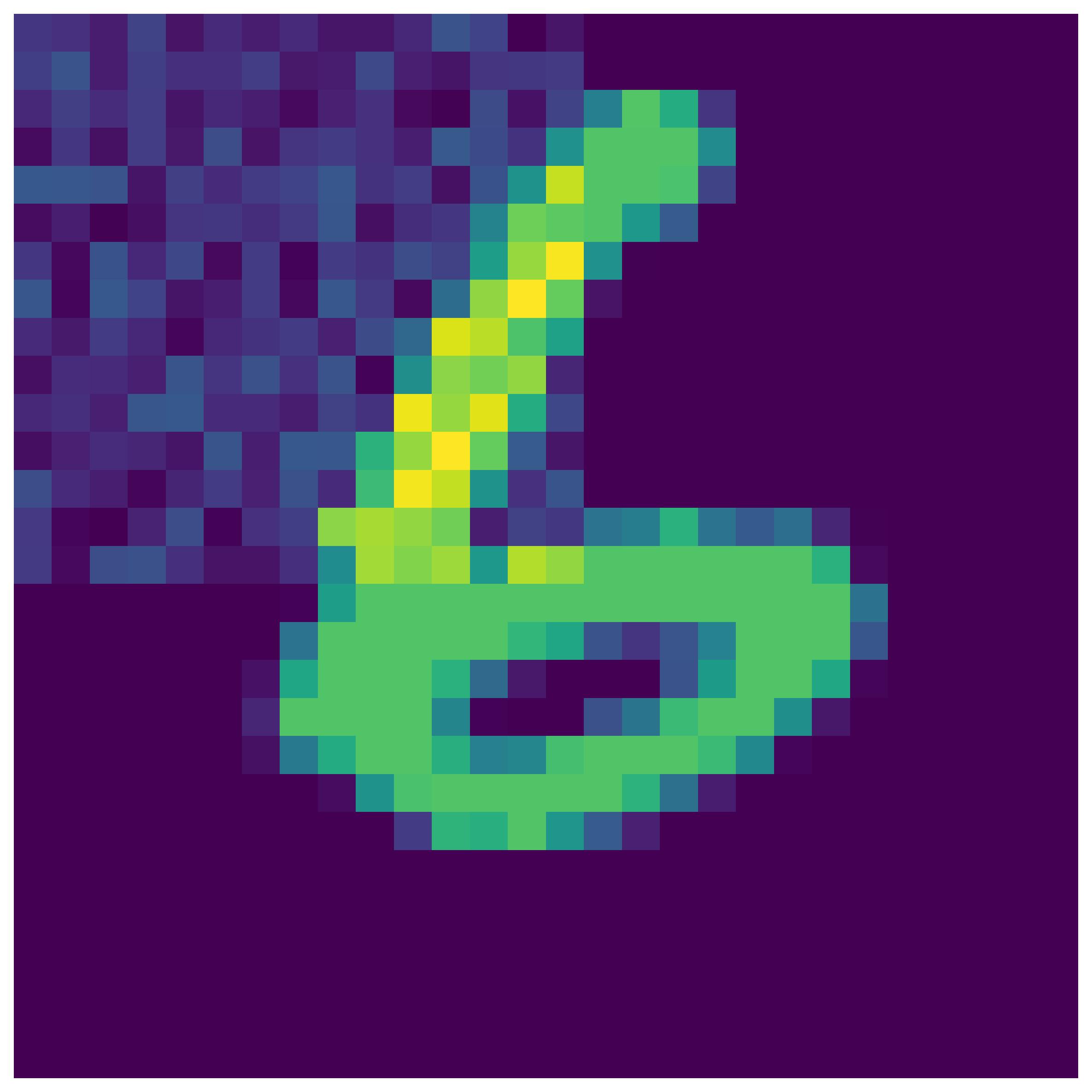}%
    };%
    \node[inner xsep=0pt, inner ysep=4pt, yshift=2pt, anchor=north, scale=0.70] at (I.south) {%
    \begin{tabular}{c:ccc}
    \hdashline
    \multirow{2}{*}{C} & L & 6 & 5 \\
    \cdashline{2-4}
    & V & 9.4836 & 2.8961 \\
    \hdashline
    \multirow{2}{*}{P} & L & 6 & 8 \\
    \cdashline{2-4}
    & V & 4.4007 & 2.6830 \\
    \hdashline
    \end{tabular}
    };%
\end{tikzpicture}
\begin{tikzpicture}
    \node[inner sep=0pt] (I) at (0,0) {%
    \includegraphics[width=0.08\linewidth]{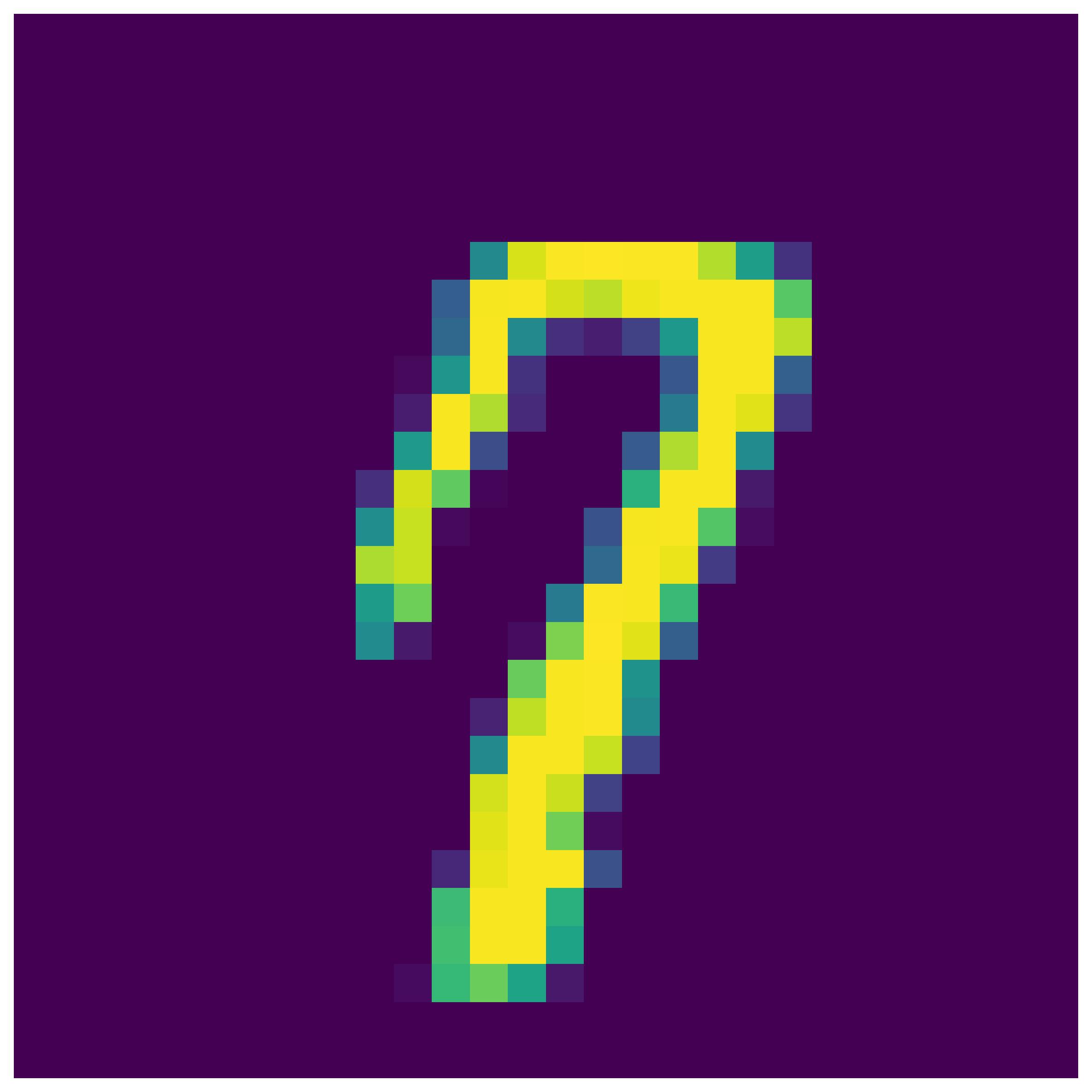}%
    \includegraphics[width=0.08\linewidth]{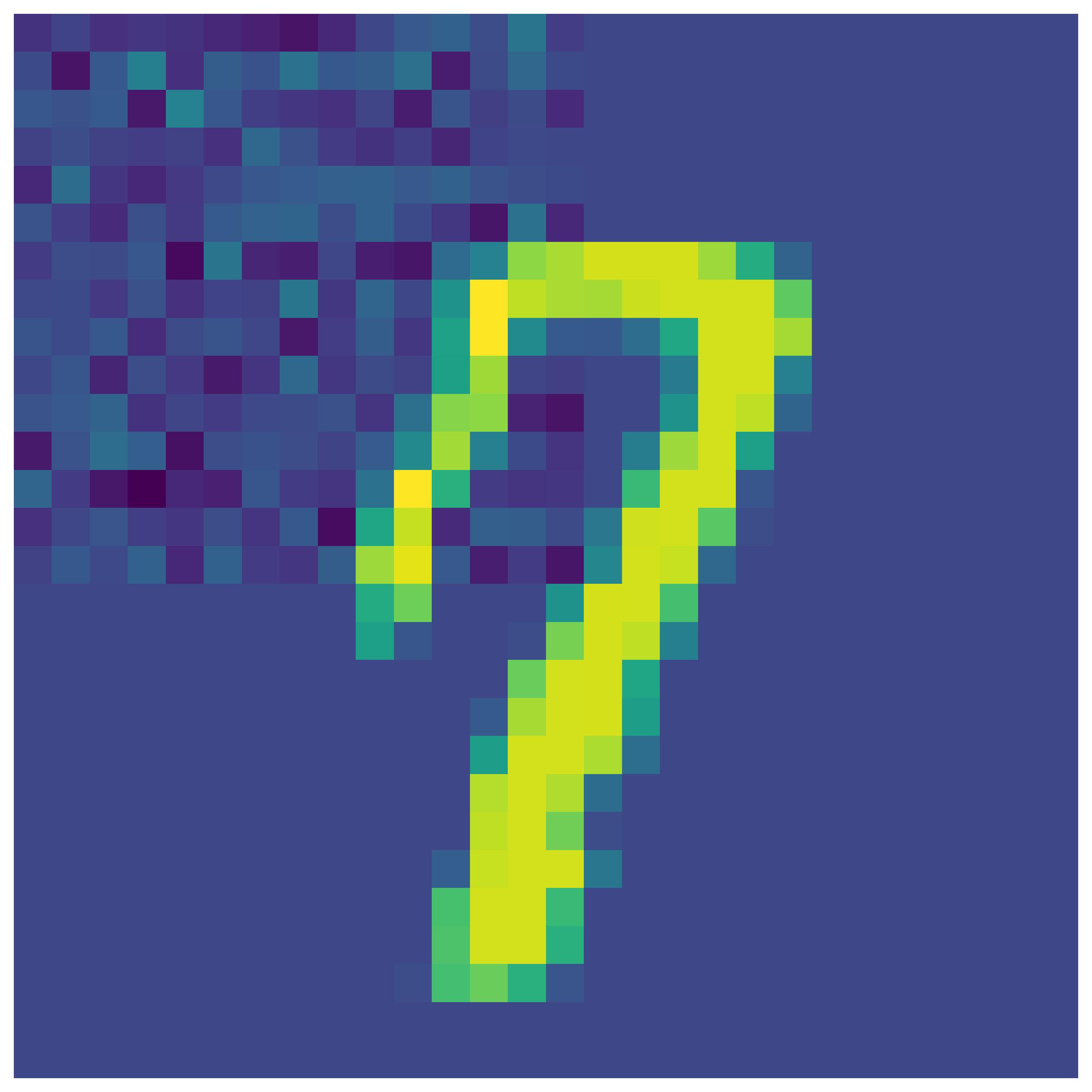}%
    };%
    \node[inner xsep=0pt, inner ysep=4pt, yshift=2pt, anchor=north, scale=0.70] at (I.south) {%
    \begin{tabular}{c:ccc}
    \hdashline
    \multirow{2}{*}{C} & L & 7 & 9 \\
    \cdashline{2-4}
    & V & 6.9076 & 2.4816 \\
    \hdashline
    \multirow{2}{*}{P} & L & 7 & 8 \\
    \cdashline{2-4}
    & V & 3.4972 & 1.5161 \\
    \hdashline
    \end{tabular}
    };%
\end{tikzpicture}
  \caption{Predictions of MNIST digits polluted by perturbations. For each frame, top left is the original image \(\tvar{x}_i\), top right is the patched image \(\tvar{x}_i + \epsilon_i\), and table below is prediction. Block ``C'' indicate top two predict values of \(\tvar{x}_i\), and block ``P'' indicate top two predict values of \(\tvar{x}_i + \epsilon_i\). Line ``L'' indicate labels and Line ``V'' indicate values.}
  \label{fig:mnist-digits-perturbed-by-specialized-perturbations}
\end{figure}

This phenomenon is similar to the adversarial example \citep{Szegedy2013,Goodfellow2014}, adding the human invisible perturbation to an image can mislead the network. Nevertheless, these two methods are different. They iteratively compute perturbations that mislead the network with or without specific objects. Our method is to maximize the energy of order-one convolution, and this increases the probability to change output of the network.
More details about the perturbations are presented in Appendix \ref{appendix:perturbations-in-vconv}.

\section{Discussion}
\label{sec:discussion}

This paper presents a new perspective on the analysis of convolutional neural networks.
It shows that a convolutional neural network can be represented in the form of Volterra convolution.
By approximating infinite Volterra convolution using the finite one, 
we can analyze the proxy kernels of Volterra convolution rather than directly analyze the original network, even if the structure and the parameters of the network are unknown.
The proxy kernels can be computed from the weights of a neural network, or be approximated by building and training a hacking network. The method of approximating the proxy kernels is more intuitive and easier to implement if the number of layers is large, and the approximated proxy kernels are comparable with the actual kernels.

In the future work, we plan to do further research on the proxy kernels, i.e., how to compute the high order kernels efficiently, how to determine the order of kernels with given approximation accuracy, etc. We will also investigate the kernels that are input related, i.e., the dynamic networks, as well as practical applications of the hacking network.

All codes associated with this article are accessible publicly at GitHub \url{https://github.com/tenghuilee/nnvolterra.git}.

\acks{The authors thank Andong Wang for helpful discussions. Sincere thanks to all anonymous reviewers, that greatly helped to improve this paper.}

\appendix

\section{Convolution From Order One to Two}
\label{appendix:convolution-from-order-1-to-2}

  In Appendix \ref{appendix:convolution-from-order-1-to-2}, we will show how to convert a one-dimensional order-one convolution to a one-dimensional order-two convolution in continuous time domain only.
  By the well-known continuous one-dimensional convolution (Equation \ref{equ:def-convolution-1d}), and the differential property of one-dimensional convolution, we have
  \begin{equation*}
    \begin{aligned}
      \left(\tvar{h} * \tvar{x}\right)(t)
       & = \intinf h(\tau) x(t - \tau) d \tau                                                         \\
       & = \intinf h(t - \tau) x(\tau) d \tau                                                         \\
       & = \intinf \dfrac{d h(t - \tau)}{d \tau} \left(\int_{-\infty}^{\tau} x(l) d l \right) d \tau. \\
    \end{aligned}
  \end{equation*}
  Applying the integration, \(\int_{-\infty}^{\tau} x(l) d l = 1/2 x(\tau) x(\tau) - 0\), we have
  \begin{equation*}
    \begin{aligned}
      \left(\tvar{h} * \tvar{x}\right)(t)
       & = \intinf \dfrac{1}{2} \dfrac{d h(t - \tau)}{d t} x(\tau) x(\tau) d \tau                                 \\
       & = \iintinf \dfrac{1}{2} \dfrac{d h(t - \tau)}{d t} x(\tau) x(\iota) \delta(\tau - \iota) d \tau d \iota, \\
    \end{aligned}
  \end{equation*}
  where Dirac delta \(\delta(t) = \left\{\begin{array}{lc}
    \infty, & t = 0   \\
    0,      & t \ne 0
  \end{array}\right.\). Let \(g(t - \tau, t - \iota) = \dfrac{1}{2} \dfrac{d h(t - \tau)}{d t} \delta(\tau - \iota)\). We have
  \begin{equation*}
    \begin{aligned}
      \left(\tvar{h} * \tvar{x}\right)(t)
       & = \iintinf g(t - \tau, t - \iota) x(\tau) x(\iota) d \tau d \iota \\
       & = \iintinf g(\tau, \iota) x(t - \tau) x(t - \iota) d \tau d \iota \\
       & = \left(\tvar{g} * \lcerfl{\tvar{x}, \tvar{x}}\right)(t).
    \end{aligned}
  \end{equation*}

During this, we are surprised to find that \(\tvar{g}\) is insensitive to time irrelevant additive noise.
Suppose \(\hat{h}(\tau) = h(\tau) + n(\tau)\), with \(\dfrac{d n(\tau)}{d \tau} = 0\), we have
\begin{equation*}
  g(t-\tau, t - \iota) = \dfrac{1}{2} \dfrac{d \hat{h}(t - \tau)}{d t} \delta(\tau - \iota) = \dfrac{1}{2}\dfrac{d h(t - \tau)}{d \tau} \delta(\tau - \iota).
\end{equation*}

\section{Proof for Combination Properties}
\label{appendix:proof-for-combination-properties}

In this appendix, only properties of discrete one-dimensional signals are proved, and this proof can be generalized to other situations.
To prevent indexing out of bound, zero padding is always considered.

\begin{proof}

  Property \ref{prop:conv-g-plus-x-y} and \ref{prop:conv-plus-g-h-x}: They can be proved by linearity of integration or summation.

  Property \ref{prop:conv-g-add-x-a}: Due to linear property of integration or summation, we can separate these two term as
  \begin{equation*}
    \tvar{G} * (\vsymb{x} + \alpha)
    = \tvar{G} * \vsymb{x} + \alpha \sum_{\vsymb{t}} G(\vsymb{t}).
  \end{equation*}

  Property \ref{prop:conv-g-square-plus-a-b}, \ref{prop:conv-g-power-n-plus-a-b} and \ref{prop:conv-g-power-n-plus-a-c-b-d}: We will prove Property \ref{prop:conv-g-power-n-plus-a-c-b-d}. Noticing that \(\tvar{G}\) is symmetric, swapping order of \(\tvar{x}_i\) does not change the result.
  \begin{equation*}
    \begin{aligned}
       & \left(\tvar{G} * (\tvar{x}_1 + \tvar{x}_2 + \cdots + \tvar{x}_m)^n\right)(t)                                                                   \\
       & = \sum_{\tau_1, \cdots, \tau_n} G(\tau_1, \cdots, \tau_n) \prod_{i=1}^{n} (x_1(t - \tau_i) + x_2(t - \tau_i) + \cdots + x_m(t - \tau_i))       \\
       & = \sum_{\tau_1, \cdots, \tau_n} \sum \binom{n}{n_1 n_2 \cdots n_m} G(\tau_1, \cdots, \tau_n)                                                   \\
       & \qquad\qquad x_1(t - \tau_1) x_1(t - \tau_2) \cdots x_1(t-\tau_{n_1})                                                                          \\
       & \qquad\qquad x_2(t - \tau_{1 + n_1}) x_2(t - \tau_{2 + n_1}) \cdots x_2(t-\tau_{n_2 + n_1})                                                    \\
       & \qquad\qquad \cdots                                                                                                                            \\
       & \qquad\qquad x_m(t - \tau_{1 + \sum_{k=1}^{m-1} n_k}) x_m(t - \tau_{2 + \sum_{k=1}^{m-1} n_k}) \cdots x_m(t-\tau_{n_m + \sum_{k=1}^{m-1} n_k}) \\
       & = \left(\sum \binom{n}{n_1 n_2 \cdots n_m} \tvar{G} * \lcerfl{\tvar{x}_1^{n_1}, \tvar{x}_2^{n_2}, \cdots, \tvar{x}_m^{n_m}}\right)(t),         \\
    \end{aligned}
  \end{equation*}
  where \(\binom{n}{n_1 n_2 \cdots n_m} = \dfrac{n!}{n_1! n_2! \cdots n_m!}\), the multinomial coefficient, and \(\sum_{i=1}^{m} n_i = m, n_i \ge 0\), for all \(i = 1, 2, \cdots, m\).
  If \(m=2\), this is Property \ref{prop:conv-g-power-n-plus-a-b}, and if \(m=2\) and \(n=2\), this is Property \ref{prop:conv-g-square-plus-a-b}.

  Property \ref{prop:conv-g-conv-h-x} and \ref{prop:conv-stride-s-g-conv-stride-z-h-x}: Only Property \ref{prop:conv-stride-s-g-conv-stride-z-h-x} is proved, and we can prove Property \ref{prop:conv-g-conv-h-x} by setting stride equal to one.
  \begin{equation*}
    \begin{aligned}
      \left(\tvar{G} *_s (\tvar{H} *_z \vsymb{x})\right)(t)
       & = \sum_{l} G(l) \sum_{\vsymb{\tau}} H(\vsymb{\tau}) \prod_{i=1}^{n} x_i( z (s t - l) - \tau_i)              \\
       & = \sum_{\vsymb{\tau}} \left( \sum_{l} G(l) H(\vsymb{\tau} - z l) \right) \prod_{i=1}^{n} x_i(sz t - \tau_i) \\
       & = \sum_{\vsymb{\tau}} (\tvar{G} \oconv_z \tvar{H})(\vsymb{\tau}) \prod_{i=1}^{n} x_i(s z t - \tau_i)        \\
       & = \left((\tvar{G} \oconv_z \tvar{H}) *_{sz} \vsymb{x}\right)(t).
    \end{aligned}
  \end{equation*}

  Property \ref{prop:conv-g-mul-conv-h1-x-conv-h2-y} and \ref{prop:conv-g-prod-n-conv-h-x}: 
  We will prove Property \ref{prop:conv-g-mul-conv-h1-x-conv-h2-y} here and this proof could be intuitively extended to Property \ref{prop:conv-g-prod-n-conv-h-x}.
  \begin{equation*}
    \begin{aligned}
       & \left(\tvar{G} * \lcerfl{\tvar{H}_1 * \vsymb{x}, \tvar{H}_2 * \vsymb{y}}\right)(t)                   \\
       & = \sum_{l_1,l_2} G(l_1, l_2)
      \left( \sum_{\vsymb{\tau}_1} H_1(\vsymb{\tau}_1) \prod_{i=1}^{n} x_i(t - \tau_{1,i} - l_1) \right)
      \left( \sum_{\vsymb{\tau}_2} H_2(\vsymb{\tau}_2) \prod_{j=1}^{m} y_j(t - \tau_{2,j} - l_2) \right)      \\
       & = \sum_{\vsymb{\tau}_1, \vsymb{\tau}_2} \left( \sum_{l_1, l_2}
      G(l_1, l_2) H_1(\vsymb{\tau}_1 - l_1) H_2(\vsymb{\tau}_2 - l_2)
      \right)
      \left(\prod_{i=1}^{n} x_i(t - \tau_{1,i})\right) \left(\prod_{j=1}^{m} y_j(t - \tau_{2,j})\right)       \\
       & = \sum_{\vsymb{\tau}_1, \vsymb{\tau}_2}
      (\tvar{G} \oconv \lcerfl{\tvar{H}_1, \tvar{H}_2})(\vsymb{\tau}_1, \vsymb{\tau}_2)
      \left(\prod_{i=1}^{n} x_i(t - \tau_{1,i})\right) \left(\prod_{j=1}^{m} y_j(t - \tau_{2,j})\right)       \\
       & = \left((\tvar{G} \oconv \lcerfl{\tvar{H}_1, \tvar{H}_2}) * \lcerfl{\vsymb{x}, \vsymb{y}}\right)(t).
    \end{aligned}
  \end{equation*}

  Property \ref{prop:oconv-g1-oconv-g2-g3}:
  \begin{equation*}
    \begin{aligned}
      \left(\tvar{G}_1 \oconv (\tvar{G}_2 \oconv \tvar{G}_3)\right)(\vsymb{\tau})
       & = \sum_{k} G_1(k) \left( \sum_{l} G_2(l) G_3(\vsymb{\tau} - k - l) \right)     \\
       & = \sum_{l} \left( \sum_{k} G_1(k) G_2(l - k) \right) G_3(\vsymb{\tau} - l)     \\
       & = \left((\tvar{G}_1 \oconv \tvar{G}_2) \oconv \tvar{G}_3\right)(\vsymb{\tau}).
    \end{aligned}
  \end{equation*}

  Property \ref{prop:conv-g-conv-alpha-conv-h-x}, \ref{prop:conv-g-conv-h-conv-alpha-x}, and Property \ref{prop:conv-g-conv-conv-h1-x-alpha-conv-h2-y}: 
  We will only prove Property \ref{prop:conv-g-conv-conv-h1-x-alpha-conv-h2-y} here, and we can prove Property \ref*{prop:conv-g-conv-alpha-conv-h-x} and \ref{prop:conv-g-conv-h-conv-alpha-x} in the same way.
    \begin{equation*}
      \begin{aligned}
         & \left(\tvar{G} * \lcerfl{\tvar{H}_1 * \vsymb{x}, \alpha, \tvar{H}_2 * \vsymb{y}}\right)(t)                                                                  \\
         & = \alpha \sum_{l_1,l_2,l_3} G(l_1, l_2, l_3)
        \left( \sum_{\vsymb{\tau}_1} H_1(\vsymb{\tau}_1) \prod_{i=1} x_i(t - \tau_{1,i} - l_1) \right)
        \left( \sum_{\vsymb{\tau}_3} H_2(\vsymb{\tau}_3) \prod_{i=1} y_i(t - \tau_{3,i} - l_3) \right)                                                                 \\
         & = \alpha \sum_{\vsymb{\tau}_1, \vsymb{\tau}_2}
        \!\!\left(\! \sum_{l_1,l_2,l_3} G(l_1, l_2, l_3) H_1(\vsymb{\tau}_1 - l_1) H_2(\vsymb{\tau}_3 - l_3) \!\right)
        \!\!\left(\! \prod_{i=1} x_i(t - \tau_{1,i} - l_1) \!\right)
        \!\!\left(\! \prod_{i=1} y_i(t - \tau_{3,i} - l_3) \!\right)                                                                                                   \\
         & = \left(\left(\sum_{\dsmark \alpha} \left(\tvar{G} \oconv \lcerfl{\tvar{H}_1, \alpha, \tvar{H}_2}\right) \right) * \lcerfl{\vsymb{x}, \vsymb{y}}\right)(t). \\
      \end{aligned}
    \end{equation*}

  Property \ref{prop:conv-h-element-power-n-x}:
  \begin{equation*}
    \begin{aligned}
      \left( \tvar{h} * [\tvar{x}]^n \right)(t)
       & = \sum_{l} h(l) \left( x(t - l) \right)^n                                                                    \\
       & = \sum_{\tau_1, \tau_2, \cdots, \tau_n} h(l) \prod_{i=1}^{n} \left( \delta(\tau_i - l) x(t - \tau_i) \right) \\
       & = \sum_{\tau_1, \tau_2, \cdots, \tau_n}
      \left( \diag(n, \tvar{h}) \right)(\tau_1, \tau_2, \cdots, \tau_n)
      \prod_{i=1}^{n} x(t - \tau_i)                                                                                   \\
       & = \left(\diag(n, \tvar{h}) * \tvar{x}^n\right)(t).
    \end{aligned}
  \end{equation*}

  Property \ref{prop:conv-g-power-conv-h-x}:
  \begin{equation*}
    \tvar{g} * [\tvar{h} * \tvar{x}]^n
    = \diag(n, \tvar{g}) * (\tvar{h} * \tvar{x})^n
    = (\diag(n, \tvar{g}) \oconv \tvar{h}^n) * \tvar{x}^n.
  \end{equation*}
\end{proof}

\section{Perturbations in Volterra Convolution}
\label{appendix:perturbations-in-vconv}

During practice, it has been found that some specific perturbations have the potential to make neural networks behave abnormally \citep{Szegedy2013,Goodfellow2014}.
An observation of how perturbations influence the proxy kernels will be theoretically explored in this appendix.
The neural network is extremely complex, our analysis limits on its impacts on the proxy kernels.

\subsection{Perturbation Upper Bound}
\label{subsec:perturbation-upper-bound}

In this subsection, we express the upper bound for perturbation of ideal Volterra convolution, showing that how the perturbation  change the output.

\begin{theorem}
  Assume input signal is \(\tvar{x}\), and the perturbation is \(\epsilon\), the approximated neural network is \(f(\tvar{x}) = \sum_{n=0}^{N} \tvar{H}_n * \tvar{x}^n\), we have
  \begin{equation}
    \|f(\tvar{x} + \epsilon) - f(\tvar{x})\|_2
    \le \min\left(\begin{aligned}
         & \sum_{n=0}^{N} \| \tvar{H}_n \|_2 \sum_{k=0}^{n-1} \left(\dfrac{e n}{k}\right)^k \|\tvar{x}\|_1^k \| \epsilon \|_1^{n-k},                 \\
         & \sum_{n=0}^{N} \| \tvar{H}_n \|_1 \sum_{k=0}^{n-1} \left(\dfrac{e n}{k}\right)^k \|\tvar{x}\|_{2k}^k \| \tvar{\epsilon} \|_{2(n-k)}^{n-k}
      \end{aligned}\right),
  \end{equation}
  where \(e = 2.718281828\cdots\), the base of the natural logarithm.
  \label{thm:vconv-perturbation-upper-bound}
\end{theorem}

\begin{proof}
  See appendix \ref{subsec:proof-for-theorem-nn-is-robust-and-fragile}.
\end{proof}

The following is an example of effect of perturbations on each convolution from order one to eight.
We randomly generate \(\tvar{x} \in \mathbb{M}^{32}\) (\(\mathbb{M}\) is defined in Equation \ref{equ:define-the-set-of-gaussian-l2-norm-set-to-1}).
Two perturbations are produced to simulate those two different cases.
To make the plot more clear, we add a considerably large number \(3.0\) or a considerably small value \(0.5\) to the middle point of \(\tvar{x}\).
For each perturbation and each convolution with order form one to eight, we compute
  \(\|\tvar{H}_n * (\tvar{x} + \epsilon)^n - \tvar{H}_{n} * \tvar{x}^n\|_2\) one thousand times with random \(\tvar{H}_n \in \mathbb{M}^{5 \times \cdots \times 5}\). The boxplot is shown in Figure \ref{fig:statistic-l2norm-diff-conv-h-power-n-add-x-epsilon-conv-h-power-x}.

\begin{figure}[htb]
  \centering
  \subfloat[add \(3.0\)]{
    \centering
    \includegraphics[width=0.42\linewidth]{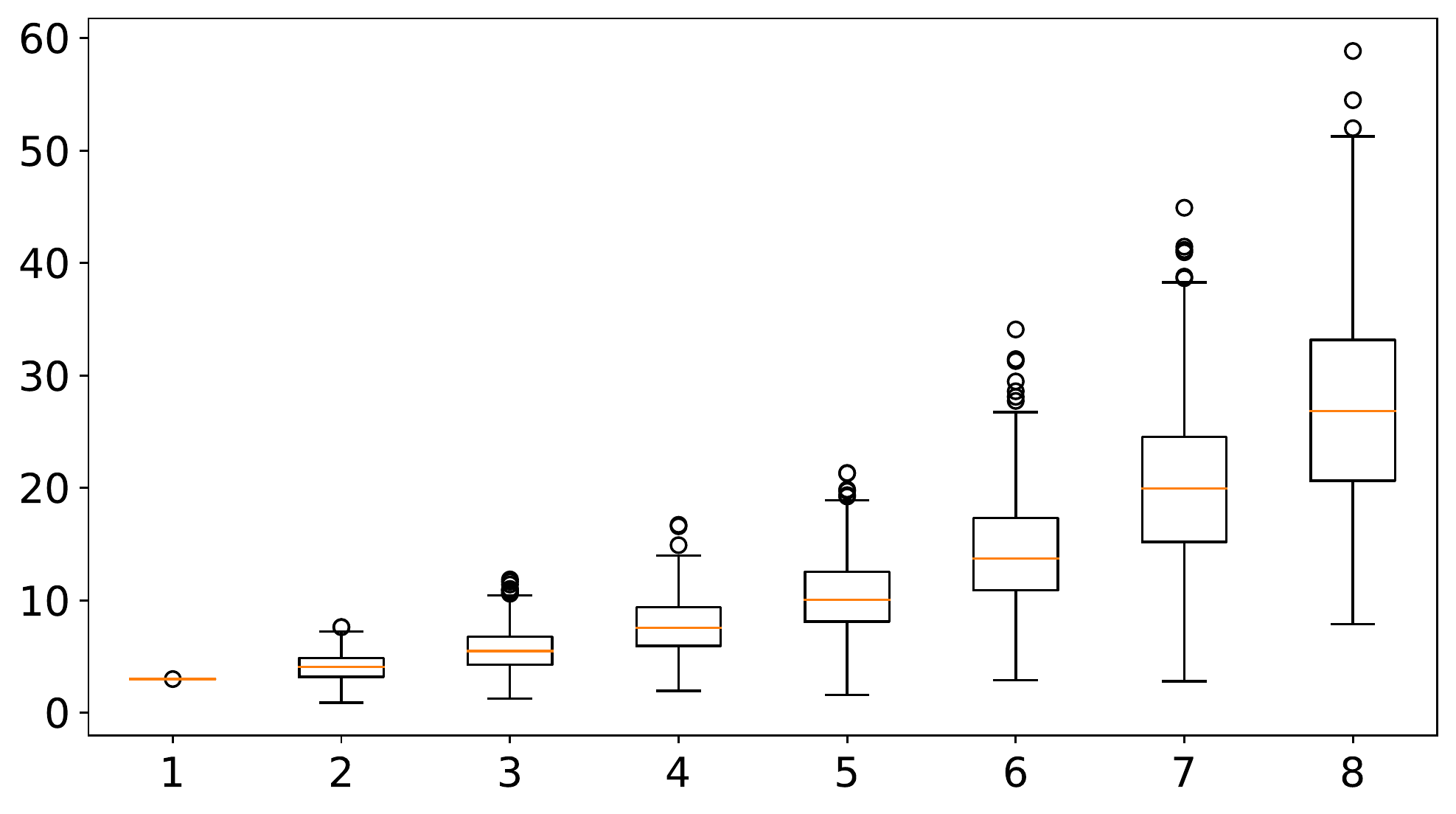}
    \label{fig:perturbation-boxplot-l2-norm-eps-3}
  }
  \subfloat[add \(0.5\)]{
    \centering
    \includegraphics[width=0.42\linewidth]{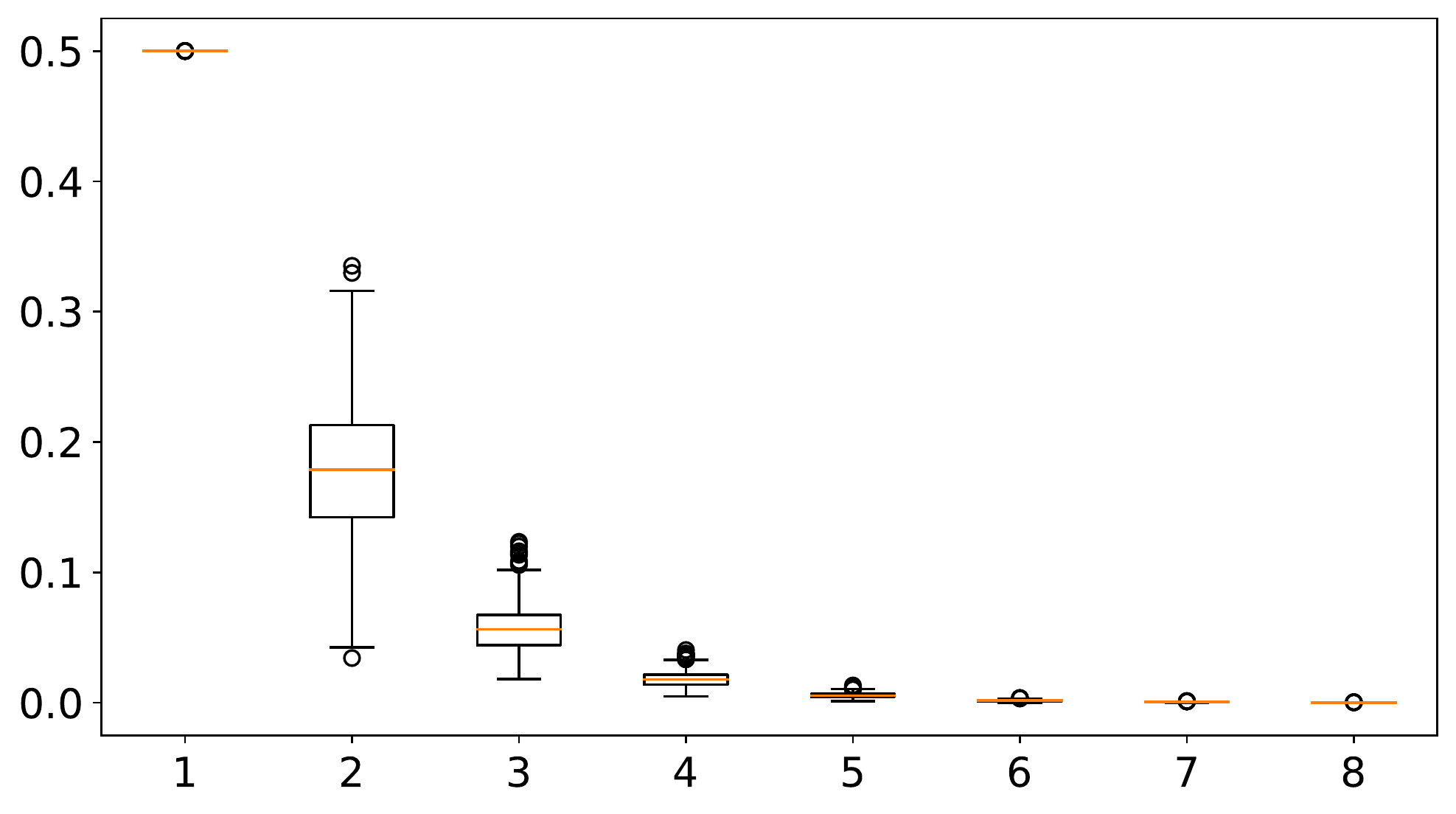}
    \label{fig:perturbation-boxplot-l2-norm-eps-0.5}
  }
  \caption{Boxplot of \(\|\tvar{H}_n * (\tvar{x} + \epsilon)^n - \tvar{H}_{n} * \tvar{x}^n\|_2, n = 1,2,\cdots,8\). The horizontal axis indicates the order and the vertical axis indicates the error.}
  \label{fig:statistic-l2norm-diff-conv-h-power-n-add-x-epsilon-conv-h-power-x}
\end{figure}

Figure \ref{fig:statistic-l2norm-diff-conv-h-power-n-add-x-epsilon-conv-h-power-x} shows that both upper bound and lower bound is reasonable. It seems that the impact of perturbation will cause exponential blowup or decay.
Nevertheless, this result is obtained under the context of ideal Volterra convolution.
In reality, the situation will be more complicated, we need to consider the kernels, input signals, and so on.

If the impact decays, this means that the related network are robust to such perturbation, which is exactly what we want.

As for the exponential blowup, we need to consider the network before the approximation.
If the network is Lipschitz in some domains, the approximated Volterra convolution must be Lipschitz in the same domains.
The upper bound of change of output is bounded by the Lipschitz constant, and the bound is not exponential.
Otherwise, if the network is non-Lipschitz, but the output is bounded, this impact is also bounded.
If the output of the network is not bounded, something wrong must have happened to this network. This is because all operations in a neural network is bounded, i.e., the convolutions are bounded if both kernel and signal are bounded, and the normalization operations are also bounded if not divide by zero.
In conclusion, the impact of perturbation will increase but won't exponential blowup in practice.

\subsection{Perturbation on Order-n Volterra Convolution}
\label{subsec:perturbation-on-order-n-volterra-convolution}

The following is an example of how perturbations affect the order-\(n\) Volterra convolution. Due to the computation complexity, \(n\) is set to be eight,
\begin{equation*}
  f(\tvar{x}) = \sum_{i=1}^{8} \tvar{H}_i * \tvar{x}^i, \tvar{H}_0 = 0.
\end{equation*}

Kernels are randomly generated as \(
\tvar{H}_1 \in \mathbb{M}^{5},
\tvar{H}_2 \in \mathbb{M}^{5 \times 5},
\cdots
\tvar{H}_8 \in \mathbb{M}^{5 \times \cdots \times 5}\) (\(\mathbb{M}\) is defined in Equation \ref{equ:define-the-set-of-gaussian-l2-norm-set-to-1}).
Input signal is selected as a sin function \(\tvar{x} = \sin(t), 0 \le t \le 8 \pi\).
Pick thirty points from \(\tvar{x}\) and plus \(0.2\) as perturbation \(\tvar{x} + \epsilon\). The perturbations are denser in the head and sparser in the tail.
Result is illustrated in Figure \ref{fig:example-for-perturbation-of-o8-volterra-convolution}.

\begin{figure}[htb]
  \centering
  \includegraphics[width=0.98\linewidth]{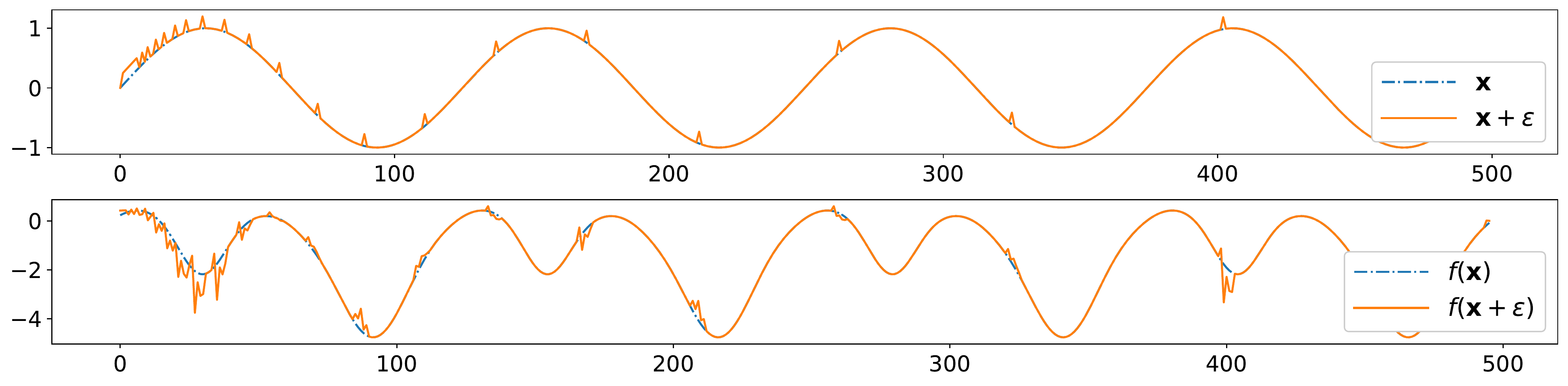}
  \caption{Example for perturbation of order-eight Volterra convolution.}
  \label{fig:example-for-perturbation-of-o8-volterra-convolution}
\end{figure}

This result is influenced by both network parameters and perturbations. Different parameters will cause different results. Some are sensitive to perturbations while others are not.

As can be seen in Figure \ref{fig:example-for-perturbation-of-o8-volterra-convolution}, points in the head are more influenced by perturbations, while points in the tail are not. This is because perturbations are denser in the head than in the tail, and short time energy is larger than that in the tail.
It can also be seen that some special points, such as peaks or valleys, are more affected by perturbation, while other points are less affected.
Not all perturbations are effective. The effective perturbations are both determined by both kernels and input signals.

\subsection{Proof for Theorem \ref{thm:vconv-perturbation-upper-bound}}
\label{subsec:proof-for-theorem-nn-is-robust-and-fragile}

In the following, we will prove Theorem \ref{thm:vconv-perturbation-upper-bound}.
For proving this Lemma, we will first introduce Young's Inequality (Theorem \ref{thm:young-inequality-convolutions}), Corollary \ref{corollary:yong-inequality-simplified-2-2-1}, Lemma \ref{lemma:inequality-l2-conv-g-x-y} and Lemma \ref{lemma:inequality-order-n-convolution-perturbation}.

\begin{theorem}[Young's Inequality for Convolutions \citep{Yong1912}]
  Let \(p,q,r \in \mathbb{R}\), \(p,q,r \ge 1\) and \(1 + \dfrac{1}{r} = \dfrac{1}{p} + \dfrac{1}{q}\).
  For signals \(\tvar{h}\) and \(\tvar{x}\), following inequality is satisfied:
  \begin{equation}
    \| \tvar{h} * \tvar{x} \|_{r} \le \| \tvar{h} \|_{p} \| \tvar{x} \|_{q}.
  \end{equation}
  \label{thm:young-inequality-convolutions}
\end{theorem}

\begin{corollary}
  Let \(r = p = 2, q = 1\). Theorem \ref{thm:young-inequality-convolutions} is simplified as
  \begin{equation}
    \| \tvar{h} * \tvar{x} \|_{2} \le \min\left(
    \| \tvar{h} \|_{2} \| \tvar{x} \|_{1},
    \| \tvar{h} \|_{1} \| \tvar{x} \|_{2}
    \right).
  \end{equation}
  \label{corollary:yong-inequality-simplified-2-2-1}
\end{corollary}

\begin{lemma}
  Suppose that \(\tvar{H}_n\) is the kernel, and \(\tvar{x}, \tvar{y}\) are two signals. We have
    \begin{equation*}
      \left\|\tvar{H}_n * \lcerfl{\tvar{x}^k, \tvar{y}^{n-k}} \right\|_2
      \le \min\left(\begin{aligned}
           & \| \tvar{H}_n \|_2 \| \tvar{x} \|_1^n \| \tvar{y} \|_1^{n-k},          \\
           & \| \tvar{H}_n \|_1 \| \tvar{x} \|_{2k}^k \| \tvar{y} \|_{2(n-k)}^{n-k} \\
        \end{aligned}\right),
      ~~~ k = 0, 1, \cdots, n.
    \end{equation*}
  
  \label{lemma:inequality-l2-conv-g-x-y}
\end{lemma}

\begin{proof}
  By definition of order-\(n\) convolution, we have
  \begin{equation}
    \left( \tvar{H}_n * \lcerfl{\tvar{x}^k, \tvar{y}^{n-k}} \right) (t)
    = \sum_{\tau_1, \cdots, \tau_n} H_n(\tau_1, \tau_2, \cdots, \tau_n)
    \prod_{i=1}^{k} x(t - \tau_i)
    \prod_{i=k+1}^{n} y(t - \tau_i)
    \label{equ:proof-expand-conv-gn-xk-yn-k}
  \end{equation}
  Let \(Z(t_1, t_2, \cdots, t_n) = \prod_{i=1}^{k} x(t_i) \prod_{i=k+1}^{n} y(t_i)\). Equation \ref{equ:proof-expand-conv-gn-xk-yn-k} becomes
  \begin{equation*}
    \begin{aligned}
      \left( \tvar{H}_n * \lcerfl{\tvar{x}^k, \tvar{y}^{n-k}} \right)(t)
       & = \sum_{\tau_1, \cdots, \tau_n} H_n(\tau_1, \tau_2, \cdots, \tau_n)
      Z(t - \tau_1, t - \tau_2, \cdots, t - \tau_n)
      \\
       & = \left(\tvar{H}_n * \tvar{Z}\right)(t, t, \cdots, t).
    \end{aligned}
  \end{equation*}

  Recall Young's Inequality (Theorem \ref{thm:young-inequality-convolutions}):
  \begin{equation*}
    \| \tvar{H}_n * \tvar{Z} \|_2 \le \min\left(
    \| \tvar{H}_n \|_2 \| \tvar{Z} \|_1,
    \| \tvar{H}_n \|_1 \| \tvar{Z} \|_2
    \right),
  \end{equation*}
  where \(\| \tvar{Z} \|_1\) is
  \begin{equation*}
    \begin{aligned}
      \| \tvar{Z} \|_1
       & = \sum_{\vsymb{t}} | Z(t_1, t_2, \cdots, t_n) | \\
       & = \sum_{\vsymb{t}} \left|
      \prod_{i=1}^{k} x(t_i)
      \prod_{i=k+1}^{n} y(t_i)
      \right|                                            \\
       & = \left(
      \prod_{i=1}^{k} \sum_{t_1 \cdots} \left| x(t_i) \right|
      \right)
      \left(
      \prod_{i=k+1}^{n} \sum_{t_{k+1} \cdots} \left| y(t_i) \right|
      \right)                                            \\
       & = \| \tvar{x} \|_1^k \| \tvar{y} \|_1^{n-k},    \\
    \end{aligned}
  \end{equation*}
  and \(\| \tvar{Z} \|_2\) is
  \begin{equation*}
    \begin{aligned}
      \| \tvar{Z} \|_2
       & = \sqrt{\sum_{t_1, t_2, \cdots, t_n} ( Z(t_1, t_2, \cdots, t_n) )^2 } \\
       & = \sqrt{ \sum_{t_1, t_2, \cdots, t_n} \left(
        \prod_{i=1}^{k} x(t_i)
        \prod_{i=k+1}^{n} y(t_i)
      \right)^2}                                                               \\
       & = \sqrt{\left( \sum_{t_1 \cdots} x^{2k}(t_i) \right)
      \left( \sum_{t_{k+1} \cdots} y^{2(n-k)}(t_i) \right)}                    \\
       & = \| \tvar{x} \|_{2k}^k \| \tvar{y} \|_{2(n-k)}^{n-k}.                \\
    \end{aligned}
  \end{equation*}
  This proof is finished.
\end{proof}

\begin{lemma}
  Suppose that \(\tvar{H}_n\) is the kernel, and \(\tvar{x}, \tvar{y}\) are two signals. We have
  \begin{equation}
    \left\| \tvar{H}_n * (\tvar{x} + \tvar{y})^n - \tvar{H}_n  * \tvar{x}^n \right\|_2
    \le \min\left(\begin{aligned}
         & \| \tvar{H}_n \|_2 \sum_{k=0}^{n-1} \left(\dfrac{e n}{k}\right)^k \|\tvar{x}\|_1^k \| \tvar{y} \|_1^{n-k},          \\
         & \| \tvar{H}_n \|_1 \sum_{k=0}^{n-1} \left(\dfrac{e n}{k}\right)^k \|\tvar{x}\|_{2k}^k \| \tvar{y} \|_{2(n-k)}^{n-k}
      \end{aligned}\right),
  \end{equation}
  where \(k= 0, 1, \cdots, n.\), and \(e = 2.718281828\cdots\), the base of the natural logarithm.
  \label{lemma:inequality-order-n-convolution-perturbation}
\end{lemma}

\begin{proof}
  Recall Property \ref{prop:conv-g-power-n-plus-a-b},
  \begin{equation}
    \tvar{H}_n * (\tvar{x} + \tvar{y})^n
    = \sum_{k=0}^{n} \binom{n}{k} \tvar{H}_n * \lcerfl{\tvar{x}^k, \tvar{y}^{n-k}}.
  \end{equation}

  We move \(\tvar{H}_n * \tvar{x}^n\) to the left and take the \(L_2\)-norm.
    By Lemma \ref{lemma:inequality-l2-conv-g-x-y}, the \(L_2\)-norm is
  \begin{equation*}
    \begin{aligned}
      \left\| \tvar{H}_n * (\tvar{x} + \tvar{y})^n - \tvar{H}_n  * \tvar{x}^n \right\|_2
       & = \left\|
      \sum_{k=0}^{n-1} \binom{n}{k} \tvar{H}_n * \lcerfl{\tvar{x}^k, \tvar{y}^{n-k}}
      \right\|_2                                          \\
       & \le \sum_{k=0}^{n-1} \binom{n}{k} \left\|
      \tvar{H}_n * \lcerfl{\tvar{x}^k, \tvar{y}^{n-k}}
      \right\|_2                                          \\
       & \le \min\left(\begin{aligned}
           & \| \tvar{H}_n \|_2 \sum_{k=0}^{n-1} \binom{n}{k} \|\tvar{x}\|_1^k \| \tvar{y} \|_1^{n-k},          \\
           & \| \tvar{H}_n \|_1 \sum_{k=0}^{n-1} \binom{n}{k} \|\tvar{x}\|_{2k}^k \| \tvar{y} \|_{2(n-k)}^{n-k}
        \end{aligned}\right)  \\
       & \le \min\left(\begin{aligned}
           & \| \tvar{H}_n \|_2 \sum_{k=0}^{n-1} \left(\dfrac{e n}{k}\right)^k \|\tvar{x}\|_1^k \| \tvar{y} \|_1^{n-k},          \\
           & \| \tvar{H}_n \|_1 \sum_{k=0}^{n-1} \left(\dfrac{e n}{k}\right)^k \|\tvar{x}\|_{2k}^k \| \tvar{y} \|_{2(n-k)}^{n-k}
        \end{aligned}\right), \\
    \end{aligned}
  \end{equation*}
  where binomial coefficient \(\binom{n}{k} \le \left(\dfrac{e n}{k}\right)^k\) (Exercise 0.0.5 of \citep{vershynin_2018}).
\end{proof}

The following is the proof of Theorem \ref{thm:vconv-perturbation-upper-bound}.

\begin{proof}[Proof of Theorem \ref{thm:vconv-perturbation-upper-bound}]
  \begin{equation*}
    \begin{aligned}
      \left\|f(\tvar{x} + \epsilon) - f(\tvar{x}) \right\|_2
       & = \left\|
      \sum_{n=0}^{N} \left(
      \tvar{H}_n * ( \tvar{x} + \epsilon )^n - \tvar{H}_n * \tvar{x}^n
      \right)
      \right\|_2                                                                                  \\
       & \le \sum_{n=0}^{N} \left\|
      \left(
      \tvar{H}_n * ( \tvar{x} + \epsilon )^n - \tvar{H}_n * \tvar{x}^n
      \right)
      \right\|_2
      ~~~~ \text{(apply Lemma \ref{lemma:inequality-order-n-convolution-perturbation})} \\
       & \le \min\left(\begin{aligned}
           & \sum_{n=0}^{N} \| \tvar{H}_n \|_2 \sum_{k=0}^{n-1} \left(\dfrac{e n}{k}\right)^k \|\tvar{x}\|_1^k \| \epsilon \|_1^{n-k},                 \\
           & \sum_{n=0}^{N} \| \tvar{H}_n \|_1 \sum_{k=0}^{n-1} \left(\dfrac{e n}{k}\right)^k \|\tvar{x}\|_{2k}^k \| \tvar{\epsilon} \|_{2(n-k)}^{n-k}
        \end{aligned}\right).
    \end{aligned}
  \end{equation*}
\end{proof}

\section{Rank for Outer Convolution}
\label{appendix:convolution-rank}

In this appendix we will discuss rank of output from outer convolution and \(n\)-dimensional convolution in discrete time.
The rank of a matrix equal to the number of non-zero eigenvalues.
It is of major importance, which tells us linear dependencies of column or row vectors, and is one of the fundamental building block of matrix completion and compress sensing \citep{Candes2006, Donoho2006, Sidiropoulos2017}.
The rank of a tensor could be defined in various ways, such as the rank based on CANDECOMP/PARAFACT decomposition and Tucker decomposition and others \citep{Kolda209TensorDecompositionAndApplication}.

Before discussing the rank, let's first take a look at linear dependence and independence \citep{Horn1985}.
Vectors \(\tvar{v}_1, \tvar{v}_2, \cdots, \tvar{v}_n\) are linear independence if \(\alpha_1 \tvar{v}_1 + \alpha_2 \tvar{v}_2 + \cdots + \alpha_n \tvar{v}_n = 0\) only implies all scalars are zero \(\alpha_1 = \alpha_2 = \cdots = 0\).
Similar to that, we call \(\tvar{G}_1, \tvar{G}_2, \cdots, \tvar{G}_n\) are linear independence if \(\alpha_1 \tvar{G}_1 + \alpha_2 \tvar{G}_2 + \cdots + \alpha_n \tvar{G}_n = 0\) only implies \(\alpha_1 = \alpha_2 = \cdots = \alpha_n = 0\).
If \(k\) vectors span a space of dimension \(r\), any of these \(k\) vectors can be represented as a linear combination of \(r\) linear independence vectors in that space.

To express the rank of matrix, Tucker rank \citep{Sidiropoulos2017} of tensor, and others, we denote
\begin{equation}
  \rank(k, \tvar{G}) =
  \left\{\begin{array}{ll}
    1,                                    & \text{if } \tvar{G} \text{ is 1 dimension}         \\
    \text{column rank of } \tvar{G},      & \text{if } \tvar{G} \text{ is 2 dimension, } k = 1 \\
    \text{row rank of } \tvar{G},         & \text{if } \tvar{G} \text{ is 2 dimension, } k = 2 \\
    k\text{-th Tucker rank of } \tvar{G}, & \text{others}                                      \\
  \end{array}\right. .
\end{equation}
Especially, if \(\tvar{G}\) is a matrix or one-dimensional vector, this notation can be rewritten as \(\rank(\tvar{G})\).
Besides, a tensor \(\tvar{G}\) is non-zero means that its arbitrary Lp-norm \(\| \tvar{G} \|_p \ne 0, p \ge 1\).

  Before talking about the rank, we would like to introduce the zero result of convolutions.
  By the definition of one-dimensional convolution, we can rewrite it into matrix multiplication format,
  \begin{equation*}
    y(t) = \sum_{\tau=0}^{T} g(\tau) h(t - \tau)
    \Leftrightarrow
    \begin{bmatrix}
      \vdots \\
      y(t)   \\
      \vdots \\
    \end{bmatrix}
    = \begin{bmatrix}
      \vdots & \vdots &        & \vdots \\
      h(t-0) & h(t-1) & \cdots & h(t-T) \\
      \vdots & \vdots &        & \vdots \\
    \end{bmatrix}
    \begin{bmatrix}
      g(0)   \\
      g(1)   \\
      \vdots \\
      g(T)   \\
    \end{bmatrix},
  \end{equation*}
  where matrix generated by shifting \(h(\cdot)\) is the Hankel matrix.
  It is clear that the convolution \(\tvar{g} * \tvar{h}\) is zero if \(\tvar{g}\) is in the null space \citep{Horn1985} of the Hankel matrix.
  For example, the zero padding one-dimensional convolution is zero if kernel and signal have the following format
  \begin{equation*}
    \begin{aligned}
      \tvar{g} & = \begin{bmatrix} \phantom{-}1 & \phantom{-}1 & \phantom{-}1 & \phantom{-}1 & \phantom{-}1 \end{bmatrix}; \\
      \tvar{h} & = \begin{bmatrix} \phantom{-}1 & -1 & \phantom{-}0 & -1 & \phantom{-}1 & \phantom{-}1 & -1 & \phantom{-}0 & -1 & \cdots \end{bmatrix}. \\
    \end{aligned}
  \end{equation*}
  The high-dimensional convolution also follows the same rule.
  We can rewrite this kind of convolutions into matrix multiplication format by flattening \(\tvar{G}\) and the patches of \(\tvar{H}\).
  A patch of \(\tvar{H}\) is \(H(\vsymb{t} - \lcerfl{\tau_1, \tau_2, \cdots, \tau_k})\), where \(\tau_i = 0, 1, \cdots, N_i\) with \(i = 1, 2, \cdots, k\) and \(\tvar{G} \in \mathbb{R}^{N_1 \times N_2 \times \cdots \times N_k}\).
  For instance, a two-dimensional convolution can be rewritten into the following format,
  \begin{equation*}
    \begin{bmatrix}
      \vdots              & \vdots              & \vdots              & \vdots              \\
      H(t_1 - 0, t_2 - 0) & H(t_1 - 0, t_2 - 1) & H(t_1 - 1, t_2 - 0) & H(t_1 - 1, t_2 - 1) \\
      \vdots              & \vdots              & \vdots              & \vdots              \\
    \end{bmatrix}
    \begin{bmatrix}
      G(0,0) \\
      G(0,1) \\
      G(1,0) \\
      G(1,1) \\
    \end{bmatrix}.
  \end{equation*}
  If the flatted \(\tvar{G}\) is in the null space of the matrix generated by flattening patches of \(\tvar{H}\), the result of this high-dimensional convolution is zero.

\subsection{Linear Independence in Convolutions}
\label{subsec:linear-independence-in-convolutions}

\begin{lemma}
  
    There are \(r \le \min\left( \rank(\mathcal{H}), \rank(\mathcal{G})\right)\) linear independence tensors in the linear combinations of \(\tvar{G}_i * \tvar{H}, i = 1, \cdots\), where \(\mathcal{H}\) is the matrix generated by flattening patches of \(\tvar{H}\), and \(\mathcal{G}\) is the matrix generated by flattening \(\tvar{G}_1, \tvar{G}_2, \cdots\). If \(\mathcal{H}\) is full column rank, \(r = \rank(\mathcal{G})\).
  
  \label{lemma:convolution-keep-linear-dependence}
\end{lemma}

\begin{proof}
  
    Suppose \(Y_i(\vsymb{t}) = \sum_{\vsymb{\tau}} H(\vsymb{t} - \vsymb{\tau}) G_i(\vsymb{\tau})\), we rewrite this convolution into matrix multiplication format,
    \begin{equation*}
      \begin{bmatrix}
        \vert          & \vert          &        \\
        Y_1(\vsymb{t}) & Y_2(\vsymb{t}) & \cdots \\
        \vert          & \vert          &        \\
      \end{bmatrix}
      =
      \begin{bmatrix}
             & \vdots                      &      \\
        -\!- & H(\vsymb{t} - \vsymb{\tau}) & -\!- \\
             & \vdots                      &      \\
      \end{bmatrix}
      \begin{bmatrix}
        \vert             & \vert             &        \\
        G_1(\vsymb{\tau}) & G_2(\vsymb{\tau}) & \cdots \\
        \vert             & \vert             &        \\
      \end{bmatrix}
      = \mathcal{H} \mathcal{G}.
    \end{equation*}

    Let matrix \(\mathcal{Y}\) be the matrix generated by flattening \(\tvar{Y}_1, \tvar{Y}_2, \cdots\). We have \citep{Horn1985}
    \begin{equation*}
      r
      = \rank(\mathcal{Y})
      = \rank(\mathcal{H} \mathcal{G})
      \le \min\left( \rank(\mathcal{H}), \rank(\mathcal{G})\right).
    \end{equation*}

  If \(\mathcal{H}\) is not full column rank, it is possible that some linear combinations of the columns of \(\mathcal{G}\) are in the null space of \(\mathcal{H}\). The matrix multiplication with \(\mathcal{H}\) and these combinations are zero, which implies that the convolution of \(\tvar{H}\) and the same combinations of \(\tvar{G}_i\) are zero.

    Otherwise, if \(\mathcal{H}\) is full column rank, left multiplication by a full column rank matrix leaves rank unchanged, \(r = \rank(\mathcal{H} \mathcal{G}) = \rank(\mathcal{G})\), and \(\rank(\mathcal{G})\) equals the number of linear independence tensors in the linear combinations of \(\tvar{G}_1, \tvar{G}_2, \cdots\). It means that there are no combinations of \(\tvar{G}_i\) that lead to a zero convolution result.
  
\end{proof}

  By Lemma \ref{lemma:convolution-keep-linear-dependence}, if the \(\mathcal{H}\) is full column rank, the number of linear independent tensors in linear combinations of \(\tvar{G}_1 * \tvar{H}, \tvar{G}_2 * \tvar{H}, \cdots\) equals to that of \(\tvar{G}_1, \tvar{G}_2, \cdots\).
  For instance, if \(\tvar{H}\) meets the condition, and \(\tvar{G}_1 \ne \alpha \tvar{G}_2\) for any scalar \(\alpha\), we have \(\tvar{G}_1 * \tvar{H} \ne \alpha \tvar{G}_2 * \tvar{H}\).

\subsection{Matrix Rank for Two-dimensional Signals}
\label{subsec:matrix-rank-for-two-dimensional-signals}

\begin{lemma}
  \(\rank(\tvar{G} \oconv \lcerfl{\tvar{h}_1, \tvar{h}_2}) \le \min\left(\rank(\tvar{G}), \rank(\mathcal{H}_1), \rank(\mathcal{H}_2)\right) \le \rank(\tvar{G})\), where \(\tvar{G} \in \mathbb{R}^{n_1 \times n_2}\), and \(\tvar{h}_1 \in \mathbb{R}^{l_1}, \tvar{h}_2 \in \mathbb{R}^{l_2}\).
    We get equality, \(\rank(\tvar{G} \oconv \lcerfl{\tvar{h}_1, \tvar{h}_2}) = \rank(\tvar{G})\), if both \(\mathcal{H}_1 \in \mathbb{R}^{(l_1 + 2p_1 -n_1) \times n_1}\) and \(\mathcal{H}_2 \in \mathbb{R}^{(l_2 + 2p_2 -n_2) \times n_2}\), the two matrices generated by shifting \(h_1(\cdot)\) and \(h_2(\cdot)\) with padding size \(p_1\) and \(p_2\), are full column rank.
  \label{lemma:rank-oconv-g-h-h-is-rank-g}
\end{lemma}

\begin{proof}
  Since the rank of \(\tvar{G}\) is \(\rank(\tvar{G})\), we can factorize \(\tvar{G}\) into the form of summation by outer product of linear independent vectors,
    \begin{equation*}
      G(\tau_1, \tau_2) = \sum_{r=1}^{\rank(\tvar{G})} p_r(\tau_1) q_r(\tau_2).
    \end{equation*}

  By the definition of outer convolution (Definition \ref{equ:def-outer-convolution}), we have
    \begin{equation*}
      \begin{aligned}
        (\tvar{G} \oconv \lcerfl{\tvar{h}_1, \tvar{h}_2})(t_1, t_2)
         & = \sum_{\tau_1, \tau_2} G(\tau_1, \tau_2) h_1(t_1 - \tau_1) h_2(t_2 - \tau_2)                                                                     \\
         & = \sum_{r=1}^{\rank(\tvar{G})} \left(\sum_{\tau_1} p_r(\tau_1) h_1(t_1 - \tau_1)\right) \left(\sum_{\tau_2} q_r(\tau_2) h_2(t_2 - \tau_2)\right). \\
         & = \sum_{r=1}^{\rank(\tvar{G})} (\tvar{p}_r * \tvar{h}_1)(t_1) (\tvar{q}_r * \tvar{h}_2)(t_2),
      \end{aligned}
    \end{equation*}
    which is a summation of the outer product of vectors \(\tvar{p}_r * \tvar{h}_1\) and \(\tvar{q}_r * \tvar{h}_2\).

  Recall Lemma \ref{lemma:convolution-keep-linear-dependence}, we have
  \begin{equation*}
    \begin{aligned}
      \rank(\tvar{G} \oconv \lcerfl{\tvar{h}_1, \tvar{h}_2})
       & \le \min\left(\min(\rank(\tvar{G}), \rank(\mathcal{H}_1)), \min(\rank(\tvar{G}), \rank(\mathcal{H}_2)) \right) \\
       & \le \min\left(\rank(\tvar{G}), \rank(\mathcal{H}_1), \rank(\mathcal{H}_2)\right)                               \\
       & \le \rank(\tvar{G}).
    \end{aligned}
  \end{equation*}
  If \(\mathcal{H}_1\) is full column rank, there is no such vector \(\tvar{v} \in \mathbb{R}^{n_1}\) such that \(\tvar{v} * \tvar{h}_1\) is zero, and \(\tvar{p}_1, \tvar{p}_2, \cdots, \tvar{p}_{\rank(\tvar{G})}\) are linear independence. If \(\mathcal{H}_2\) is full column rank, the same rule can also be applied to \(\mathcal{H}_2\). Under this assumption, we have
  \begin{equation*}
    \rank(\tvar{G} \oconv \lcerfl{\tvar{h}_1, \tvar{h}_2})
    = \rank(\tvar{G}).
  \end{equation*}
  
\end{proof}

Let \(H(t_1, t_2) = h_1(t_1) h_2(t_2)\). It becomes a two-dimensional convolution,
\begin{equation*}
  \begin{aligned}
    \left(\tvar{G} \oconv \lcerfl{\tvar{h}_1, \tvar{h}_2}\right)(t_1, t_2)
     & = \sum_{\tau_1, \tau_2} G(\tau_1, \tau_2) h_1(t_1 - \tau_1) h_2(t_2 - \tau_2) \\
     & = \sum_{\tau_1, \tau_2} G(\tau_1, \tau_2) H(t_1 - \tau_1, t_2 - \tau_2).      \\
  \end{aligned}
\end{equation*}

  More generally, we extend this to the \(\tvar{H}\) of rank grater than or equal to one.

\begin{lemma}
  \(\rank(\tvar{G} * \tvar{H}) \le \min\left(s_1, s_2, \rank(\tvar{G}) \rank(\tvar{H}) \right)\), where both \(\tvar{G}\) and \(\tvar{H}\) are two-dimensional matrices, and \(\tvar{G} * \tvar{H} \in \mathbb{R}^{s_1 \times s_2}\).
  \label{lemma:2d-conv-g-h-is-mul-rank-g-rank-h}
\end{lemma}

\begin{proof}
  
  Suppose that \(\tvarhat{H}_1, \tvarhat{H}_2, \cdots, \tvarhat{H}_{\rank(\tvar{H})}\) are linear independence matrices, and \(\tvar{H} = \sum_{r=1}^{\rank(\tvar{H})} \tvarhat{H}_r\). The convolution becomes
  \begin{equation*}
    \tvar{G} * \tvar{H}
    = \tvar{G} * \left(\sum_{r=1}^{\rank(\tvar{H})} \tvarhat{H}_r \right)
    = \sum_{r=1}^{\rank(\tvar{H})} \tvar{G} * \tvarhat{H}_r.
  \end{equation*}
  By Lemma \ref{lemma:rank-oconv-g-h-h-is-rank-g}, the rank of this convolution is
  \begin{equation*}
    \rank\left(\tvar{G} * \tvar{H}\right)
    = \rank\left(\sum_{r=1}^{\rank(\tvar{H})} \tvar{G} * \tvarhat{H}_r\right)
    \le \sum_{r=1}^{\rank(\tvar{H})} \rank(\tvar{G})
    \le \rank(\tvar{H}) \rank(\tvar{G}).
  \end{equation*}
  In addition, the rank of a matrix must be bounded by its size, implying that
  \begin{equation*}
    \rank(\tvar{G} * \tvar{H}) \le \min(s_1, s_2, \rank(\tvar{G}) \rank(\tvar{H})).
  \end{equation*}
  
\end{proof}

\begin{remark}
  
    The two-dimensional convolution often increase the rank of a two-dimensional image.
  
\end{remark}

\begin{remark}
  Dilated convolutions \citep{Yu2016} still follow Lemma \ref{lemma:2d-conv-g-h-is-mul-rank-g-rank-h}, since adding zero-filled rows or columns will not alter the rank.
  \label{remark:dilated-convolution-still-works}
\end{remark}

\subsection{Tucker Rank for Multi-dimensional Signals}
\label{subsec:tucker-rank-for-multi-dimensional-signals}

\begin{lemma}
  \(\rank(k, \tvar{G} \oconv \vsymb{h}) \le \rank(k, \tvar{G})\), where \(k = 1, 2, \cdots, n\), and \(\vsymb{h} = \lcerfl{\tvar{h}_1, \tvar{h}_2, \cdots, \tvar{h}_n}\) is a list of non-zero one-dimensional signals, and \(\tvar{G}\) is an \(n\)-dimensional tensor, and \(\rank(k, \tvar{G})\) is the \(k\)-th Tucker rank of \(\tvar{G}\).
  \label{lemma:rank-oconv-g-hhh-is-tucker-rank}
\end{lemma}

\begin{proof}
  
    If \(n=2\), this has been proved in Lemma \ref{lemma:rank-oconv-g-h-h-is-rank-g}.

    If \(n > 2\), suppose that \(\tvar{G} \in \mathbb{R}^{s_1 \times s_2 \times \cdots \times s_n}\).
    Focusing on the \(k\)-th dimension, we have
    \begin{equation*}
      \begin{aligned}
        (\tvar{G} \oconv \vsymb{h})(\vsymb{t})
         & = \sum_{\vsymb{\tau}} G(\vsymb{\tau}) \prod_{i=1}^{n} h_i(t_i - \tau_i) \\
         & = \sum_{\tau_k} \left(
        \sum_{\tau_1, \cdots, \tau_{k-1}}
        \sum_{\tau_{k+1}, \cdots, \tau_n}
        G(\vsymb{\tau}) \prod_{i=1, i \ne k}^{n} h_i(t_i - \tau_i)
        \right) h_k(t_k - \tau_k).
      \end{aligned}
    \end{equation*}
    Let
    \begin{equation*}
      P(t_1, \cdots, t_{k-1}, \tau_k, t_{k+1}, \cdots, t_n)
      = \sum_{\tau_1, \cdots, \tau_{k-1}}
      \sum_{\tau_{k+1}, \cdots, \tau_n}
      G(\vsymb{\tau}) \prod_{i=1, i \ne k}^{n} h_i(t_i - \tau_i).
    \end{equation*}
    The convolution becomes
    \begin{equation*}
      (\tvar{G} \oconv \vsymb{h})(\cdots, t_k, \cdots)
      = \sum_{\tau_k} P(\cdots, \tau_k, \cdots) h_k(t_k - \tau_k).
    \end{equation*}
    We permute and reshape \(\tvar{P}\) into matrix format \(\tvarhat{P} \in \mathbb{R}^{s_k \times \cdot}\), and rewrite the convolution into matrix multiplication format,
    \begin{equation*}
      \begin{bmatrix}
        \vdots       & \vdots     &        & \vdots         \\
        h_k(t_k - 0) & h_k(t_k-1) & \cdots & h_k(t_k - s_k) \\
        \vdots       & \vdots     &        & \vdots         \\
      \end{bmatrix}
      \begin{bmatrix}
        \hat{P}(0, 1)   & \hat{P}(0, 2)   & \cdots \\
        \hat{P}(1, 1)   & \hat{P}(1, 2)   & \cdots \\
        \vdots          & \vdots          &        \\
        \hat{P}(s_k, 1) & \hat{P}(s_k, 2) & \cdots \\
      \end{bmatrix}
      \equiv \mathcal{H}_k \tvarhat{P}.
    \end{equation*}
    By the rank inequality of matrix multiplication, we have
    \begin{equation*}
      \begin{aligned}
        \rank(k, \tvar{G} \oconv \vsymb{h})
         & \le \min(\rank(\mathcal{H}_k), \rank(\tvarhat{P}))                                                                   \\
         & \le \min(\rank(\mathcal{H}_k), \rank(k, \tvar{G})) ~~~~(\text{Lemma \ref{lemma:convolution-keep-linear-dependence}}) \\
         & \le \rank(k, \tvar{G}).
      \end{aligned}
    \end{equation*}
  
\end{proof}

  Let \(H(t_1, t_2, \cdots, t_n) = h_1(t_1) h_2(t_2) \cdots h_n(t_n)\). A tensor with Tucker rank equals \([1, 1, \cdots, 1]\). We have
  \(\tvar{G} \oconv \vsymb{h} = \tvar{G} * \tvar{H}\), and we extend this to a more general \(\tvar{H}\) in Lemma \ref{lemma:nd-conv-g-h-is-mul-tucker-rank-gh}.

\begin{lemma}
  \(\rank(k, \tvar{G} * \tvar{H}) \le \min(s_k, \rank(k,\tvar{G}) \rank(k,\tvar{H}))\), where \(\tvar{G}, \tvar{H}\) are both non-zero \(n\)-dimensional tensors, and \(\tvar{G} * \tvar{H} \in \mathbb{R}^{s_1 \times s_2 \times \cdots \times s_n}\).
  \label{lemma:nd-conv-g-h-is-mul-tucker-rank-gh}
\end{lemma}

\begin{proof}
  
    If \(n = 2\), it has been proved in Lemma \ref{lemma:2d-conv-g-h-is-mul-rank-g-rank-h}.

  If \(n > 2\), suppose that
  \begin{equation*}
    \tvar{H} = \sum_{r=1}^{\rank(k, \tvar{H})} \tvarhat{H}_r,
  \end{equation*}
  where \(\tvar{H}_1, \tvar{H}_2, \cdots, \tvar{H}_{\rank(k, \tvar{H})}\) are linear independent and all \(\rank(k, \tvarhat{H}_r) = 1\).
  The \(n\)-dimensional convolution becomes,
  \begin{equation*}
    \tvar{G} * \tvar{H}
    = \tvar{G} * \left(\sum_{r=1}^{\rank(k, \tvar{H})} \tvarhat{H}_r\right)
    = \sum_{r=1}^{\rank(k, \tvar{H})} \tvar{G} * \tvarhat{H}_r.
  \end{equation*}
  By Lemma \ref{lemma:nd-conv-g-h-is-mul-tucker-rank-gh}, the \(k\)-th Tucker rank of this convolution is
  \begin{equation*}
    \begin{aligned}
      \rank\left(k, \tvar{G} * \tvar{H}\right)
       & = \rank\left(k, \sum_{r=1}^{\rank(k, \tvar{H})} \tvar{G} * \tvarhat{H}_r\right) \\
       & \le \sum_{r=1}^{\rank(k, \tvar{H})} \rank\left(k, \tvar{G}\right)               \\
       & \le \rank(k, \tvar{H}) \rank(k, \tvar{G}).                                      \\
    \end{aligned}
  \end{equation*}
  In addition, the \(k\)-th Tucker rank is bounded by the size, which implies that
  \begin{equation*}
    \rank(k, \tvar{G} * \tvar{H}) \le \min(s_k, \rank(k,\tvar{G}) \rank(k,\tvar{H})).
  \end{equation*}
  
\end{proof}

\subsection{Main Theorem of Rank for Outer Convolution}
\label{subsec:main-theorem-of-rank-for-outer-convolution}

Now, all previous discussion of outer convolutions are combined, expanding Lemma \ref{lemma:rank-oconv-g-hhh-is-tucker-rank} to Theorem \ref{thm:rank-oconv-g-hHH-is-mul-tucker-rank}.
To simplify discussions, outer convolution is divided into groups,
\begin{equation*}
  (\tvar{G} \oconv \vsymb{H})\left(
  \underbrace{((1,1), (1,2), \cdots)}_{\text{group } 1, ~ (\tvar{H}_1)},
  \cdots,
  \underbrace{((i,1), (i,2), \cdots)}_{\text{group } k, ~ (\tvar{H}_k)},
  \right).
\end{equation*}

\begin{theorem}[Rank for outer convolution]
  
  \begin{equation*}
    \rank(\lcerfl{k,i}, \tvar{G} \oconv \vsymb{H})
    \le \left\{\begin{array}{ll}
      \min\left(s_{k,1}, \rank(k, \tvar{G})\right),       & \tvar{H}_k \text{ is one-dimensional} \\
      \min\left(s_{k,i}, z_k \rank(i, \tvar{H}_k)\right), & \text{otherwise}
    \end{array}\right.,
  \end{equation*}
  where \(\vsymb{H} = \lcerfl{\tvar{H}_1, \tvar{H}_2, \cdots, \tvar{H}_n}\) is a list of tensors, and \(\tvar{G} \in \mathbb{R}^{z_1 \times z_2 \times \cdots \times z_n}\), and \(\tvar{G} \oconv \vsymb{H} \in \mathbb{R}^{s_{1,1} \times s_{1,2} \times \cdots \times s_{2,1} \times s_{2,2} \times \cdots}\).
  
  \label{thm:rank-oconv-g-hHH-is-mul-tucker-rank}
\end{theorem}

\begin{proof}
  
    If \(\tvar{H}_k\) is one-dimensional, this proof is similar to that of Lemma \ref{lemma:rank-oconv-g-hhh-is-tucker-rank}. The \(k\)-th Tucker rank is
    \begin{equation*}
      \rank(\lcerfl{k,1}, \tvar{G} \oconv \vsymb{H}) \le \min\left(s_{k,1}, \rank(k, \tvar{G})\right).
    \end{equation*}

    If \(\tvar{H}_k\) is not one-dimensional, we focus on the operations in group \(k\). By the definition of outer convolution (Equation \ref{equ:def-continuous-outer-convolution-1d}), we have
    \begin{equation*}
      \begin{aligned}
        (\tvar{G} \oconv \vsymb{H}) (\vsymb{t}_1, \vsymb{t}_2, \cdots, \vsymb{t}_n)
         & = \sum_{\vsymb{\tau}} G(\vsymb{\tau}) \prod_{i=1}^{n} H_i(\vsymb{t}_i - \tau_i) \\
         & = \sum_{\tau_k} \left(
        \sum_{\tau_1, \cdots, \tau_{k-1}}
        \sum_{\tau_{k+1}, \cdots, \tau_n}
        G(\vsymb{\tau})
        \prod_{i=1, i \ne k}^{n} H_i(\vsymb{t}_i - \tau_i)
        \right) H_k(\vsymb{t}_k - \tau_k).                                                 \\
      \end{aligned}
    \end{equation*}
    Let
    \begin{equation*}
      P(\vsymb{t}_1, \cdots, \vsymb{t}_{k-1}, \tau_k, \vsymb{t}_{k+1}, \cdots, \vsymb{t}_n)
      = \sum_{\tau_1, \cdots, \tau_{k-1}}
      \sum_{\tau_{k+1}, \cdots, \tau_n}
      G(\vsymb{\tau})
      \prod_{i=1, i \ne k}^{n} H_i(\vsymb{t}_i - \tau_i).
    \end{equation*}
    The convolution becomes
    \begin{equation*}
      \begin{aligned}
        (\tvar{G} \oconv \vsymb{H})(\cdots, \vsymb{t}_k, \cdots)
         & = \sum_{\tau_k} P(\cdots, \tau_k, \cdots) H_k(\vsymb{t}_k - \tau_k) \\
         & = \sum_{\iota_1, \iota_2, \cdots}
        \hat{P}(\cdots; \iota_1, \iota_2, \cdots; \cdots)
        H_k(t_{k,1} - \iota_1, t_{k,2} - \iota_2, \cdots),                     \\
      \end{aligned}
    \end{equation*}
    where \(\tvarhat{P}\) is a super-diagonal format of \(\tvar{G}\),
    \begin{equation*}
      \hat{P}(\cdots; \iota_1, \iota_2, \cdots; \cdots)
      = \left\{\begin{array}{ll}
        P(\cdots, \tau_k, \cdots), & \tau_k = \iota_1 = \iota_2 = \cdots \\
        0,                         & \text{otherwise}
      \end{array}\right. .
    \end{equation*}
    This operation can be considered as a collection of multidimensional convolutions in group \(k\), where all kernels are super-diagonal tensors.
    The Tucker rank of these super-diagonal tensors are \(\rank\left(\lcerfl{k,i}, \tvarhat{P}\right) \le z_k\), with equality if and only if all diagonal entries are non-zero.
    Recall Lemma \ref{lemma:nd-conv-g-h-is-mul-tucker-rank-gh}, we conclude that Tucker rank of this group is
    \begin{equation*}
      \begin{aligned}
        \rank\left(\lcerfl{k,i}, \tvar{G} \oconv \vsymb{H} \right)
         & \le \min(s_{k,i}, \rank\left(\lcerfl{k,i}, \tvarhat{P}\right) \rank(i, \tvar{H}_k)) \\
         & \le \min(s_{k,i}, z_k \rank(i, \tvar{H}_k)).
      \end{aligned}
    \end{equation*}
  
  In conclusion, we completed this proof.
\end{proof}

\subsection{Validation of the Rank Properties}
\label{subsec:validate-rank-after-outer-convolution}

\textbf{Data description}.
  To cover more situations, the following tensors are taken from three sets, \(\mathbb{M}\), \(\mathbb{U}\) and \(\mathbb{O}\). 

The first set is the Gaussian distribution with normalization,
  \begin{equation*}
    \mathbb{M} = \left\{
    \tvar{x} : \tvar{x} = \dfrac{\tvar{y}}{\|\tvar{y}\|_{2}}, y(\cdot) \sim \mathcal{N}(0, 1)
    \right\},
  \end{equation*}
  where \(\mathcal{N}(0,1)\) is standard Gaussian distribution with mean zero and variance one. This is the same as Equation \ref{equ:define-the-set-of-gaussian-l2-norm-set-to-1}.

The second set is the Uniform distribution with normalization,
  \begin{equation*}
    \mathbb{U} = \left\{
    \tvar{x} : \tvar{x} = \dfrac{\tvar{y}}{\|\tvar{y}\|_{2}}, y(\cdot) \sim \mathcal{U}(0, 1)
    \right\},
  \end{equation*}
  where \(\mathcal{U}(0, 1)\) is the Uniform distribution form zero to one.

The third set is
  \begin{equation*}
    \mathbb{O} = \left\{
    \tvar{h} : \exists \tvar{g}, ~ \tvar{g} * \tvar{h} \text{ is zero}, \tvar{g} \text{ and } \tvar{h} \text{ are non-zero}
    \right\}.
  \end{equation*}
  One method for generating the desired signal \(\tvar{h}\) with a known kernel \(\tvar{g}\) is illustrated as follows.
  For any fixed index \(t\), if \((\tvar{g} * \tvar{h})(t) = 0\), we have
  \begin{equation*}
    \sum_{\tau=0}^{T} g(\tau) h(t - \tau)
    = g(0) h(t) + \sum_{\tau=1}^{T} g(\tau) h(t - \tau)
    = 0.
  \end{equation*}
  If the weighted sum \(\sum_{\tau=1}^{T} g(\tau) h(t - \tau)\) is known, and \(g(0)\) is non-zero, \(h(t)\) can be uniquely determined by
  \begin{equation}
    h(t) = \dfrac{- 1}{g(0)} \sum_{\tau=1}^{T} g(\tau) h(t - \tau).
    \label{equ:zero-conv-1d-sequence-computation}
  \end{equation}
  We can iteratively compute the upcoming sequence \(h(T), h(T+1), \cdots\), if all initial terms, \(h(0), h(1), \cdots, h(T-1)\), are manually specified.
  It is clear that we cannot pad any other elements to the head or tail of \(\tvar{h}\), because the new elements may not follow Equation \ref{equ:zero-conv-1d-sequence-computation}, and padding these elements may leave \(\tvar{g} * \tvar{h}\) non-zero.

Let the matrix generated by shifting \(\tvar{h}\) be
  \begin{equation*}
    \mathcal{H}
    = \left[\begin{array}{c|c}
        \begin{matrix}
          h(0)   & h(1)   & \cdots & h(T-1) \\
          h(1)   & h(2)   & \cdots & h(T+0) \\
          h(2)   & h(3)   & \cdots & h(T+1) \\
          \vdots & \vdots &        & \vdots \\
        \end{matrix}
         &
        \begin{matrix}
          h(T+0) & \cdots \\
          h(T+1) & \cdots \\
          h(T+2) & \cdots \\
          \vdots &        \\
        \end{matrix}
      \end{array}\right].
  \end{equation*}
  According to the computing process, the columns in the right of vertical line are linear combinations of the columns in the left. It implies that \(\rank(\mathcal{H}) \le T\), with equality if the columns in the left of vertical line are linear independence.

  \textbf{Validation process.} The main object is to compare the numerical rank with the theoretical rank. To begin with, we compute the convolutions of the generated kernels and signals. After that, we clip the singular values of the convolution output to range  \([1e^{-16}, 1e^{16}]\). At last, we plot the singular values to a figure with y-axis scaled by \(\log_{10}\).
  We can easily distinguish the zero singular values from the non-zero ones in the figures. If there is a sharp slope, the singular values in the left of this slop are non-zero, and the number of the non-zero singular values equals the rank.

\textbf{Validate Lemma \ref{lemma:rank-oconv-g-h-h-is-rank-g}}: \(\rank(\tvar{G} \oconv \lcerfl{\tvar{h}_1, \tvar{h}_2}) \le \min\left(\rank(\tvar{G}), \rank(\mathcal{H}_1), \rank(\mathcal{H}_2)\right)\).

\(\tvar{G}\) is randomly generated with \(\rank(\tvar{G}) = r_g\). \(\tvar{h}_1\) is iteratively computed by \(T_1\) random initial terms and kernel \([1, 1, \cdots, 1]\). \(\tvar{h}_2\) is computed in the same way.
  Since zero padding of \(\tvar{h}_1\) or \(\tvar{h}_2\) is not allowed here, the outer convolution \(\tvar{G} \oconv \lcerfl{\tvar{h}_{1}, \tvar{h}_{2}}\) is zero padded.
  The singular values are plotted in Figure \ref{fig:oconv-rank-zero-convolution}.
  As shown in figure, \(\rank(\tvar{G} \oconv \lcerfl{\tvar{h}_{1}, \tvar{h}_{2}})\) is less than the smallest number of \(r_g, T_1, T_2\), and \(r_g = \rank(\tvar{G}), \rank(\mathcal{H}_1) \le T_1, \rank(\mathcal{H}_2) \le T_2\).

\begin{figure}[htb!]
  \centering
  \subfloat[\(\tvar{G} \in \mathbb{M}^{9 \times 9}\), \(\tvar{h}_1, \tvar{h}_2 \in \mathbb{O}^{27}\)]{
    \includegraphics[width=0.48\linewidth]{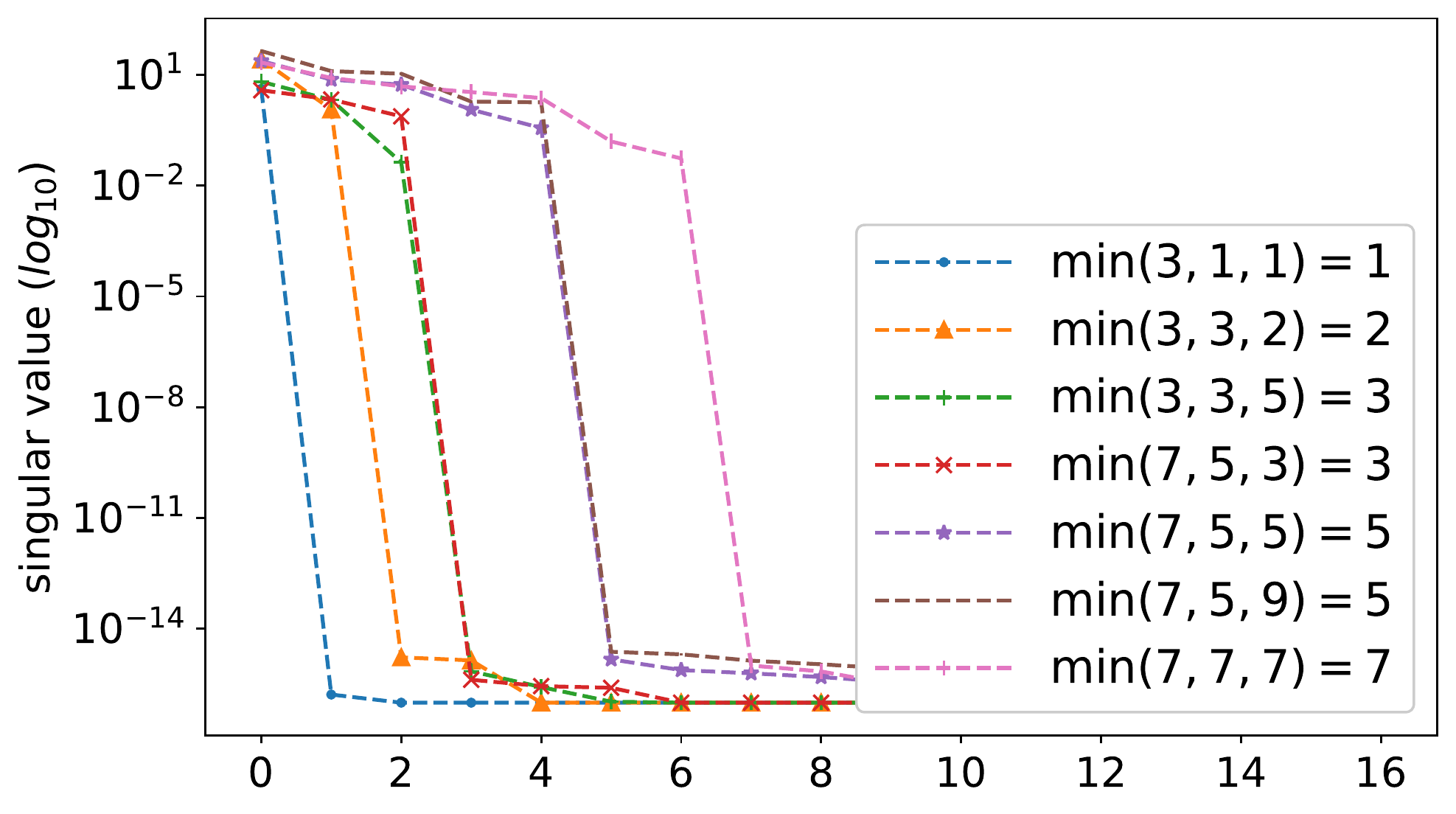}
    \label{fig:oconv-rank-zero-convolution-g}
  }
  \subfloat[\(\tvar{G} \in \mathbb{U}^{9 \times 9}\), \(\tvar{h}_1, \tvar{h}_2 \in \mathbb{O}^{27}\)]{
    \includegraphics[width=0.48\linewidth]{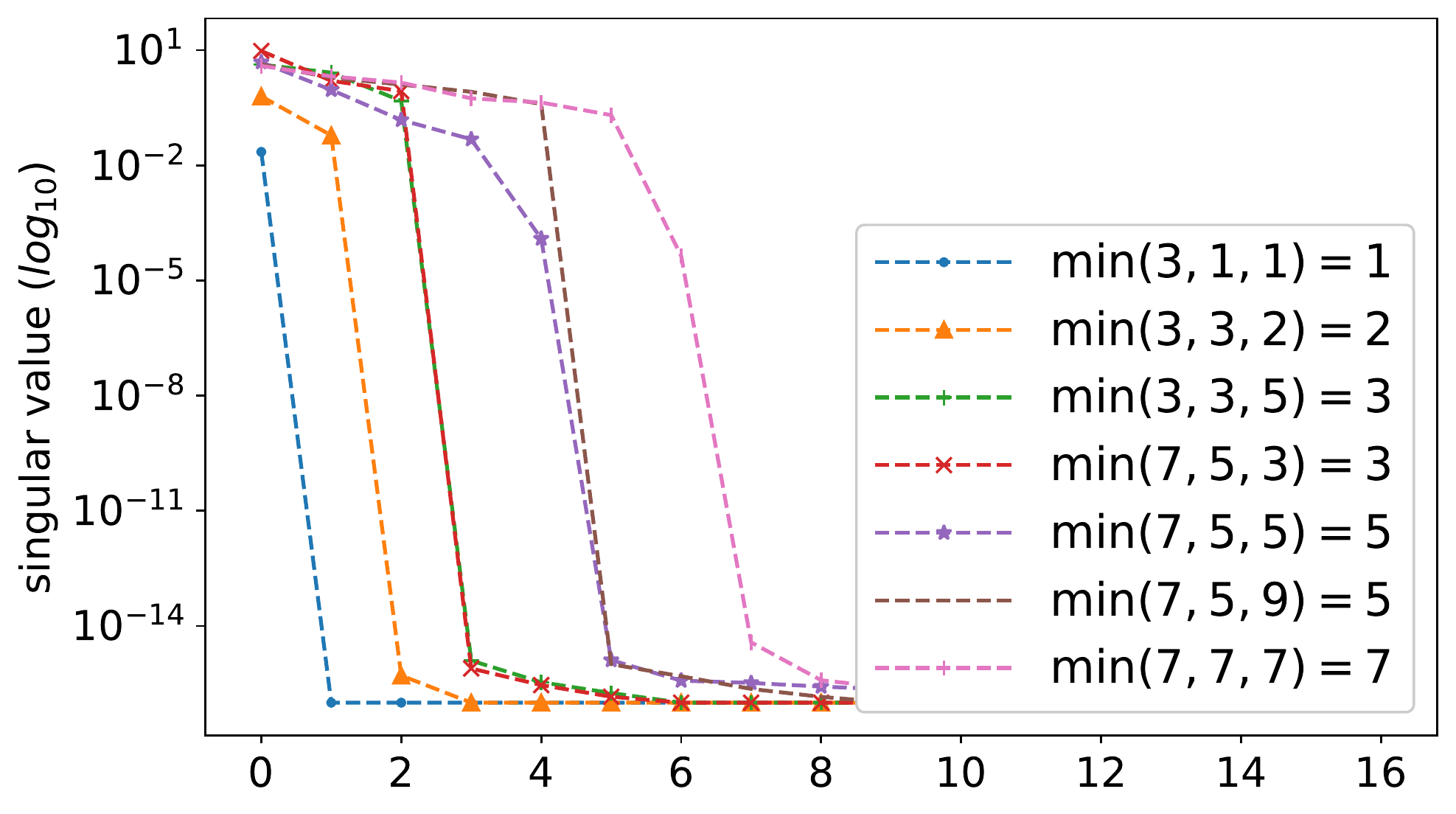}%
    \label{fig:oconv-rank-zero-convolution-u}
  }
  \caption{The singular values of \(\tvar{G} \oconv \lcerfl{\tvar{h}_{1}, \tvar{h}_{2}}\) with no padding. The min operator with three numbers is \(\min(r_g, T_1, T_2)\).}
  \label{fig:oconv-rank-zero-convolution}
\end{figure}

  \textbf{Validate Lemma \ref{lemma:2d-conv-g-h-is-mul-rank-g-rank-h}.} Regardless of size, the rank of a two-dimensional convolution \(\tvar{G} * \tvar{H}\) is \(\rank(\tvar{G} * \tvar{H}) \le \rank(\tvar{G}) \rank(\tvar{H})\).

Pairs of matrices \(\tvar{G}\) and \(\tvar{H}\) are randomly generated with \(\rank(\tvar{G}) = r_g\) and \(\rank(\tvar{H}) = r_h\).
  The singular values of \(\tvar{G} * \tvar{H}\) are plotted in Figure \ref{fig:oconv-rank-conv-kernel-image}.
  As the figure shows, \(\rank(\tvar{G} * \tvar{H})\) is less than or equal to the multiplication of \(\rank(\tvar{G})\) and \(\rank(\tvar{H})\).

\begin{figure}[htb!]
  \centering
  \subfloat[\(\tvar{G} \in \mathbb{M}^{7 \times 7}\), \(\tvar{H} \in \mathbb{M}^{32 \times 32}\)]{
    \includegraphics[width=0.48\linewidth]{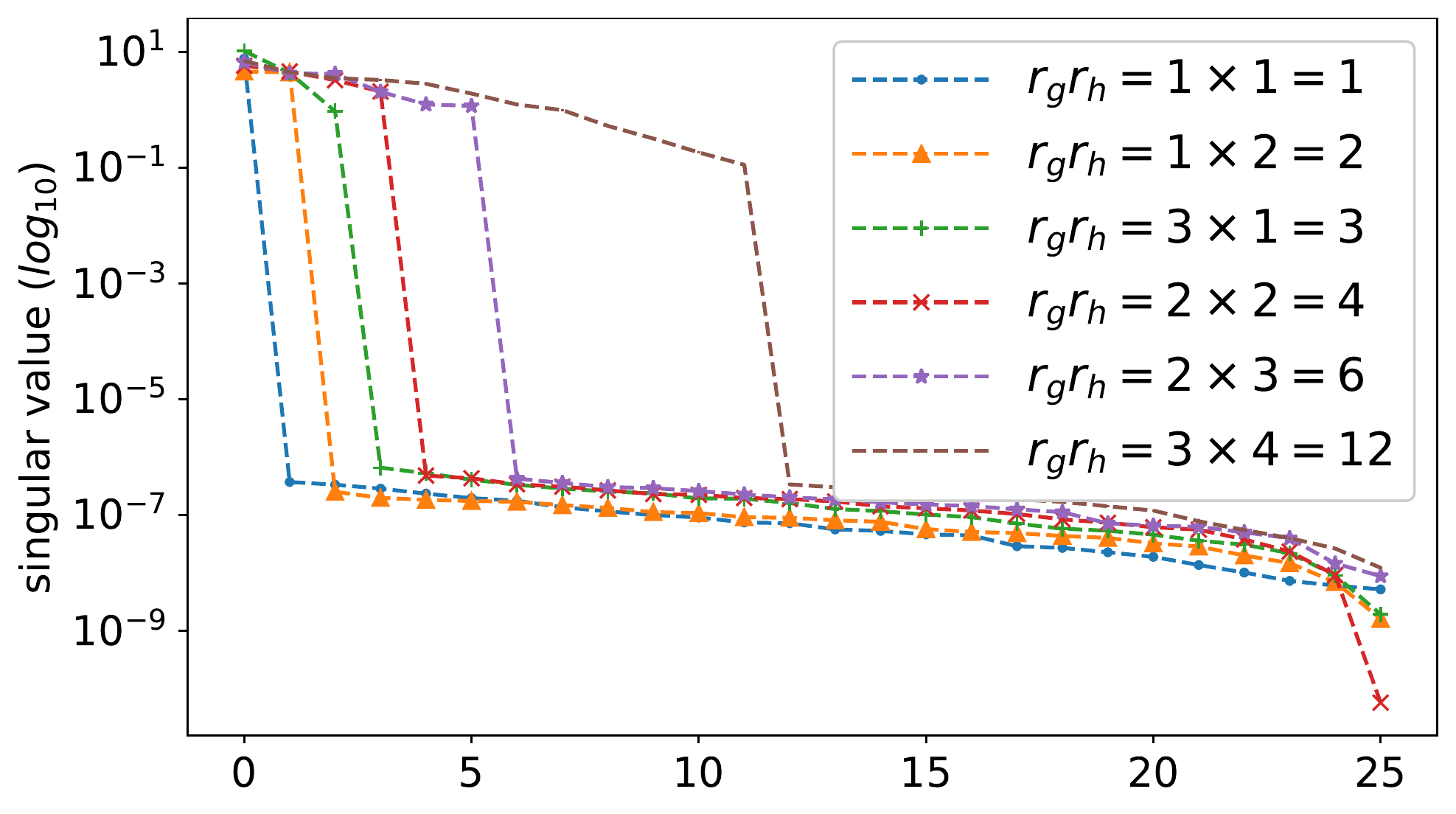}%
    \label{fig:oconv-rank-conv-kernel-image-g}
  }
  \subfloat[\(\tvar{G} \in \mathbb{U}^{7 \times 7}\), \(\tvar{H} \in \mathbb{U}^{32 \times 32}\)]{
    \includegraphics[width=0.48\linewidth]{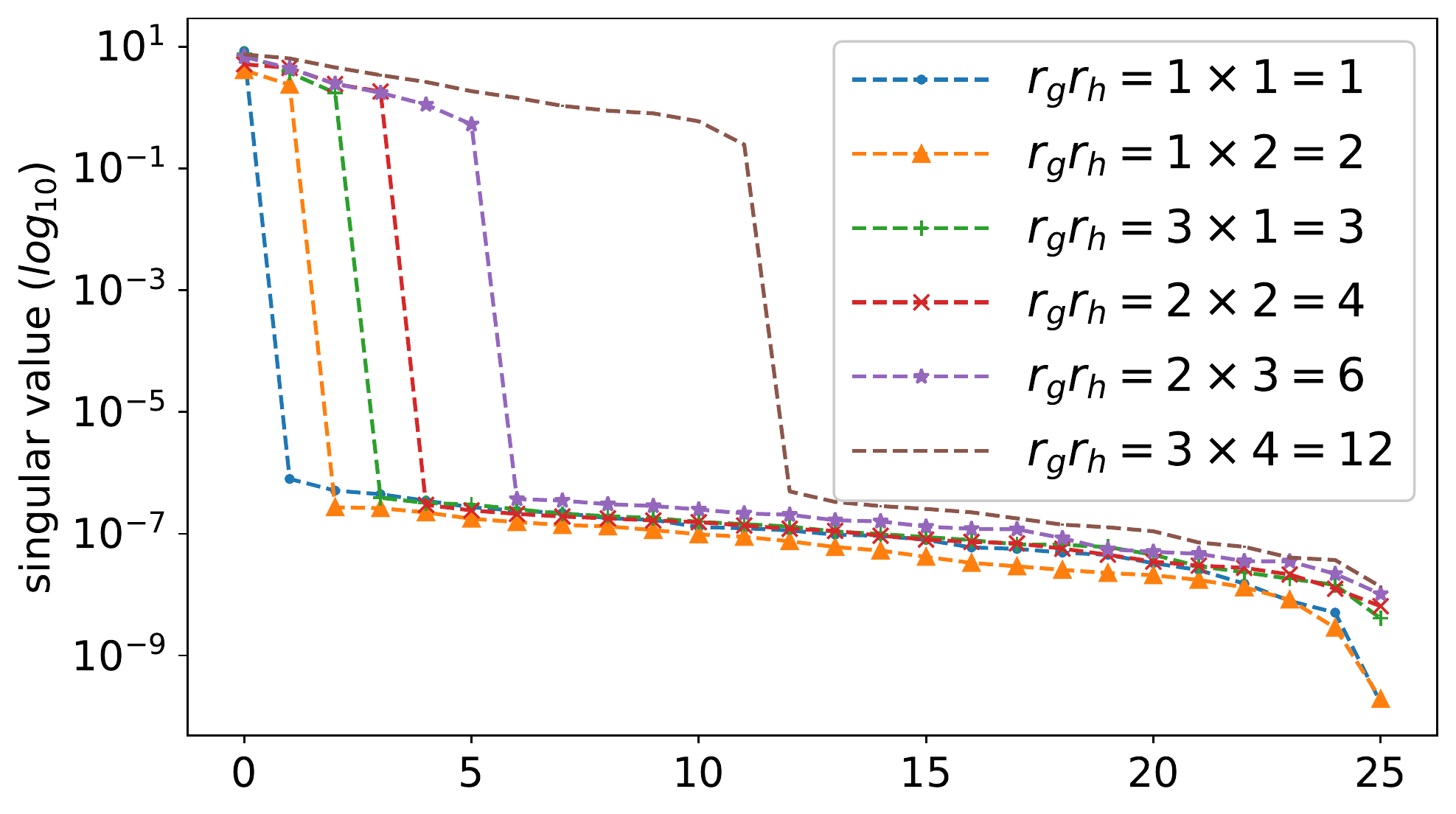}%
    \label{fig:oconv-rank-conv-kernel-image-u}
  }
  \caption{
    Singular values of two-dimensional convolutions \(\tvar{G} * \tvar{H}\).
  }
  \label{fig:oconv-rank-conv-kernel-image}
\end{figure}

  \textbf{Validate Lemma \ref{lemma:nd-conv-g-h-is-mul-tucker-rank-gh}.} Regardless of size, the Tucker rank of an \(n\)-dimensional convolution is \(\rank(k, \tvar{G} * \tvar{H}) \le \rank(k,\tvar{G}) \rank(k,\tvar{H})\).
  This inequality of three-dimensional convolution is validated below.

\(\tvar{G}\) is randomly generated with Tucker rank equals \([2,4,3]\) and \(\tvar{H}\) is randomly generated with Tucker rank equals \([3, 2, 4]\).
  The \(k\)-th Tucker rank of a tensor equals to the matrix rank of mode-\(k\) matricization of that tensor, where the mode-\(k\) matricization is to permute and reshape the tensor with shape \((\cdots, s, \cdots)\) to \((\cdots, s)\) \citep{Kolda209TensorDecompositionAndApplication}.
  The singular values of mode-(one, two, three) matricization of \(\tvar{G} * \tvar{H}\) are plotted in Figure \ref{fig:oconv-rank-3D-conv}.
  The \(k\)-th Tucker rank of \(\tvar{G} * \tvar{H}\) is less than or equal to the multiplication of the \(k\)-th Tucker rank of \(\tvar{G}\) and \(\tvar{H}\).

\begin{figure}[htb]
  \centering
  \subfloat[\(\tvar{G} \in \mathbb{M}^{6 \times 6 \times 6}\), \(\tvar{H} \in \mathbb{M}^{28 \times 28 \times 28}\)]{
    \includegraphics[width=0.48\linewidth]{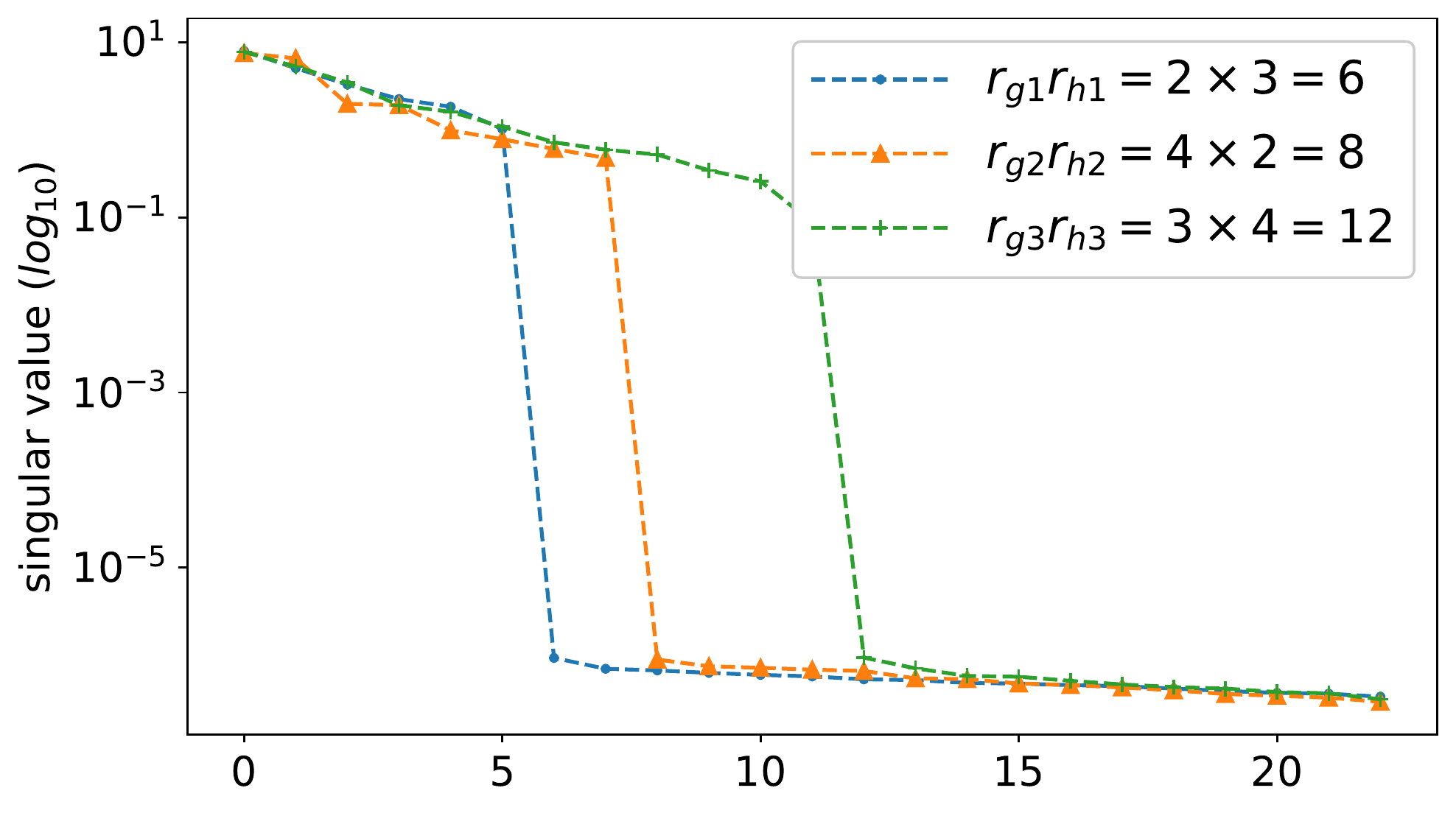}%
    \label{fig:oconv-rank-3D-conv-g}
  }
  \subfloat[\(\tvar{G} \in \mathbb{U}^{6 \times 6 \times 6}\), \(\tvar{H} \in \mathbb{U}^{28 \times 28 \times 28}\)]{
    \includegraphics[width=0.48\linewidth]{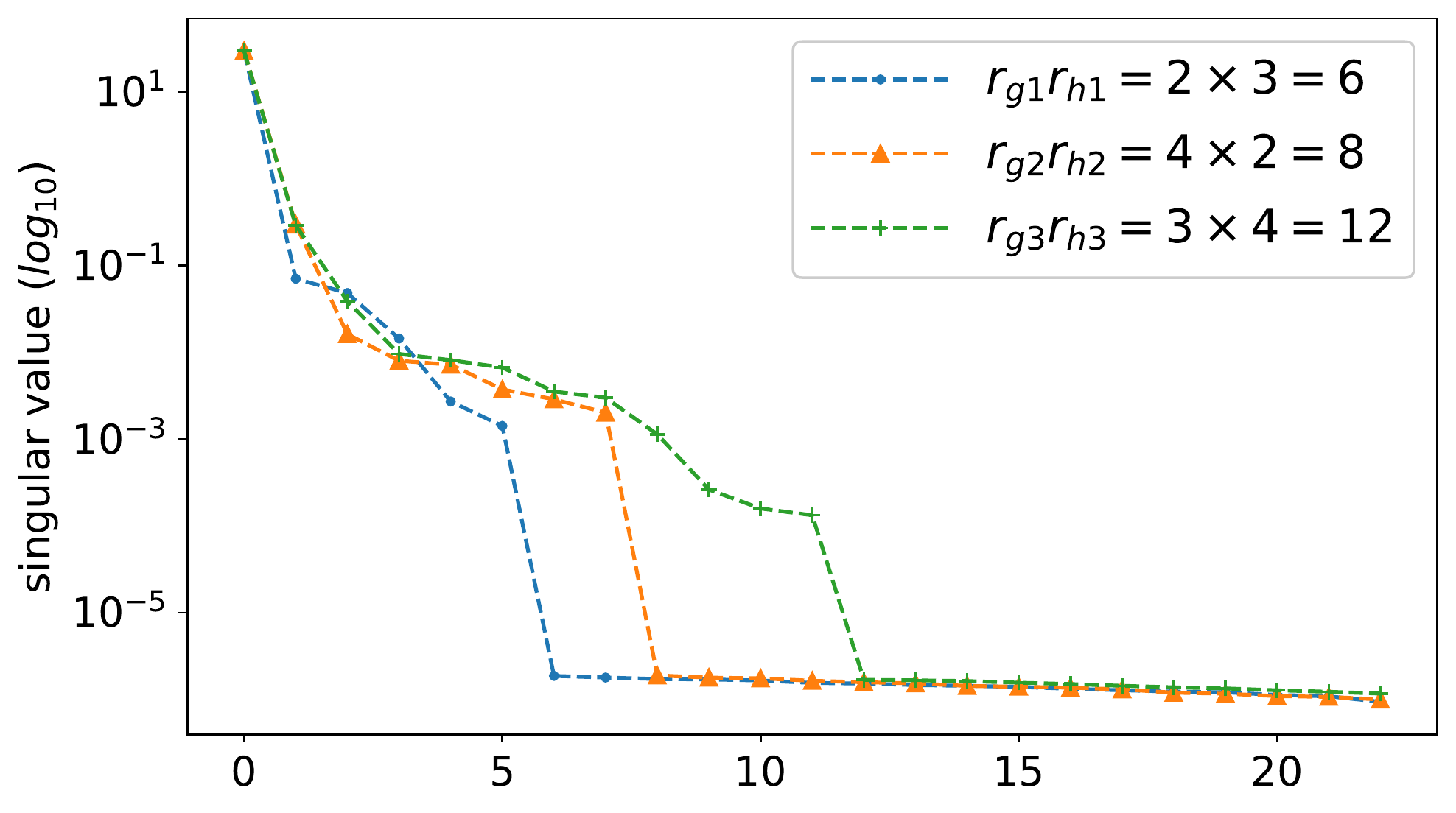}%
    \label{fig:oconv-rank-3D-conv-u}
  }
  \caption{
    Singular values of mode-(one, two, three) matricization of three-dimensional convolution \(\tvar{G} * \tvar{H}\).
  }
  \label{fig:oconv-rank-3D-conv}
\end{figure}

  \textbf{Validate Theorem \ref{thm:rank-oconv-g-hHH-is-mul-tucker-rank} and Lemma \ref{lemma:rank-oconv-g-hhh-is-tucker-rank}.} Since Lemma \ref{lemma:rank-oconv-g-hhh-is-tucker-rank} is a special case of Theorem \ref{thm:rank-oconv-g-hHH-is-mul-tucker-rank} with all signals are one-dimensional arrays, Theorem \ref{thm:rank-oconv-g-hHH-is-mul-tucker-rank} is mainly validate below.
  Regardless of size, The Tucker rank of outer convolution is
  \begin{equation*}
    \rank(\lcerfl{k,i}, \tvar{G} \oconv \vsymb{H})
    \le \left\{\begin{array}{ll}
      \rank(k, \tvar{G}),       & \tvar{H}_k \text{ is one-dimensional} \\
      z_k \rank(i, \tvar{H}_k), & \text{otherwise}
    \end{array}\right.,
  \end{equation*}
  where \(\vsymb{H} = \lcerfl{\tvar{H}_1, \tvar{H}_2, \cdots, \tvar{H}_n}\) is a list of tensors, and \(\tvar{G} \in \mathbb{R}^{z_1 \times z_2 \times \cdots \times z_n}\).
  This inequality of outer convolution \(\tvar{G} \oconv \lcerfl{\tvar{h}_1, \tvar{h}_2, \tvar{H}_3, \tvar{H}_4}\) is validated below.

Four-dimensional tensor \(\tvar{G}\) is randomly generated with Tucker rank equals \([2, 3, 3, 2]\). \(\tvar{h}_1\) and \(\tvar{h}_2\) are one-dimensional vectors, \(\tvar{H}_3\) is a two-dimensional matrix with \(\rank(\tvar{H}_3) = 2\), and \(\tvar{H}_4\) is a three-dimensional tensor with Tucker rank equals \([2, 3, 4]\).
  The singular values of mode-\(k\) matricization of \(\tvar{G} \oconv \lcerfl{\tvar{h}_1, \tvar{h}_2, \tvar{H}_3, \tvar{H}_4}\) is plotted in Figure \ref{fig:oconv-outer-convolution-tucker-rank}.
  \(\tvar{h}_1\) and \(\tvar{h}_2\) are one-dimensional, the first and second Tucker rank of \(\tvar{G} \oconv \lcerfl{\tvar{h}_1, \tvar{h}_2, \tvar{H}_3, \tvar{H}_4}\) is less than or equal to that of \(\tvar{G}\). \(\tvar{H}_3\) and \(\tvar{H}_4\) are not one-dimensional, the corresponding Tucker rank of \(\tvar{G} \oconv \lcerfl{\tvar{h}_1, \tvar{h}_2, \tvar{H}_3, \tvar{H}_4}\) is not grater than \(z_k\) times the Tucker rank of \(\tvar{H}_3\) and \(\tvar{H}_4\).

\begin{figure}[H]
  \centering
  \subfloat[\(\tvar{G} \in \mathbb{M}^{3 \times 3 \times 3 \times 3}\), \(\tvar{h}_1, \tvar{h}_2 \in \mathbb{M}^{5}\), \(\tvar{H}_3 \in \mathbb{M}^{7 \times 7}\), \\ \hspace*{1em} \(\tvar{H}_4 \in \mathbb{M}^{9 \times 9 \times 18}\)]{
    \includegraphics[width=0.48\linewidth]{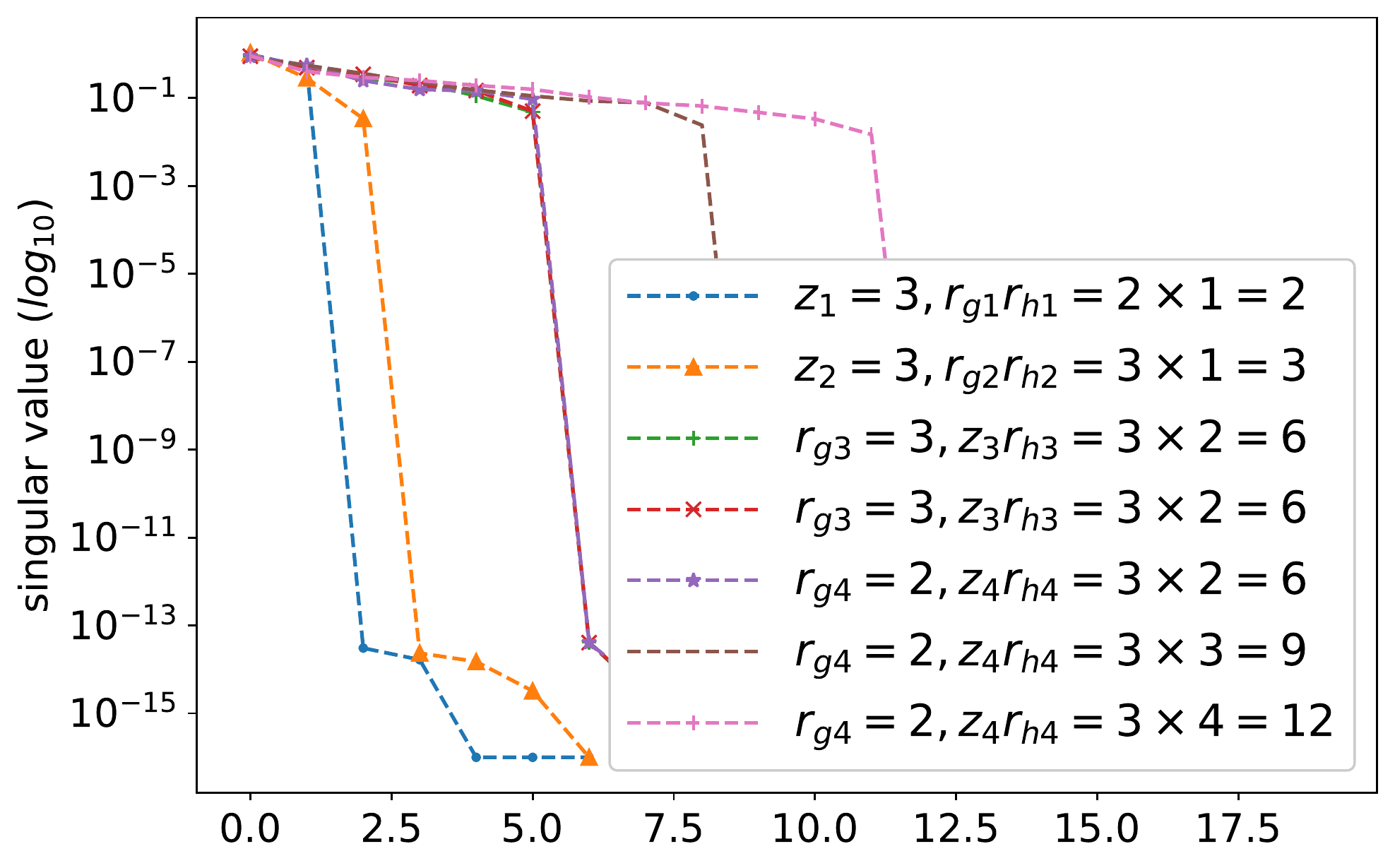}%
    \label{fig:oconv-outer-convolution-tucker-rank-g}
  }
  \subfloat[\(\tvar{G} \in \mathbb{U}^{3 \times 3 \times 3 \times 3}\), \(\tvar{h}_1, \tvar{h}_2 \in \mathbb{U}^{5}\), \(\tvar{H}_3 \in \mathbb{U}^{7 \times 7}\), \\ \hspace*{1em} \(\tvar{H}_4 \in \mathbb{U}^{9 \times 9 \times 18}\)]{
    \includegraphics[width=0.48\linewidth]{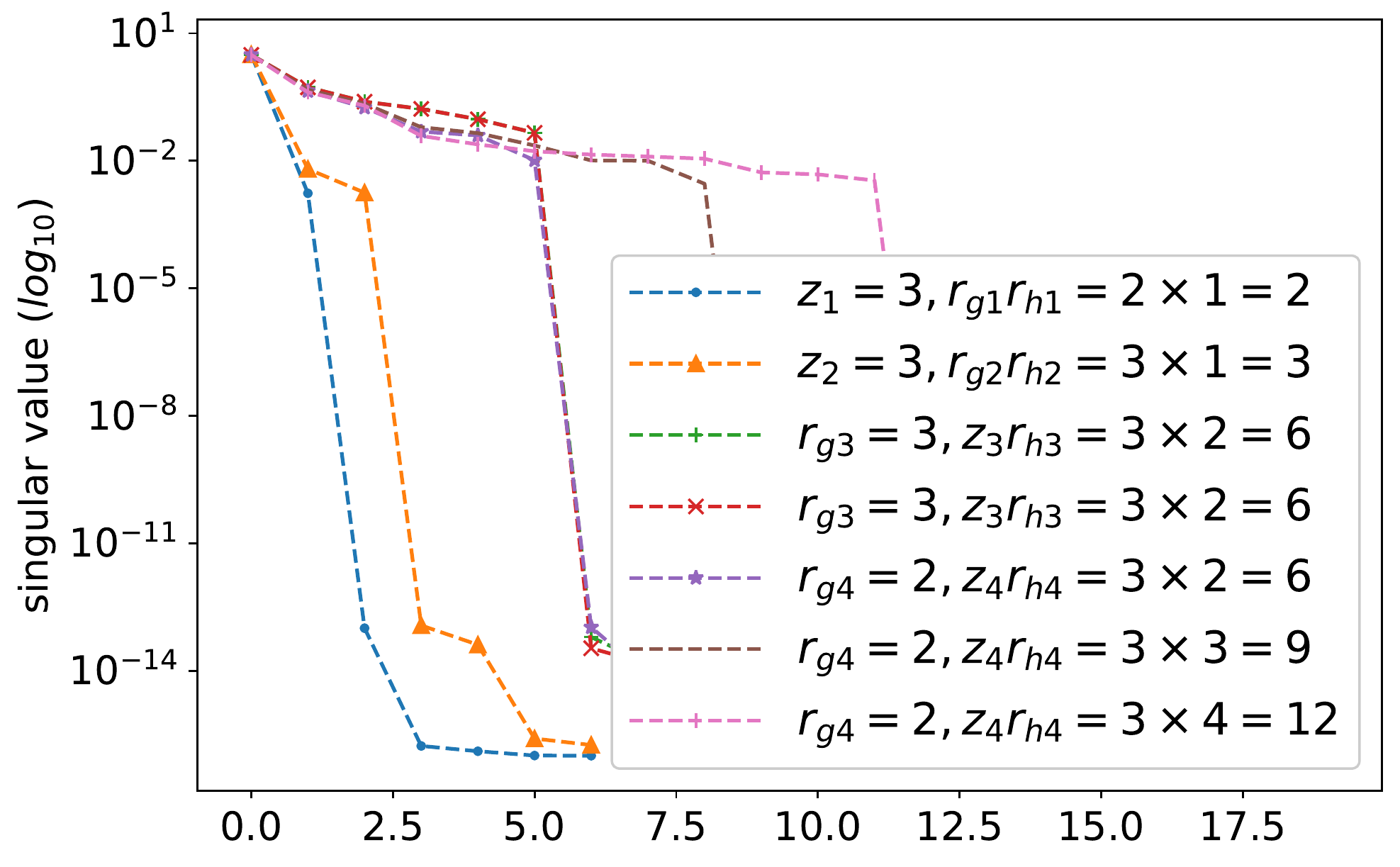}%
    \label{fig:oconv-outer-convolution-tucker-rank-u}
  }
  \caption{
    Singular values of mode-\(k\), \(k=1, 2, \cdots, 7\), matricization of \(\tvar{G} \oconv \lcerfl{\tvar{h}_1, \tvar{h}_2, \tvar{H}_3, \tvar{H}_4}\).
  }
  \label{fig:oconv-outer-convolution-tucker-rank}
\end{figure}

\end{document}